\documentclass{tlp}

\usepackage{graphicx}
\usepackage{multirow}

\let\proof\relax 
\usepackage{amsmath,amssymb} 
\usepackage{amsthm}
\usepackage{hyperref} 
\usepackage{subcaption} 
\usepackage{appendix}
\usepackage{enumitem}
\newtheorem{theorem}{Theorem}[section]
\newtheorem{proposition}{Proposition}[section]
\newtheorem{example}{Example}[section]
\newtheorem{definition}{Definition}[section]
\newtheorem{lemma}{Lemma}[section]
\newtheorem{corollary}{Corollary}[section]

\newtheorem{theoremA}{Theorem}[section]
\newtheorem{propositionA}{Proposition}[section]

\newtheorem{lemmaA}{Lemma}[section]
\newtheorem{corollaryA}{Corollary}[section]

\usepackage{algorithm}%
\usepackage{algorithmicx}%

\newcommand{\Body}{\textit{bd}}
\newcommand{\Pos}{\textit{\Body{$^+$}}} 
\newcommand{\Neg}{\textit{\Body{$^-$}}}
\newcommand{\Head}{\textit{hd}}
\newcommand{\Not}{\textit{not \,}}

\newcommand{\rto}{\rightarrow}
\newcommand{\lto}{\leftarrow}

\newcommand{\Rto}{\Rightarrow}

\newcommand{\LRto}{\Leftrightarrow}

\newcommand{\ILPSM}{\textsc{ilpsm}}
\newcommand{\ILPSMmin}{\textsc{ilpsm}min}
\newcommand{\ILSMmin}{\textsc{ilsm}min}

\newcommand{\ILASP}{\textsc{ilasp}}
\newcommand{\size}[1]{\mathit{size(#1)}}
\newcommand{\XFOLD}{\textsc{xfold}}
\newcommand{\PossILP}{\textsc{P}oss\textsc{ILP}}

\newcommand{\tuple}[1]{{\langle{#1}\rangle}}

\newcommand\ie{{\it i.e.}}
\newcommand\wrt{{\it w.r.t.\ }}

	\newif\ifcomments\commentstrue
	\commentsfalse
	
	\newif\ifcolors\colorstrue
	\colorsfalse
	
	\ifcolors 
	\definecolor{darkred}{rgb}{0.7,0,0}
	\newcommand{\ysm}[2]{{}{\ysa{#2}}}   
	\newcommand{\ysa}[1]{{\leavevmode\color{darkred}#1}} 
	\newcommand{\ysd}[1]{}  
	\newcommand{\ryd}[1]{}  
    \newcommand{\hbl}[1]{{\leavevmode\color{darkred}#1}}
	\else  
	\newcommand{\ysm}[2]{#2}
	\newcommand{\ysa}[1]{#1}
	\newcommand{\ysd}[1]{}  
    \newcommand{\hbl}[1]{#1}
	\fi

\begin{document}

\setlist[enumerate]{leftmargin=*}
\setlist[itemize]{leftmargin=*}

\lefttitle{Hongbo Hu et al.}

\jnlPage{xx}{xx}
\jnlDoiYr{2024}
\doival{10.1017/xxxxx}

\title[Inductive Learning for Possibilistic Logic Programs Under Stable Models]{Inductive Learning for Possibilistic Logic Programs Under Stable Models
}

\begin{authgrp}
\author{\sn{Hongbo} \gn{Hu}}
\affiliation{Multi-dimensional data perception and intelligent recognition Chongqing Engineering Research Center, Chongqing University of Arts and Sciences, Chongqing, 402160, China}
\affiliation{College of Computer Science and Technology, Guizhou University, Guiyang, 550025, China}
\author{\sn{Yisong} \gn{Wang}}
\affiliation{State Key Laboratory of Public Big Data; Key Laboratory of Advanced Medical Imaging and Intelligent Computing of Guizhou Province; College of Computer Science and Technology; Institute for Artificial Intelligence, Guizhou University, Guiyang, 550025, China}
\email{yswang@gzu.edu.cn}
\author{\sn{Yi} \gn{Huang}}
\affiliation{Multi-dimensional data perception and intelligent recognition Chongqing Engineering Research Center, Chongqing University of Arts and Sciences, Chongqing, 402160, China}
\author{\sn{Kewen} \gn{Wang}}
\affiliation{School of Information and Communication Technology, Griffith University, QLD 4111, Australia}
\email{k.wang@griffith.edu.au}
\end{authgrp}

\history{\sub{xx xx xxxx;} \rev{xx xx xxxx;} \acc{xx xx xxxx}}

\maketitle

\begin{abstract}
Possibilistic logic programs (poss-programs) under stable models are a major variant of answer set programming (ASP).
While its semantics (possibilistic stable models) and properties have been well investigated, the problem of inductive reasoning has not been investigated yet.  
This paper presents an approach to extracting poss-programs from a background program and  examples (parts of intended possibilistic stable models). To this end, the notion of induction tasks is first formally defined, its properties are investigated and two algorithms {\ILPSM} and {\ILPSMmin} for computing induction solutions are presented. An implementation of {\ILPSMmin} is also provided and experimental results show that when inputs are ordinary logic programs, the prototype outperforms a major inductive learning system for normal logic programs \hbl{from stable models} on the datasets that are randomly generated. This paper is under consideration in Theory and Practice of Logic Programming (TPLP).
\end{abstract}


\section{Introduction}
Inductive logic programming plays an important role in knowledge discovering, which has been applied in various practical applications such as natural language processing, multi-agent systems and bioinformatics~\citep{DBLP:journals/ml/MuggletonRPBFIS12,DBLP:journals/cacm/GulwaniHKMSZ15,DBLP:journals/ml/CropperDEM22}. Logic programming under stable models (also known as answer set programming or ASP in short) is one of the major formalisms for representing incomplete knowledge and nonmonotonic reasoning~\citep{GelfondLifschitz88,Brewka:CACM:2011,Erdem:AIM:2016}. 
As the semantics of standard stable models is unable to handle priority and uncertain information, various extensions of ASP have been proposed in the literature, including prioritized logic programming~\citep{SchaubW01,baral02a} and probabilistic logic programs~\citep{RaedtK15}\ysa{, $\textsc{lp}^{\textsc{mln}}$  ~\citep{DBLP:conf/kr/LeeW16} among others.}

Another major extension of ASP are the possibilistic logic programs~\citep{possSM/Nicolas/2005, nicolas2006possibilistic, dubois2021possibility}, in which each rule is assigned a weight (also called necessity). Possibility theory has been applied in several areas ~\citep{dubois2021possibility} such as belief revision, information fusion, and preference modelling. For instance, the statement `if a person is a resident in New York, then they are a USA citizen with a possibility of $0.7$' can be conveniently expressed as a possibilistic rule below, which is a pair of a rule and a number.
\[
(citizenUSA \lto residentNY, 0.7).
\]
The semantics of possibilistic logic programs is defined by an extension of stable models called \emph{possibilistic stable models} or \emph{poss-stable models}~\citep{possSM/Nicolas/2005, nicolas2006possibilistic} (see Section 2 for details). The possibilistic extension of ASP is different from probabilistic ones, as ~\citet{zadeh1999computing} has commented, if our focus is on the meaning of information rather than with its measure, the proper framework for information analysis is possibilistic rather than probabilistic in nature. We are not going to discuss this statement in detail but point out that possibilistic information is more on the side of representing the priority of formulas and rules. Moreover, as probability axioms are not enforced in possibilistic reasoning, it is easier for the user to manage `possibilistic' information than `probabilistic' information~\citep{dubois2000possibility}.

Because the poss-stable models can deal with non-monotonicity and uncertainty simultaneously, they can be applied to reasoning about uncertain epistemic beliefs and possibilistic dynamic systems.  
In addition to analysing the epistemic belief of a single agent, poss-stable models can be applied to reasoning about trust and belief among autonomous agents via argumentation~\citep{maia2016reasoning}, 
or to reasoning about possibilities across multiple information sources via a possibilistic multi-context system\footnote{
In general, multi-context systems differ from
multi-agent systems in that, unlike an agent, a context is not autonomous while there is information ﬂow between contexts.}~\citep{jin2012possibilistic}.
Besides, a poss-NLP can also represent a dynamic system~\citep{hu2025learning} whose dynamic characterization is depicted via the possibilistic interpretation transitions. 
Since an early study~\citep{inoue2011logic} theoretically revealed a strong mathematical relationship between the attractors/steady states of Boolean networks (BNs for short)  and the stable models, some ASP-based methods~\citep{mushthofa2014asp, khaled2023using} have been developed to find attractors in Boolean networks such as Genetic Regulatory Networks (GRNs for short).
It follows that the poss-stable models can correspond to the steady states of the dynamic system.

To our best knowledge, the problem of inductive reasoning with possibilistic logic programs has not been investigated in the literature. Informally, given a set of positive examples and a set of negative examples, the task is to induce a possibilistic logic program that satisfies these examples. 

Let us consider the following example that is adapted from Example 1 and Example 8 in Nicolas et al. (\citeyear{nicolas2006possibilistic}), which shows how possibilistic logic programs under poss-stable model semantics are used for representing quantitative priority information.

\hbl{
\begin{example}\label{exam:PAS}
    Assume that a clinical expert system contains the following knowledge:
    \begin{itemize}[label=-]
        \item If medicine A is taken, the possibility of relieving the vomiting is $0.7$; if medicine B is taken, the possibility of relieving the vomiting is $0.6$. 
        \item A physician prescribes medicine B if a vomiting patient has not taken medicine A. 
        \item If medicine A is taken, pregnancy causes malnutrition with a possibility of $0.7$; if medicine B is taken, pregnancy causes malnutrition with a possibility of $0.1$. 
    \end{itemize}
    
    This set of background (medical) knowledge can be expressed as the possibilistic logic program $\overline{P_{med}}$ below (under possibilistic stable models, see Section~\ref{sec:preliminary}).
    
    \[
        \overline{P_{med}} = \left \{ 
            \begin{array}{c}
                    (\mathit{relief} \lto vomiting, medA, 0.7), \\
                    (\mathit{relief} \lto vomiting, medB, 0.6), \\ 
                    (medB \lto vomiting, \Not medA, 1), \\
                    (malnutrition \lto medA, pregnancy, 0.7), \\ 
                    (malnutrition \lto medB, pregnancy, 0.1)
            \end{array}	
        \right\}
	\]
\ysa{Intuitively, the first rule says that if someone is vomiting and she takes the medicine A then her vomiting symptom will be relieved with the possible degree 0.7. Similarly, the second rule says that the medicine B can relieve the vomiting symptom with possible degree 0.6. These possible degrees are given by experts in this field.
}

\ysa{Suppose that ``A woman is definitely in pregnancy and she suffers from vomiting'', which can be presented as the two rules in $\overline{P_{fact}}$.  
\[\overline{P_{fact}} = \{ (pregnancy \lto , 1),\quad (vomiting \lto , 1) \}\]
The knowledge base expressed as program $\overline{B_{med}} = \overline{P_{fact}} \cup \overline{P_{med}}$ can be expanded further by learning new rules from new observations (examples). 

Let us assume that we have the following three examples:
\begin{itemize}
    \item[$\overline{A_1}$] $\{ (pregnancy,1), (vomiting,1), (medA,1), (\mathit{relief},0.7), (malnutrition,0.7) \}$,
    \item[$\overline{A_2}$] $\{ (pregnancy,1), (vomiting,1), (medB,1), (\mathit{relief},0.6), (malnutrition,0.1) \}$, 
    \item[$\overline{A_3}$] $\{ (pregnancy,1), (vomiting,1), (medA,0.7), (\mathit{relief},0.7) \}$
\end{itemize}
where $\overline{A_1}$ and $\overline{A_2}$ are positive examples (possibilistic stable models), while  $\overline{A_3}$ is a negative example. 
$\overline{{A_1}}$ is not a possibilistic stable model of the current $\overline{B_{med}}$, nor is $\overline{A_2}$. 
This means that some knowledge of the expert system is missing from the complete $\overline{B_{med}}$. 
What are the missing rules?

For instance, if we \ysm{insert}{add} another rule
\[\overline{r} = (medA \lto vomiting, \Not medB, 1)\]
into $\overline{B_{med}}$, then both $\overline{A_1}$ and $\overline{A_2}$ are possibilistic stable models of $\overline{B_{med}} \cup \{ \overline{r} \}$, while $\overline{A_3}$ is not.  How can we discover such a rule?
}
\end{example}
}

Thus, in this paper we aim to establish a framework for learning possibilistic rules from a given background possibilistic program, a set of positive examples and a set of negative examples. For instance, our method will be able to generate the rule $(medA \lto vomiting, \Not medB, 1)$ from the background knowledge $\overline{B_{med}}$, positive examples $E^+$ and negative examples $E^-$.
In order to set up our framework for induction in possibilistic programs, we first introduce the notion of induction tasks and then investigate its properties including useful characterizations of solutions for induction tasks. Based on these results, we present algorithms for computing solutions for induction tasks. We have also implemented and evaluated our algorithm using three randomly generated datasets.

In this paper, we will focus on the class of possibilistic normal logic programs (poss-NLPs) and our main contributions are summarised as follows.
\begin{itemize}
    \item We propose a definition of induction in poss-NLPs for the first time and investigate its properties.
    \item We present two algorithms {\ILPSM} and {\ILPSMmin} for computing induction solutions for poss-NLPs. We show that our algorithms are sound and complete. The first algorithm computes a poss-NLP that is a solution for the given induction task, while the second one finds a minimal solution. 
    \item We study two special cases of inductive reasoning for poss-NLPs. 
    \hbl{
    \ysm{First, we observe that in some applications, }{The first one is that }\ysa{observations are complete, \ie, the given positive examples are exactly the possibilistic stable models, while all other possibilistic interpretations are negative ones.}. 
    }
    In such a special case, we obtain an elegant characterisation for induction solutions in terms of poss-stable models. The other special case is when an input poss-NLP is an ordinary NLP. In this case, we show that the induction problem coincides with the standard induction for NLPs.
    \item We also generalise our definition of induction for poss-NLPs to a more general case where an interpretation is a partial interpretation. We show that this generalised induction problem can be reduced to the induction problem of poss-NLPs as defined in Definition~\ref{def:ILT}. Thus, the algorithms ILPSM and ILPSMmin are applicable to the more general case of induction for poss-NLPs.
    \item We have implemented the algorithm {\ILSMmin} to learn ordinal logic programs from stable models, which calls the answer set solver Clingo~\citep{Multishot}. 
    \hbl{Comparison experimental results show that {\ILSMmin} outperforms {\ILASP}~\citep{ILASP} on the randomly generated tasks of inducing NLP from stable models, in which {\ILASP} can learn non-ground answer set programs from partial interpretations. }
\end{itemize}

The rest of this paper is organized as follows. In Section~\ref{sec:preliminary}, we briefly recall some basics of possibilistic logic programs and poss-stable models. In Section~\ref{sec:learning:PNLP}, we propose the induction task for poss-NLPs and analyze its properties. 
Its two different algorithms are proposed and analyzed in Section~\ref{sec:learning:algorithm}. 
Section~\ref{sec:variants} discusses some variants of the induction for poss-NLPs. 
Section~\ref{sec:implementation:experiment} presents an efficient implementation and reports the results of the experimental evaluation. 
Section~\ref{sec:related:work} discusses related work. Finally, Section~\ref{sec:conclusion} concludes the paper. 
All proofs have been relegated to the appendix.

\section{Preliminaries}\label{sec:preliminary}
In this section, we briefly introduce some basics of possibilistic (normal) logic programs (poss-NLPs) and fix the notations that will be used in the paper~\citep{nicolas2006possibilistic}. 
For a set $O$, $\vert O \vert$ denotes cardinality of $O$. 

\subsection{Possibilistic logic programs}
We first introduce the syntax of poss-NLPs in this subsection.
We assume a propositional logic $\cal L$ over a finite set  ${\cal A}$ of {\em atoms}. 
Literals (positive and negative), terms, clauses, models, and satisfiability are defined as usual.  
A {\em possibilistic formula} $\overline{\phi}$ is a pair $(\phi,\alpha)$ with $\phi$ being a (propositional) formula of $\mathcal L$ and
$\alpha$ being the weight of $\phi$. 
In general, this $\alpha$ is an element of a given lattice $({\cal Q},\le)$~\citep{dubois1998possibility}. 
\ysd{\hbl{
Each task has its own set ${\cal Q}$ of weights. 
} *** you have not defined "task"}
In this paper, $({\cal Q},\le)$ is assumed to be a totally ordered set where ${\cal Q}$ is finite. 
Intuitively, every pair of weights in the finite set ${\cal Q}$ of a given induction task is comparable.

For example, ${\cal Q}= \{ 0.1, 0.6, 0.7, 1 \}$ and $\leq$ is the `less than or equal' relation for real numbers. 
The supremum of ${\cal Q}$ in this $({\cal Q},\le)$ is $1$. 
Except for a finite set of decimals in $[0, 1]$, ${\cal Q}$ can also be a set of adverbs representing necessities. For instance, ${\cal Q}= \{ slightly, highly, extremely, absolutely \}$ 
and $slightly \leq highly \leq extremely \leq absolutely$. 
The supremum of ${\cal Q}$ in this $({\cal Q},\le)$ is $absolutely$. 
Hereafter, we use $\mu$ to denote the supremum of ${\cal Q}$ in $({\cal Q},\le)$.

The possibilistic formula $(\phi,\alpha)$ intuitively states that the ordinary formula $\phi$ is certain at least to the level $\alpha$. 
If $\phi$ is an atom (resp. a literal, a clause and a term)  then $(\phi,\alpha)$
is a {\em possibilistic atom} (resp. literal, clause and term). 
Hereafter, we use the symbol $X$ to denote the classic projection ignoring all uncertainties from its possibilistic counterpart $\overline{X}$. 
Given a set $\cal A$ of atoms and a complete lattice $(\cal Q, \le)$, a set $\overline{I}\subseteq {\cal A}\times{\cal Q}$ of possibilistic atoms is called a {\em possibilistic interpretation} over $\cal A$ and $(\cal Q,\le)$ if $\vert \{ (x,\alpha) \in \overline{I} \mid \alpha \in {\cal Q} \} \vert \leq 1$ for each $x \in {\cal A}$. 
Thus, for a possibilistic interpretation $\overline{I}$, $(x,\alpha_1)\in \overline{I}$ and $(x,\alpha_2)\in \overline{I}$ imply $\alpha_1=\alpha_2$. 
Given a possibilistic interpretation $\overline{I}$, we use $I$ to denote the classic interpretation $\{p\mid (p,\alpha)\in \overline{I} \}$.

Before introducing possibilistic logic programs, we recall the set operations for finite sets of possibilistic formulas. 

Let $\overline{\Gamma_1}$ and $\overline{\Gamma_2}$ be two finite sets of possibilistic formulas. 
\begin{itemize}
    \item $\overline{\Gamma_1} \sqsubseteq \overline{\Gamma_2}$ if for each $(\phi,\alpha)\in \overline{\Gamma_1} $ there exists $(\phi,\beta)\in \overline{\Gamma_2}$ such that $\alpha \le \beta$. Consequently, $\overline{\Gamma_1} \sqsubset \overline{\Gamma_2}$ if $\overline{\Gamma_1} \sqsubseteq \overline{\Gamma_2}$ and $\overline{\Gamma_1} \neq \overline{\Gamma_2}$. 
    \item $ \overline{\Gamma_1}\sqcup  \overline{\Gamma_2}=\{(x, \alpha)\mid (x,\alpha)\in  \overline{\Gamma_1}, x \notin \Gamma_2\}\cup \{(x,\beta)\mid (x,\beta)\in \overline{\Gamma_2},x \notin \Gamma_1\}\cup \{(x,\max\{\alpha,\beta\})\mid (x,\alpha)\in  \overline{\Gamma_1}, (x,\beta)\in \overline{\Gamma_2}\}$.
    \item $ \overline{\Gamma_1}\sqcap  \overline{\Gamma_2}=\{(x, \min\{\alpha,\beta\})\mid (x,\alpha)\in  \overline{\Gamma_1}, (x,\beta)\in \overline{\Gamma_2}\}$.       
\end{itemize}

A {\em possibilistic normal logic program} ({\em poss-NLP} for short) on a complete lattice $({\cal Q},\le)$ is a finite set of {\em possibilistic normal rules} (or {\em rules}) of the form $\overline{r} = (r,\alpha)$, where
\begin{itemize}
	\item $\alpha\in {\cal Q}$ is the weight of $\overline{r}$, also written as $N(\overline{r})$, and
	\item $r$ is a classic (normal) rule of the form:
	\begin{align}
	\label{eq:normal:rule}
	p_0\lto p_1,\ldots, p_m,\Not p_{m+1}, \ldots, \Not p_n
	\end{align}
	with $n \geq 0$ and $p_i\in{\cal A}~(0\le i\le n)$.
\end{itemize}
A possibilistic logic program is also referred to as a {\em possibilistic logic knowledge base} or a {\em possibility theory} in the literature.

Given a poss-NLP $\overline{P}$, its classic counterpart, consisting of all classic rules in the poss-NLP, is denoted as $P$. 
Let $r$ be a classic rule of the form (\ref{eq:normal:rule}).
We denote $\Head(r)=p_0$,  $\Pos(r)=\{p_1,\ldots, p_m\}$,   $\Neg(r)=\{p_{m+1},\ldots, p_n\}$ and $\Body(r)=\Pos(r)\cup\Not\Neg(r)$. Thus, the rule $r$ can be written
as
\begin{align*}
\Head(r)\lto \Pos(r),\Not\Neg(r) \qquad\textit{or}\qquad \Head(r)\lto \Body(r)
\end{align*}
where $\Not S=\{\Not p\mid p\in S\}$. The rule $r$ is {\em definite} if $\Neg(r)=\emptyset$. For a possibilistic normal logic rule $\overline{r}$, we also denote $\Head(\overline{r})=\Head(r)$, $\Pos(\overline{r})=\Pos(r)$ and $\Neg(\overline{r})=\Neg(r)$. 
A {\em possibilistic definite (logic) program} is a finite set of possibilistic definite rules.

Two classic rules $r_1$ and $r_2$ of the form (\ref{eq:normal:rule}) are identical if $\Head(r_1) = \Head(r_2)$ and $\Body(r_1) = \Body(r_2)$. For instance, rule $p_0\lto p_1,  p_2,\Not p_3, \Not p_4$ is the same as rule $p_0\lto p_2,  p_1,\Not p_4, \Not p_3$. 
Similar to the assumption in Possibilistic Logic, we assume that every classic rule occurs at most once in a poss-NLP. 

Given two poss-NLPs $\overline{P_1}$ and $\overline{P_2}$, the operations $\cup$ and $-$ over poss-NLPs can be defined as follows~\citep{DBLP:conf/ijcai/GarciaLPSW18}.
\begin{itemize}
    \item $\overline{P_1}\sqcup  \overline{P_2}=\{(r, \alpha)\mid (r,\alpha)\in  \overline{P_1}, r \notin P_2\}\cup \{(r,\beta)\mid (r,\beta)\in \overline{P_2},r \notin P_1\}\cup \{(r,\max\{\alpha,\beta\})\mid (r,\alpha)\in  \overline{P_1}, (r,\beta)\in \overline{P_2}\}$.
    \item $\overline{P_1} - \overline{P_2} = \{ (r, \alpha) \mid (r,\alpha)\in  \overline{P_1}, r \notin P_2 \} \cup \{ (r, \alpha) \mid (r,\alpha)\in  \overline{P_1}, (r,\beta)\in \overline{P_2}, \alpha > \beta \}$.
\end{itemize}

\subsection{Possibilistic Stable Models}
In this subsection, we introduce the semantics of poss-NLPs, that is, \emph{possibilistic stable models} or \emph{poss-stable models}~\citep{nicolas2006possibilistic}.
To this end, we first recall basics of normal logic programs under stable models~\citep{GelfondLifschitz88}.

An atom set $S$ satisfies a definite logic program $P$, written as $S \models P$, if $\Pos(r) \subseteq S$ implies $\Head(r) \in S$ for each $r \in P$. $S$ is a {\em stable model} of $P$, written as $S \in \mathit{SM}(P)$, if $S \models P$ and there exists no $S' \subset S$ such that $S' \models P$.
In fact, a definite logic program $P$ has the unique stable model which is its least Herbrand model $\mathit{Cn}(P)$~\citep{lloyd2012foundations} (the set of consequences of $P$). 
$\mathit{Cn}(P)$ is also the least fixpoint $\mathit{lfp}(T_P)$ of the immediate consequence operator $T_P:2^{\cal A} \rto 2^{\cal A}$ defined by $T_P(A)=\Head(\mathit{App}(P, A))$, where $\Head(P) = \{ \Head(r) \mid r \in P \}$ and $\mathit{App}(P, A) = \{ r \in P \mid \Pos(r) \subseteq A \}$. 
That is, for a definite logic program $P$, we have $\mathit{Cn}(P) = \mathit{lfp}(T_P)$.

The following proposition is useful for checking whether $S = \mathit{lfp}(T_P)$ for a set $S$ of atoms.
A definite logic program $P$ is \emph{grounded} if it can be ordered as a sequence $(r_1, \dots, r_n)$ such that $ r_i \in \mathit{App}(P,\Head(\{ r_1, \dots, r_{i-1} \}))$ for each $1 \leq i \leq n$.


\begin{proposition}[Proposition~1 of~\citep{nicolas2006possibilistic}]\label{prop:DLP:LHM}	
    Let $P$ be a definite logic program and $S$ be an atom set. $S$ is a least Herbrand model of $P$ if and only if 
    \begin{itemize}
        \item $S = \Head(\mathit{App}(P,S))$, and 
        \item $\mathit{App}(P,S)$ is grounded.
    \end{itemize}
\end{proposition}

An atom set $S$ is a {\em stable model} of normal logic program (NLP) $P$ if $S$ is the stable model of $P^S$, the reduct of $P$ \wrt $S$, where $P^S$ denotes the definite logic program $\{\Head(r) \lto \Pos(r) \mid r \in P, \Neg(r) \cap S = \emptyset\}$~\citep{GelfondLifschitz88}. 
A stable model of $P$ is also referred to as an answer set of $P$.
The set of all stable models for $P$ is denoted as $\mathit{SM}(P)$.


Now we are ready to introduce the notion of poss-stable models~\citep{DUBOIS2023109028} , which is a generalisation of the stable models for ordinary NLPs.
First, we extend the definitions of reduct, rule applicability, and consequence operator to the case of poss-NLPs.

Given a possibilistic definite rule $\overline{r} = (r, \alpha)$ where $r$ is of the form $p \lto \{p_1,\ldots,p_m$, if there exists weights $\alpha_1, \ldots, \alpha_m$ such that $(p_i,\alpha_i)\in  \overline{A}$ for all $i~(1\le i\le m)$ ($1\le i\le m$), then we say that $\overline{r}$ is  {\em $\beta$-applicable} in a possibilistic atom set $\overline{A}$.

A definite rule $r$ is applicable in an atom set $A$ if $\Pos(r) \subseteq A$. Correspondingly, a possibilistic definite rule $(r,\alpha)$ with $\Pos(r)=\{p_1,\ldots,p_m\}$ is {\em $\beta$-applicable} in a possibilistic atom set $\overline{A}$ with $\beta=\min\{\alpha,\alpha_1,\ldots, \alpha_m\}$ if there exists $(p_i,\alpha_i)\in  \overline{A}$ for every $i~(1\le i\le m)$. Otherwise, it is not applicable in $\overline{A}$. 
For a certain atom $q \in {\cal A}$ and a possibilistic definite program $\overline{P}$, define 
\begin{align}\label{eq:app:atom}
\mathit{App}(\overline{P}, \overline{A}, q) = \{ \overline{r} \in \overline{P} \mid \Head(\overline{r}) = q,\overline{r} \mbox{ is $\beta$-applicable in } \overline{A} \}.
\end{align} 
Additionally, $\mathit{App}(\overline{P}, \overline{A}) = \bigcup_{q \in A} \mathit{App}(\overline{P}, \overline{A}, q)$.

Given a possibilistic definite logic program $\overline{P}$ and a possibilistic atom set $\overline{A}$, the immediate possibilistic consequence operator ${\cal T}_{\overline{P}}$ as Definiton 9 in ~\citep{nicolas2006possibilistic} maps a possibilistic atom set $\overline{A}$ to another one as follows.
\begin{align}\label{eq:Tp}
    {\cal T}_{\overline{P}}(\overline{A}) = &\{(q,\delta) \mid q \in {\cal A}, \mathit{App}(\overline{P},\overline{A},q) \neq \emptyset, \delta=\max\{\beta\mid \overline{r} \in \mathit{App}(\overline{P},\overline{A},q)\mbox{ is $\beta$-applicable in } \overline{A}\}\}.
\end{align}   
Then the iterated operator ${\cal T}_{\overline{P}}^k$ is deﬁned by ${\cal T}_{\overline{P}}^0 = \emptyset$ and ${\cal T}_{\overline{P}}^{n+1} = {\cal T}_{\overline{P}}({\cal T}_{\overline{P}}^{n})$ for each $n \geq 0$. 
For a possibilistic definite logic program $\overline{P}$, we can compute the set $\mathit{Cn}(\overline{P})$ of possibilistic consequences via $\mathit{Cn}(\overline{P}) = \mathit{lfp}({\cal T}_{\overline{P}})$ where $\mathit{lfp}({\cal T}_{\overline{P}}) = \bigsqcup_{n\geq0}{\cal T}_{\overline{P}}^n$ is the least fixpoint of the immediate possibilistic consequence operator ${\cal T}_{\overline{P}}$.

The possibilistic reduct of a poss-NLP $\overline{P}$ \wrt an atom set $S$ is the possibilistic definite logic program 
\[
\overline{P}^S = \{(\Head(r) \lto \Pos(r), N(\overline{r})) \mid \overline{r} \in \overline{P},\Neg(r) \cap S = \emptyset\}.
\]
A set $\overline{S}$ of possibilistic atom is a poss-stable model (or poss-stable model) of poss-NLP $\overline{P}$ if $\overline{S} = \mathit{Cn}(\overline{P}^S)$. 
The set of all poss-stable models of $\overline{S}$ is denoted $\mathit{PSM}(\overline{P})$. 

The following example illustrates some of the above notions for poss-NLPs. 
\begin{example}
    Let ${\cal A} = \{ a,b,c \}$, $\overline{P} = \{(a \lto \Not b,0.6),(a \lto ,0.9),(b \lto , 0.6),(c \lto a,b,0.8)\}$ and $\overline{S} = \{ (a,0.9), (b,0.6), (c,0.6) \}$. Then $\mathit{Cn}(\overline{P}^S) = \{ (a,0.9),(b,0.6),(c,0.6) \}$ since $\overline{P}^S = \{(a \lto ,0.9),(b \lto , 0.6),(c \lto a,b,0.8)\}$ and 
    \begin{align*}
        {\cal T}_{\overline{P}^S}^0 = &{\cal T}_{\overline{P}^S}(\emptyset) = \{ (a,0.9),(b,0.6) \}, \\
        {\cal T}_{\overline{P}^S}^1 = &{\cal T}_{\overline{P}^S}(\{ (a,0.9),(b,0.6) \}) = \{ (a,0.9),(b,0.6),(c,0.6) \}, \\
        {\cal T}_{\overline{P}^S}^2 = &{\cal T}_{\overline{P}^S}(\{ (a,0.9),(b,0.6),(c,0.6) \}) = \{ (a,0.9),(b,0.6),(c,0.6) \} = \overline{S}.
    \end{align*}  
    As a result, $\overline{S} \in \mathit{PSM}(\overline{P})$.
\end{example}

The next proposition shows that there is a one-to-one correspondence between the set $\mathit{PSM}(\overline{P})$ of poss-stable models for $\overline{P}$ and the set $\mathit{SM}(P)$ of stable models for $P$. 


\begin{proposition}[Proposition~10 of ~\citep{nicolas2006possibilistic}]\label{prop:SM:maps}	
    Let $\overline{P}$ be a poss-NLP. 
    \begin{itemize}
        \item If $A$ is a stable model of $P$, then $\mathit{Cn}(\overline{P}^A)$ is a poss-stable model of $\overline{P}$. 
        \item If $\overline{A}$ is a poss-stable model of $\overline{P}$, then $A$ is a stable model of $P$.  
    \end{itemize}
\end{proposition}

The immediate possibilistic consequence operator introduced in~\citet{nicolas2006possibilistic} is for possibilistic definite logic programs only, but it can be generalised to poss-NLPs as follows.
Beforehand, let us define the applicability of poss-rules in a poss-NLP.

\begin{definition}[Poss-rule applicability]\label{def:applicable}
	Given a possibilistic interpretation $ \overline{I}$ and a weight $\beta$, 
    a possibilistic normal rule $(r,\alpha)$ 
	is {\em $\beta$-applicable} in $ \overline{I}$ if
	the possibilistic definite rule $(\Head(r)\lto\Pos(r),\alpha)$ is $\beta$-applicable in $ \overline{I}$ 
	and,
	$\Neg(r)\cap I=\emptyset$. Otherwise, it is not applicable in $\overline{I}$.
\end{definition}

With this new definition of applicability in hand, Equation~(\ref{eq:app:atom}), the definition of $\mathit{App}(\overline{P}, \overline{A}, q)$ of applicable rules over a poss-NLP $\overline{P}$, still works for poss-NLPs. Consequently, the immediate possibilistic consequence operator over a poss-NLP, Equation~\ref{eq:Tp}, can be easily extended to poss-NLPs as follows.

\begin{definition}[Immediate possibilistic consequence operator ${\cal T}_{\overline{P}}$]\label{def:Tp}
	Let $\overline{P}$ be a poss-NLP and $\overline{I} \in 2^{{\cal A}\times {\cal Q}}$ be a possibilistic interpretation. We define ${\cal T}_{\overline{P}}: 2^{{\cal A}\times {\cal Q}}\rto 2^{{\cal A}\times {\cal Q}}$ as follows.
    \begin{align*}
        {\cal T}_{\overline{P}}(\overline{A}) = &\{(q,\delta) \mid q \in {\cal A}, \mathit{App}(\overline{P},\overline{A},q) \neq \emptyset, \delta=\max\{\beta\mid \overline{r} \in \mathit{App}(P,\overline{A},q)\mbox{ is $\beta$-applicable in } \overline{A}\}\}.
    \end{align*}  
\end{definition}

The following example illustrates how the reduct of a poss-NLP wrt a poss-interpretation and its least fixpoint can be computed.
\begin{example}  \label{exam:Tp}
    Let $\overline{I}=\{(q, 0.9), (s, 0.7)\}$ be a possibilistic interpretation and $\overline{P} = \{ r_1, r_2, r_3, r_4 \}$ be a poss-NLP where 
	\begin{align*}
    	& r_1 = ( p \lto q,s, 0.9),\\
    	&r_2 = ( p \lto \Not r, 0.9),\\
    	&r_3 = ( p \lto \Not s, 0.7),\\
    	&r_4 = ( p \lto r, 0.7). 
	\end{align*}
    By Definition~\ref{def:applicable}, we have   
	\begin{itemize}
		\item $r_1$ is 0.7-applicable in $\overline{I}$ since $\Pos(r_1) \subseteq I$ and $\Neg(r_1) \cap I = \emptyset$ and $\min\{0.9, 0.7, 0.9\}=0.7$;
		\item $r_2$ is 0.9-applicable in $\overline{I}$ since $\Pos(r_2) \subseteq I$ and $\Neg(r_2) \cap I = \emptyset$ and $\min\{0.9\}=0.9$;
		\item $r_3$ is not applicable in $\overline{I}$ since $\Neg(r_3) \cap I \neq \emptyset$;
		\item $r_4$ is not applicable in $\overline{I}$ since $\Pos(r_4) \not\subseteq I$.
	\end{itemize}
	By Definition~\ref{def:Tp}, it follows that ${\cal T}_{\overline{P}}(\overline{I})=\{(p,0.9)\}$ as $\max\{0.9,0.7\}=0.9$.
\end{example}

As the next proposition shows, the immediate consequence operator for poss-NLPs in Definition~\ref{def:Tp} reserves the key properties of the original immediate consequence operator for ordianry NLPs.

\begin{proposition}[Equivalent consequence]\label{prop:eq:con}
    For a given poss-NLP $\overline{P}$ and a possibilistic interpretation $\overline{I}$, ${\cal T}_{\overline{P}^I}(\overline{I}) = {\cal T}_{\overline{P}}(\overline{I})$.  
\end{proposition}

Such an integration simplifies the reasoning and representation concerning poss-stable models. 
For instance, Corollary~\ref{cor:SM:absorption} discussed a property of a poss-stable model without directly touching upon the GL-reduct. This corollary states that a possibilistic interpretation $\overline{I}$ is still a poss-stable model of the extension of a poss-NLP $\overline{P}$ when $\overline{P}$ absorbs a new poss-NLP $\overline{B}$ such that ${\cal T}_{\overline{B}}(\overline{I}) \sqsubseteq \overline{I}$.

\begin{corollary}[Absorption for a poss-stable model]\label{cor:SM:absorption}
    Given a possibilistic interpretation $\overline{I}$ and two poss-NLPs $\overline{P}$ and $\overline{B}$, $\overline{I} \in \mathit{PSM}(\overline{P} \sqcup \overline{B})$ if $\overline{I} \in \mathit{PSM}(\overline{P})$ and ${\cal T}_{\overline{B}}(\overline{I}) \sqsubseteq \overline{I}$.
\end{corollary}

\section{Induction Tasks for Possibilistic Logic Programs}\label{sec:learning:PNLP}
In this section, we first present the definition of induction tasks for poss-NLPs and then investigates some properties. These results will further be applied in computing (minimal) induction solutions. 
The necessary and sufficient condition for existing a solution for an induction task relies on the relationships such as (in)comparability and coherency among the components of the task. 
To solve such an induction task, we propose an algorithm {\ILPSM} for the construction of a particular solution and further propose another algorithm {\ILPSMmin} to identify a minimal solution.  
To narrow the search scope of the algorithm {\ILPSMmin}, the concept of solution space is introduced and its properties are investigated.

Informally, assume that there is a given background knowledge base formalised as a set of possibilistic rules, and a set of (positive and negative) examples, the induction task is to extract a new poss-NLP that covers positive examples but none of the negative examples. 
We first formally define the notion of induction tasks for poss-NLPs in Definition~\ref{def:ILT}. 

\begin{definition}[Induction task for poss-NLPs]\label{def:ILT} Let ${\cal A}$ be a set of atoms and ${\cal Q}$ be a complete lattice.
An {\em induction task} for poss-NLPs is a tuple $T = \tuple{ \overline{B}, E^+, E^-}$, where the background knowledge $\overline{B}$ is a poss-NLP over ${\cal A}$ and ${\cal Q}$, $E^+$ and $E^-$ are two sets of possibilistic interpretations called the set of positive examples and the set of negative examples, respectively. A hypothesis $\overline{H}$ is a NLP over ${\cal A}$ and ${\cal Q}$. We say $\overline{H}$ is a solution of the induction task $T$ if the following two conditions are satisfied:
\begin{align}
    E^+\subseteq \mathit{PSM}(\overline{B} \sqcup \overline{H}), \tag{{G1}} \label{eq:G1}\\
    E^-\cap \mathit{PSM}(\overline{B} \sqcup \overline{H})=\emptyset. \tag{{G2}} \label{eq:G2}
\end{align}
The set of all solutions of the induction task is denoted $\mathit{ILP_{LPoSM}}(T)$.

\end{definition}


\ysa{For convenience, we assume that $\cal A$ and $\cal Q$ consists of the atoms and weights occurring in the induction task $T$, receptively, unless explicitly stated otherwise. \ysd{In the above definition, an example is a possibilistic interpretation, which can be positive, or negative, or none.}}

Note that Condition~\ref{eq:G1} requires only $E^+$ is a portion of the poss-stable models of $\overline{B} \sqcup \overline{H}$. 
This ``partiality of examples'' should not be confused with the ``model partiality''. 
The following example shows some instances of induction tasks.

\begin{example}\label{exam:PAS:solutions}
    For the scenario in Example~\ref{exam:PAS}, the induction task can be represented as $T_1 = \tuple{\overline{B_{med}}, \{ \overline{A_1},\overline{A_2} \}, \{ \overline{A_3} \} }$ where 
    \begin{align*}
    	& \overline{A_1} = \{ (pregnancy,1), (vomiting,1), (medA,1), (\mathit{relief},0.7), (malnutrition,0.7) \},\\
    	& \overline{A_2} = \{ (pregnancy,1), (vomiting,1), (medB,1), (\mathit{relief},0.6), (malnutrition,0.1) \},\\
    	& \overline{A_3} = \{ (pregnancy,1), (vomiting,1), (medA,0.7), (\mathit{relief},0.7) \}.   
    \end{align*} 
    Here ${\cal A} = \{ pregnancy, vomiting, medA, medB, \mathit{relief}, malnutrition \}$ and ${\cal Q} = \{ 0.1,0.6,0.7,1 \}$ and $0.1 \leq 0.6 \leq 0.7 \leq 1$.  
    The induction task $T_1$ has a solution $\overline{H_1} = \{ (medA \lto vomiting, \Not medB, 1) \}$. 
    Also,  $\overline{H_2} = \overline{H_1} \cup \{ (\mathit{relief} \lto vomiting, 0.6) \}$ is a solution of $T_1$. 
    Let $\overline{P_{med}} = \overline{B_{med}} \cup \overline{H_1}$. 
    Then $\overline{A_1}$ and $\overline{A_2}$ are poss-stable models of $\overline{P_{med}}$, while $\overline{A_3}$ is not.

    However, the induction task $T_2 = \tuple{\overline{B_{med}}, \{ \overline{A_1},\overline{A_2}, \{ (pregnancy, 0.6) \} \}, \{ \overline{A_3} \} }$ has no solution, that is, $\mathit{ILP_{LPoSM}}(T_2) = \emptyset$. 

    For induction task $T_3  = \tuple{\emptyset, \{ \{ (p, 0.3), (q,0.3) \} \}, \emptyset}$, we have ${\cal A} = \{ p, q \}$ and ${\cal Q} = \{ 0.3 \}$. 
    This induction task has more than one solution. 
    For instance,  $\overline{H_3}$, $\overline{H_4}$, $\overline{H_5}$, $\overline{H_6}$ are all solutions of $T_3$, that is, $\{ \overline{H_3}, \overline{H_4}, \overline{H_5}, \overline{H_6} \} \subseteq \mathit{ILP_{LPoSM}}(T_3)$, where $\overline{H_3} = \{ (p \lto , 0.3), (q \lto p, 0.3) \}$, $\overline{H_4} = \{ (p \lto q, 0.3), (q \lto , 0.3) \}$, $\overline{H_5} = \{ (p \lto , 0.3), (q \lto , 0.3) \}$, and $\overline{H_6} = \{ (p \lto , 0.3), (q \lto , 0.3), (q \lto p, 0.3), (p \lto q, 0.3) \}$.   

    For induction task $T_4  = \tuple{ \emptyset, \{ \{ (p, 0.3), (q,0.3) \} \}, \{ \{ (p, 0.3), (q,0.3) \} \}}$, it is obvious that $\mathit{ILP_{LPoSM}}(T_4) = \emptyset$. 

    For induction task $T_5  = \tuple{ \{ (p \lto , 1) \}, \{\{(q,1),(p,1)\}\}, \{\{(q,1)\}\} }$, it can be checked that $\{ (q \lto , 1) \} \subseteq \mathit{ILP_{LPoSM}}(T_5)$. 
\end{example}

We recall that any two stable models of a NLP are $\subseteq$-incomparable. 
Namely, $\mathit{SM}(P)$ of any NLP $P$ is $\subseteq$-incomparable. 
For a poss-NLP, there is a similar property.
Before formally stating the property, we note that the (in)comparability of interpretations can be similarly defined for poss-interpretations. 

\begin{definition}[Incomparability for possibilistic interpretations]\label{def:poss:comparable}
    Let $\overline{I}$ and $\overline{J}$ be two possibilistic interpretations, and $\overline{S}$ be a set of possibilistic interpretations.
    \begin{enumerate}
        \item $\overline{I}$ and $\overline{J}$ are comparable, written as $\overline{I} \parallel \overline{J}$, if $I\subseteq J$ or $J\subseteq I$. Otherwise, they are incomparable, written as $\overline{I} \not\parallel \overline{J}$. 
        \item $\overline{I}$ is incomparable \wrt $\overline{S}$, written as $\overline{I} \not\parallel \overline{S}$, if $\overline{I} \not\parallel \overline{J}$ for every $\overline{J}\in S$. Otherwise, $\overline{I}$ is comparable \wrt $\overline{S}$, written as $\overline{I} \parallel \overline{S}$.
        \item $\overline{S}$ is incomparable if $\overline{I} \not\parallel \overline{J}$ for every pair of different interpretations $\overline{I}\in \overline{S}$ and $\overline{J} \in \overline{S}$. Otherwise, $\overline{S}$ is comparable.     
    \end{enumerate}     
\end{definition}

By the minimality of stable models, two different $\subseteq$-comparable interpretations cannot simultaneously be poss-stable models of the same poss-NLP. 

\begin{proposition}[Incomparability between poss-stable models]\label{prop:possSM:incomparable}
    Given two different possibilistic interpretations $\overline{I}$ and $\overline{J}$ such that $\overline{I} \parallel \overline{J}$, $\{ \overline{I}, \overline{J} \} \not\subseteq \mathit{PSM}(\overline{P})$ for any poss-NLP $\overline{P}$.
\end{proposition}

Intuitively, two different comparable possibilistic interpretations cannot simultaneously appear in the set of stable models of the same poss-NLP. This result are useful for optimising the algorithms for solving induction tasks in poss-NLPs, which actually provides two heuristics for early termination of some search paths. First, the induction task has no solution if two positive examples in the given $E^+$ are comparable. Moreover, a negative example in $E^-$ can be ignored immediately if it is comparable with a positive example in $E^+$.

Induction task $T = \tuple{ \overline{B}, E^+, E^-}$ of poss-NLP from poss-stable models in two aspects at least. On one hand, an induction task has no solution if two positive examples in the given $E^+$ are comparable as Example~\ref{exam:PSM:incomparable}. On the other hand, a negative example $\overline{e} \in E^-$ can be ignored before the solving process if $\overline{e}$ is comparable with a positive example in $E^+$.

\begin{example}\label{exam:PSM:incomparable}
    Let $\overline{S} = \{ \overline{I}, \overline{J} \} $ be a set of possibilistic interpretations, where $\overline{I} = \{ (p,0.3), (q,0.5) \}$ and $\overline{J} = \{ (p,0.4), (q,0.4) \}$. It is evident that $\overline{I} \parallel \overline{J}$, but $\{ \overline{I}, \overline{J} \} \not\subseteq \mathit{PSM}(\overline{P})$ for any poss-NLP $\overline{P}$. Thus, $T_{21}  = \tuple{\emptyset, E^+, \emptyset}$ has no solution when $E^+ = \overline{S}$ or $E^+ = \{ \{ (p,0.3), (q,0.5) \}, \{ (p, 0.4) \} \}$.
\end{example}

In the process of extracting a poss-NLP from both positive and negative examples, we are going to obtain some other necessary conditions for a solution of the induction task, which will be useful for implementation of algorithms for induction for poss-NLPs.

By Proposition~\ref{prop:SM:maps}, $\mathit{PSM}(\overline{P}) = \overline{S}$ implies $\mathit{SM}(P) = S$, but not vice versa. 

\begin{example}\label{exp:no:CE}
    Let $\overline{S} = \{ \overline{I}, \overline{J} \} $ and $\overline{P} = \{ \overline{r_1}, \overline{r_2}, \overline{r_3} \}$ where $\overline{I} = \{ (p,0.3), (q,0.5) \}$, $\overline{J} = \{ (p,0.4), (r,0.4) \}$, $\overline{r_1} = (p \lto ,0.3)$, and $\overline{r_2} = (q \lto \Not r, 0.6)$, $\overline{r_3} = (r \lto \Not q,0.4)$. It is evident that $\mathit{SM}(P) = S$ but $\mathit{PSM}(\overline{P}) \neq \overline{S}$. 

        We note that changing the weights of rules in $\overline{P}$ does not make that the resulting poss-NLP $\overline{H}$ to satisfy the condition $\mathit{PSM}(\overline{H}) = \overline{S}$, here $H = P$. To see this, on the contrary we assume that $\mathit{PSM}(\overline{H}) = \overline{S}$ for some $\overline{H}$ obtained from $\overline{P}$ by updating the weights of three rules. We note that $\overline{I} \in \mathit{PSM}(\overline{H})$ and $r_1$ is the only rule in $P$ that supports the atom $p$. Thus, the weight $\alpha_1$ of $r_1$ has to be $0.3$. On the other hand, $\overline{J} \in \mathit{PSM}(\overline{H})$ would imply $\alpha_1 = 0.4$ for the identical rationale, a contradiction.
        
\end{example}

From the above example, given a collection $\overline{S}$ of poss-interpretations, an ordinary NLP $H$ can be constructed from $\overline{S}$ such that $\mathit{SM}(H) = S$, but there may be no $\overline{H}$ such that $\mathit{PSM}(\overline{H}) = \overline{S}$.
This shows that, due to the introduction of weights in rules, the problem of induction for poss-NLPs is not as easy as in the case of ordinary NLPs.


In the rest of this section, we will investigate some properties of induction tasks for poss-programs, especially, the existence of induction solutions. To this aim, we first introduce two notions that are used for characterising rules that can cover a candidate model (supporting positive examples and blocking negative examples, respectively).


Let $\overline{S}$ be a set of possibilistic interpretations. We define 
    \begin{equation} \label{eq:PPE}
        \mathit{PPE}(\overline{S}) = \bigsqcup_{\overline{I} \in \overline{S}} \mathit{PPE}(\overline{I})
    \end{equation} 
    where $\mathit{PPE}(\overline{I}) = \{ (x \lto \Not ({\cal A} - I), \alpha) \mid \text{ for all } (x,\alpha) \in \overline{I} \}$.


In an induction task, if $\overline{I}$ in $E^+$, then $\mathit{PPE}(\overline{I})$ contains the possibilistic rules that potentially cover the positive example. 
Thus, $\mathit{PPE}(E^+)$ contains the rules that are potentially cover $E^+$ (i.~e., all positive examples) 
if $E^+$ is incomparable. 





\begin{proposition}[Satisfiability for program $\mathit{PPE}(\overline{S})$]	\label{prop:possNLP:exist} 
    For a given set $\overline{S}$ of possibilistic interpretations, $\mathit{PSM}(\mathit{PPE}(\overline{S})) = \overline{S}$ if $\overline{S}$ is incomparable. 
\end{proposition}

\begin{example}\label{exam:possNLP:incomparable}
    Let $T_{22}  = \tuple{\emptyset, E^+, \emptyset}$ be an induction task, where $E^+ = \{ \{ (p, 0.5), (r, 0.5) \}, \{ (q, 0.3), (r, 0.8) \} \}$. Then ${\cal A} = \{ p,q,r \}$, ${\cal Q} = \{ 0.3, 0.5, 0.8 \}$ and $0.3 \leq 0.5 \leq 0.8$. Obviously, $E^+$ is incomparable. It is easy to see that $\mathit{PSM}(\mathit{PPE}(E^+)) = E^+$. So, $\mathit{PPE}(E^+) \in \mathit{ILP_{LPoSM}}(T_{22})$, here $\mathit{PPE}(E^+) = \{ (p \lto \Not q, 0.5), (r \lto \Not q, 0.5), (q \lto \Not p, 0.3), (r \lto \Not p, 0.8) \}$. 
\end{example}

Apart from the positive examples, the impact of negative examples on the induction solution can be also handled by constructing a special poss-NLP.  
Let $\overline{S_N}$ and $\overline{S_P}$ be two sets of possibilistic interpretations. We define 
    \begin{equation} \label{eq:PNE}
        \mathit{PNE}(\overline{S_N},\overline{S_P}) = \bigsqcup_{\overline{I} \in \overline{S_N}, I \neq {\cal A}, \overline{I} \not\parallel \overline{S_P}} \mathit{PNE}(\overline{I})
    \end{equation}
    where $\mathit{PNE}(\overline{I}) = \{ (x_0 \lto I, \Not ({\cal A} - I), \mu) \mid \text{ for some } x_0 \in ({\cal A} - I) \}$. Here $x_0$ is an arbitrary element in ${\cal A} - I$ and $\mu$ is 
    the supremum of ${\cal Q}$ in $({\cal Q},\le)$. In other words, $\mathit{PNE}(\overline{I})$ contains exactly one rule that guarantees the second condition in the definition of induction tasks is satisfied.

    For an induction task $T = \tuple{\overline{B}, E^+, E^-}$, it will be proved that $\mathit{PNE}(E^-,E^+)$ corresponds to a poss-NLP blocking some of possibilistic interpretations in $E^-$ to become the poss-stable models. 
    \ysa{Please note that   $|\mathit{PPE(\overline S)}|\le \sum_{\overline I\in\overline S}|\overline I|$ and $|\mathit{PNE(\overline{S_N},\overline{S_P})}|\le |S_N|$. Both of them can be computed in polynomial time.} 

\begin{proposition}[Unsatisfiability for program $\mathit{PNE}(\overline{S_N},\overline{S_P})$]\label{prop:PSM:PNE}
    Given a poss-NLP $\overline{P}$ and two sets $\overline{S_N}$ and $\overline{S_P}$ of possibilistic interpretations, $\overline{I} \notin \mathit{PSM}(\mathit{PNE}(\overline{S_N},\overline{S_P}) \sqcup \overline{P})$ if $\overline{I} \in \overline{S_N}$, $I \neq {\cal A}$, and $ \overline{I} \not\parallel \overline{S_P}$.
\end{proposition}
By the above proposition, the poss-NLP $\mathit{PNE}(\overline{S_N},\overline{S_P})$ blocks some possibilistic interpretations in $\overline{S_N}$ to become poss-stable models. Consider the example below.

\begin{example}\label{exam:PSM:PNE}
    Let $\overline{S_N} = \{ \overline{I_1}, \overline{I_2}, \overline{I_3} \}$ and $\overline{S_P}= \{ \overline{J} \}$ be two sets of possibilistic interpretations over ${\cal A} = \{ p,q,r \}$ where $\overline{I_1} = \{ (p, 0.3), (q, 0.3), (r, 0.5) \}$,  $\overline{I_2} = \{ (p, 0.5), (r, 0.5) \}$,  $\overline{I_3} = \{ (q, 0.3), (r, 0.8) \}$ and $\overline{J} = \{ (p, 0.3) \}$. It is evident that $I_1 = {\cal A}$, $I_1 \supseteq J$ and $I_2 \supseteq J$. Hence, $\{ \overline{I} \in \overline{S_N} \mid  I \neq {\cal A}, \overline{I} \not\parallel \overline{S_P}\} = \{ \overline{I_3} \}$. Let $\mathit{PNE}(\overline{S_N},\overline{S_P}) = \mathit{PNE}(\overline{I_3}) = \{ \overline{r} \} = \{ (p \lto q, r, \Not p, 0.8) \}$ so that $I_3 \not\models r$. As a result, $\overline{I_3} \notin \mathit{PSM}(\mathit{PNE}(\overline{S_N},\overline{S_P}) \sqcup \overline{P})$ for any poss-NLP $\overline{P}$.

    Assume $\overline{S_N} = \{ \overline{I_3}, \overline{I_4} \}$ where $\overline{I_4} = \{ (r, 0.5) \}$. 
    When other variables remain the same, it is evident that $I_4 \neq {\cal A}$ and $\overline{I_4} \not\parallel \overline{J}$. 
    Now we have $\{ \overline{I} \in \overline{S_N} \mid  I \neq {\cal A}, \overline{I} \not\parallel \overline{S_P}\} = \{ \overline{I_3}, \overline{I_4} \}$. 
    For $\overline{I_3}$, $\mathit{PNE}(\overline{I_3})$ must be $ \{ (p \lto q, r, \Not p, 0.8) \}$ since ${\cal A} - I_3 = \{ p \}$ contains only one atom. 
    In contrast, $\mathit{PNE}(\overline{I_4})$ can be $ \{ (p \lto r, \Not p, \Not q, 0.8) \}$ or $ \{ (q \lto r, \Not p, \Not q, 0.8) \}$ since ${\cal A} - I_4 = \{ p, q \}$ contains two atoms. 
    The rule $\overline{r_4} = (x \lto r, \Not p, \Not q, 0.8)$ always satisfies the condition $I_4 \not\models r_4$ when $x \in {\cal A} - I_4$. 
    Therefore, $\mathit{PNE}(\overline{S_N},\overline{S_P}) = \mathit{PNE}(\overline{I_3}) \sqcup \mathit{PNE}(\overline{I_4})$ can be $\{ (p \lto q, r, \Not p, 0.8), (p \lto r, \Not p, \Not q, 0.8) \}$ or $\{ (p \lto q, r, \Not p, 0.8), (q \lto r, \Not p, \Not q, 0.8) \}$. 
    Regardless of which of these two poss-NLPs assigned to $\mathit{PNE}(\overline{S_N},\overline{S_P})$, $\overline{S_N} \cap \mathit{PSM}(\mathit{PNE}(\overline{S_N},\overline{S_P}) \sqcup \overline{P}) = \emptyset$ for any poss-NLP $\overline{P}$.
    
\end{example}

As explained, the poss-NLP $\mathit{PNE}(\overline{S_N},\overline{S_P})$ only considers a part of possibilistic interpretations in $\overline{S_N}$. A possibilistic interpretation $\overline{G}$ such that $G = {\cal A}$ goes beyond the consideration. So we also construct another special poss-NLP $\mathit{PPE}(\overline{G})$.
As Lemma~\ref{lm:SM:PPI} shows, $\mathit{PPE}(\overline{G})$ can be used to generate the poss-stable model $\overline{G}$ under certain conditions. Intuitively, a possibilistic interpretation $\overline{I}$ such that $I = {\cal A}$ is so particular since it is comparable to any possibilistic interpretation. Hereafter, 
we denote $\overline{\mathbb{A}}=\{\overline I\mid \mbox{$\overline I$ is a possibilistic interpretation with $I=\cal A$}\}$.
%

\begin{lemma}[Unique stable model for program $\mathit{PPE}(\overline{G})$]\label{lm:SM:PPI}
    Given a poss-NLP $\overline{B}$ and a possibilistic interpretation $\overline{G}$ with $G = {\cal A}$, if ${\cal T}_{\overline{B}}(\overline{G}) \sqsubseteq \overline{G}$, then $\mathit{PSM}(\overline{B} \sqcup \mathit{PPE}(\overline{G})) = \{ \overline{G} \}$.
\end{lemma}

With constructions of the above special poss-NLPs, for a given poss-interpretation $\overline{I}$, the existence of a poss-NLP $\overline{G}$ with $\overline{I}$ being a poss-stable mode of $\overline{G}$ is guaranteed by the closeness of $\overline{I}$ under the immediate consequence operator of $\overline{G}$.


\begin{proposition}[Existence of a program w.r.t. a poss-stable model]	\label{prop:possNLP:interpretation:exist} 
    For a possibilistic interpretation $\overline{I}$ and a poss-NLP $\overline{B}$, there exists a poss-NLP $\overline{P}$ such that $\overline{I} \in \mathit{PSM}(\overline{B} \sqcup \overline{P})$ if and only if ${\cal T}_{\overline{B}}(\overline{I}) \sqsubseteq \overline{I}$.
\end{proposition}

By the above proposition, we are able to state useful corollaries in the following.

\begin{corollary}
\label{cor:possNLP:B:exist}
For an incomparable set $\overline{S}$ of possibilistic interpretations and a poss-NLP $\overline{B}$, there exists a poss-NLP $\overline{P}$ such that $\overline{S} \subseteq \mathit{PSM}(\overline{B} \sqcup \overline{P})$ if and only if ${\cal T}_{\overline{B}}(\overline{I}) \sqsubseteq \overline{I}$ for each $\overline{I} \in \overline{S}$.
\end{corollary}

\begin{example}\label{exam:possNLP:B:exist}
    Let $T_{23}  = \tuple{\overline{B}, E^+, \emptyset}$ be an induction task, where $\overline{B} = \{ (r \lto , 0.3) \}$ and $E^+ = \{ \{ (p, 0.5), (r, 0.5) \}, \{ (q, 0.3), (r, 0.8) \} \}$. Then $E^+$ is incomparable and ${\cal T}_{\overline{B}}(\overline{I}) \sqsubseteq \overline{I}$ for each $\overline{I} \in E^+$. It is easy to verify $E^+ \subseteq \mathit{PSM}(\overline{B} \sqcup \mathit{PPE}(E^+))$, where $\mathit{PPE}(E^+) = \{ (p \lto \Not q, 0.5), (r \lto \Not q, 0.5), (q \lto \Not p, 0.3), (r \lto \Not p, 0.8) \}$. Namely, $\mathit{PPE}(E^+) \in \mathit{ILP_{LPoSM}}(T_{23})$ in this case. 
    
    Let $T_{24}  = \tuple{\overline{B}, E^+, \emptyset}$, where $\overline{B} = \{ (r \lto , 0.8) \}$ and $E^+ = \{ \{ (p, 0.5), (r, 0.5) \} \} $. $\mathit{ILP_{LPoSM}}(T_{24}) = \emptyset$ since ${\cal T}_{\overline{B}}(\{ (p, 0.5), (r, 0.5) \}) = \{ (r, 0.8) \} \not\sqsubseteq \{ (p, 0.5), (r, 0.5) \}$.  
\end{example}

The condition ${\cal T}_{\overline{B}}(\overline{I}) \sqsubseteq \overline{I}$ is important and often referred in our subsequent discussions. 
Given the above corollary, it makes sense to give the definition below.

\begin{definition}[Coherency]\label{def:incoherent}
Let  $\overline{B}$ be a poss-NLP.
\begin{enumerate}
    \item A possibilistic interpretation $\overline{I}$ is {\em coherent} with poss-NLP $\overline{B}$ if 
    ${\cal T}_{\overline{B}}(\overline{I}) \sqsubseteq \overline{I}$. Otherwise, $\overline{I}$ is {\em incoherent} with $\overline{B}$.
    \item A set $\overline{S}$ of possibilistic interpretations is {\em coherent} with poss-NLP $\overline{B}$ if $\overline{I}$ is coherent with $\overline{B}$ for each  $\overline{I} \in \overline{S}$. Otherwise, $\overline{S}$ is {\em incoherent} with $\overline{B}$. 
\end{enumerate}
\end{definition}

Given this definition, the next corollary is straightforward.

\begin{corollary}[Existence of a solution]\label{cor:B:pos}	
For an induction task $T = \tuple{\overline{B}, E^+, \emptyset}$, $\mathit{ILP_{LPoSM}}(T) \neq \emptyset$ if and only if $E^+$ is incomparable and $E^+$ is coherent with $\overline{B}$.
\end{corollary}

In a special case when the set of positive examples may be empty, a necessary and sufficient condition for the existence of an induction solution can be provided as follows. 

\begin{proposition}
\label{prop:task:B:neg:exist}  
Let $T = \tuple{\overline{B}, \emptyset, E^-}$ be an induction task and $\overline{\mathbb{A}}=\{\overline I\mid \overline{I}\mbox{ is a poss-interpretation with } I={\cal A}\}$. 
Then $\mathit{ILP_{LPoSM}}(T) = \emptyset$ if and only if the following three conditions are satisfied. 
\begin{enumerate}[label=(c\arabic*)]
    \item $\mathit{lfp}(T_{B^{\cal A}}) = {\cal A}$.
    \item $\overline{\mathbb{A}} - E^- \neq \overline{\mathbb{A}}$.  
    \item $\overline{\mathbb{A}} - E^- = \emptyset$, or ${\cal T}_{\overline{B}}(\overline{G}) \not\sqsubseteq \overline{G}$ for each $\overline{G} \in \overline{\mathbb{A}} - E^-$.    
\end{enumerate}
\end{proposition}

Let us look at an example.

\begin{example}\label{exam:task:B:neg:exist}
    Let $T_{25}  = \tuple{\overline{B}, \emptyset, E^-}$ be an induction task, where $\overline{B} = \{ (p \lto , 0.5), (q \lto p, 0.5) \}$ and $E^- = \{ \{ (p, 0.5), (q, 0.5) \} \}$. Then we have ${\cal A} = \{ p,q \}$, ${\cal Q} = \{ 0.5 \}$ and $0.5 \leq 0.5$. So, $\overline{\mathbb{A}} = \{ \{ (p, 0.5), (q, 0.5) \} \}$. It is clear that $\mathit{lfp}(T_{B^{\cal A}}) = {\cal A}$ and $\overline{\mathbb{A}} - E^- = \emptyset \neq \overline{\mathbb{A}}$. Thus, $\mathit{ILP_{LPoSM}}(T_{25}) = \emptyset$.   

    Let $T_{26}  = \tuple{\overline{B}, \emptyset, E^-}$ be another induction task, where $\overline{B} = \{ (p \lto , 0.8), (q \lto p, 0.5) \}$ and $E^- = \{ \{ (p, 0.8), (q, 0.5) \}$, $\{ (p, 0.8), (q, 0.8) \} \}$. 
   Note that ${\cal A} = \{ p,q \}$, ${\cal Q} = \{ 0.5, 0.8 \}$ and $0.5 \leq 0.8$. 
   Consequently, $\overline{\mathbb{A}} = \{ \{ (p, 0.5), (q, 0.5) \}, \{ (p, 0.5), (q, 0.8) \}$, $ \{ (p, 0.8), (q, 0.5) \}, \{ (p, 0.8), (q, 0.8) \} \}$. It is clear that $\mathit{lfp}(T_{B^{\cal A}}) = {\cal A}$, $\overline{\mathbb{A}} - E^- \neq \overline{\mathbb{A}}$, and ${\cal T}_{\overline{B}}(\overline{G}) \not\sqsubseteq \overline{G}$ for each $\overline{G} \in \overline{\mathbb{A}} - E^-$. Thus, $\mathit{ILP_{LPoSM}}(T_{26}) = \emptyset$.
\end{example}


By Proposition~\ref{prop:task:B:neg:exist}, if we consider only negative examples, the above three conditions provide some clue on the existence or nonexistence of induction solutions (even if an induction task contains positive examples). Thus, we have the follwoing definition.

\begin{definition}[Compatibility]\label{def:compatible}    
    A set $E^-$ of possibilistic interpretations is {\em incompatible} with a poss-NLP $\overline{B}$ if  
    \begin{enumerate}[label=(c\arabic*)]
        \item $\mathit{lfp}(T_{B^{\cal A}}) = {\cal A}$, and
        \item $\overline{\mathbb{A}} - E^- \neq \overline{\mathbb{A}}$, and 
        \item $\overline{\mathbb{A}} - E^- = \emptyset$, or ${\cal T}_{\overline{B}}(\overline{G}) \not\sqsubseteq \overline{G}$ for each $\overline{G} \in \overline{\mathbb{A}} - E^-$.
    \end{enumerate}
    Otherwise, $E^-$ is {\em compatible} with  $\overline{B}$.
\end{definition}

Putting all these considerations together, we present Theorem~\ref{thm:taskNLP:exist}. It reveals a necessary and sufficient condition for the existence of a solution for a general induction task. All indispensable relations in $\tuple{\overline{B}, E^+, E^-}$ has been formally exhibited. 

Now we are ready to present the major result in this section, which provides a necessary and sufficient condition for the existence of a solution for a general induction task. 

\begin{theorem}[Existence of a poss-NLP for an induction solution]	\label{thm:taskNLP:exist} 
For an induction task $T = \tuple{\overline{B}, E^+, E^-}$, $\mathit{ILP_{LPoSM}}(T) \neq \emptyset$ if and only if 
\begin{enumerate}[label=(C\arabic*)]
    \item $E^+$ is incomparable, and
    \item $E^+$ is coherent with $\overline{B}$, and 
    \item $E^-$ is compatible with  $\overline{B}$, and 
    \item $E^+ \cap E^- = \emptyset$.
\end{enumerate}
\end{theorem}

\begin{example}\label{exam:PAS:solutions:check}
    Let us revisit the induction tasks in Example~\ref{exam:PAS:solutions}.
    For the three induction tasks $T_1$, $T_3$ and $T_5$, the necessary conditions for the existence of a solution are satisfied. 
    Induction task $T_2$ has no solution, since $E^+$ is incoherent with $\overline{B}$. To see this, we note that ${\cal T}_{\overline{B}}(\overline{I}) = \{ (pregnancy, 1), (vomiting, 1)  \} \not\sqsubseteq \overline{I}$, where $\overline{I} = \{ (pregnancy, 0.6) \} \in E^+$.
    For induction task $T_4  = \tuple{ \emptyset, \{ \{ (p, 0.3), (q,0.3) \} \}, \{ \{ (p, 0.3), (q,0.3) \} \}}$, it has no solution since $E^+ \cap E^- \neq \emptyset$.
\end{example}

Intuitively, $E^+$ is utilized to construct the candidate induction solution. Conversely, in a general induction task $T = \tuple{\overline{B}, E^+, E^-}$, $E^-$ is employed to eliminate unintended candidate solutions covering $E^+$ though. It functions akin to a hard constraint in answer set programming~\citep{Brewka:CACM:2011}. After a candidate solution is constructed, we check whether the candidate conflicts with $E^-$. Then, either the candidate solution satisfy the requirements (return as a solution) or repeat this process.


In the following, we consider a special case of induction task, that is, when a fact rule is in the background knowledge. 

\begin{proposition}[Solution for induction task containing fact rules]	\label{prop:task:B:upper} 
    Let $T = \tuple{\overline{B}, E^+, E^-}$ be an induction task.
    If there exists $a \in {\cal A}$ such that $(a \lto , \mu) \in \overline{B}$, then 
    \begin{enumerate}[label=(\roman*)]
        \item $\mathit{ILP_{LPoSM}}(T) = \emptyset$ when there exists $\overline{e} \in E^+$ such that $(a, \mu) \notin \overline{e}$, and
        \item for any poss-NLP $\overline{H}$, $(a, \mu) \notin \overline{e}$ implies $\overline{e} \notin \mathit{PSM}(\overline{B} \sqcup \overline{H})$ where $\overline{e} \in E^-$, and  
        \item if $\overline{H} \in \mathit{ILP_{LPoSM}}(T)$ and $a \in \Head(H)$, then there exists $\overline{K} \in \mathit{ILP_{LPoSM}}(T)$ such that $\vert K \vert < \vert H \vert$. 
    \end{enumerate}
\end{proposition}

Intuitively, for a special induction task $T = \tuple{\overline{B}, E^+, E^-}$ whose background knowledge base contains a fact rule $(a \lto , \mu)$, the construction of an induction solution is significantly simplified. 

\begin{enumerate}[label=(\roman*)]
    \item The induction task $T$ has no solution if we see a positive example without $(a, \mu)$. 
    \item For each $\overline{e} \in E^-$ such that $(a, \mu) \notin \overline{e}$, $\overline{e}$ can be skipped during the construction of a solution. 
    \item When constructing a minimal solution of $T$ (i.e., with a minimum number of rules), we can ignore the rules whose head is $a$.
\end{enumerate}

Thus, Proposition~\ref{prop:task:B:upper} provides a way for computing a minimal solution (that is, the number of rules in the induction solution is minimal).

\begin{definition}[Minimal solution]\label{def:minimal:solution} 
    Given an induction task $T = \tuple{ \overline{B}, E^+, E^-}$ and a poss-NLP $\overline{H}$ such that $\overline{H} \in \mathit{ILP_{LPoSM}}(T)$, $\overline{H}$ is called a \emph{minimal} solution if there exists no $\overline{K} \in \mathit{ILP_{LPoSM}}(T)$ such that $\vert K \vert < \vert H \vert$.
\end{definition}

Theorem~\ref{thm:taskNLP:exist} provides a method (Algorithm~\ref{alg:Existence}) for checking the existence of a \ysm{minimal}{(minimal)} solution
\ysa{ of an induction task. It is not difficult to verify that its time complexity is $O(n^c)$ where $n$ is the size of its input and $c$ is a constant}.

\begin{algorithm}[t]
	\caption{{Existence}$(\overline{B}, E^+, E^-)$}\label{alg:Existence}
	\begin{algorithmic}[1]
		\Statex \textbf{Input:} An induction task $T = \tuple{\overline{B}, E^+, E^-}$.
		\Statex \textbf{Output:} {\bf true} if $\mathit{ILP_{LPoSM}}(T) \neq \emptyset$, and {\bf false} otherwise. 
        \State {\bf if} $E^+$ is comparable {\bf then} {\bf return} {\bf false}
		\State {\bf if} $E^+$ is incoherent with $\overline{B}$ {\bf then} {\bf return} {\bf false}
		\State {\bf if} $E^-$ is incompatible with  $\overline{B}$ {\bf then} {\bf return} {\bf false}
		\State {\bf if} $E^+ \cap E^- \neq \emptyset$ {\bf then return} {\bf false}
        \State {\bf return} {\bf true}
	\end{algorithmic}
\end{algorithm}
\ysm{Note that the size of an induction task $T(\overline{B}, E^+, E^-)$ can be defined as $\size{T} = \size{\overline{B}} + \size{E^+} +\size{E^-}$, here $\size{\overline{B}} = |B|\cdot |{\cal Q}|$, $\size{E^+} = \Sigma_{\overline{I}\in E^+} |\overline{I}|$ and $\size{E^-}= \Sigma_{\overline{I}\in E^-} |\overline{I}|$ (the numbers of atoms in $E^+$ and $E^-$, respectively). 
Let $\size{T} = n$, then each of the four steps in the above algorithm can be done in polynomial time. Thus, the complexity of Algorithm~\ref{alg:Existence} is polynomial in the size of the input induction task. }

\section{Algorithms for Computing Induction Solutions}\label{sec:learning:algorithm}
Theorem~\ref{thm:taskNLP:exist} not only provides a sufficient and necessary condition for the existence of an induction solution, but its proof also paves a way for constructing a particular solution. Thus, the method for computing an induction solution is formulated as Algorithm~\ref{alg:ILPSM1}.

\begin{algorithm}[t]
	\caption{{\ILPSM}$(\overline{B}, E^+, E^-)$}\label{alg:ILPSM1}
	\begin{algorithmic}[1]
		\Statex \textbf{Input:} An induction task $T = \tuple{\overline{B}, E^+, E^-}$.
		\Statex \textbf{Output:} A solution of $T$ if $\mathit{ILP_{LPoSM}}(T) \neq \emptyset$, and {\bf fail} otherwise. 
        \State {\bf if} $\mathit{Existence}(\overline{B}, E^+, E^-)$ is {\bf false} {\bf then} {\bf return} {\bf fail}
        \State $\overline{H} \lto \emptyset$
        \State {\bf if} $E^+ \neq \emptyset$  {\bf then}
        \State \quad $\overline{H} \lto \mathit{PPE}(E^+)$        
        \State \quad $\overline{E} \lto \{ \overline{e} \in E^- \mid {\cal T}_{\overline{B} \sqcup \overline{H}}(\overline{e}) \sqsubseteq \overline{e} \}$
        \State \quad $\overline{H} \lto \overline{H} \sqcup \mathit{PNE}(\overline{E},E^+) - \overline{B}$
        \State {\bf else}
        \State \quad {\bf if} $\mathit{lfp}(T_{B^{\cal A}}) \neq {\cal A}$ or $\overline{\mathbb{A}} - E^- =\overline{\mathbb{A}}$  {\bf then} 
        \State \quad \quad $\overline{H} \lto \mathit{PNE}(E^-,\emptyset)  - \overline{B}$
        \State \quad {\bf else} 
        \State \quad \quad Find a poss-interpretation $\overline{G}$ in $\overline{\mathbb{A}} - E^-$  s.~t. ${\cal T}_{\overline{B}}(\overline{G}) \sqsubseteq \overline{G}$
        \State \quad \quad $\overline{H} \lto \mathit{PPE}(\overline{G})$ 
        \State {\bf return} $\overline{H}$
	\end{algorithmic}
\end{algorithm}


If the set of positive examples is not empty, $\mathit{PPE}(E^+)$ in line (4) and $\mathit{PNE}(\overline{E},E^+)$ in line (6) respectively ensure that condition~\eqref{eq:G1} and condition~\eqref{eq:G2} in Definition~\ref{def:ILT} are achieved. 
By Proposition~\ref{prop:possNLP:interpretation:exist}, negative example satisfying ${\cal T}_{\overline{B} \sqcup \overline{H}}(\overline{e}) \not\sqsubseteq \overline{e}$ does not require further consideration. Thus, such negative examples are eliminated in line (6).
When $E^+ \neq \emptyset$, condition \eqref{eq:G2} in Definition~\ref{def:ILT} is guaranteed by either $\mathit{PNE}(E^-,\emptyset)$ in line (9) or $\mathit{PPE}(\overline{G})$ in line (12). 

Consider the following example. 
\begin{example}\label{exam:alg1:running}
    Let $T_{31}  = \tuple{\overline{B}, E^+, E^-}$ be an induction task, where $\overline{B} = \{ (p \lto q, 0.3), (q \lto \Not r, 0.5) \}$, $E^+ = \{ \{ (r, 0.3) \} \}$ and $E^- = \{ \{ (q, 0.3), (r, 0.5) \}, \{ (p, 0.3), (q, 0.5) \} \}$. Then ${\cal A} = \{ p,q,r \}$, ${\cal Q} = \{ 0.3, 0.5 \}$ and $0.3 \leq 0.5$. 
   For the input $T_{31}$, we can trace the algorithm as follows.
    \begin{itemize}
        \item In line (1), $\mathit{Existence}(\overline{B}, E^+, E^-)$ returns {\bf true}.
        \item Line (4) is executed as $E^+ \neq \emptyset$. Then $\overline{H}$ is assigned the poss-NLP $\mathit{PPE}(E^+) = \{ (r \lto \Not p, \Not q, 0.3) \}$.
        \item In line (5), $\overline{E} = \{ (p, 0.3), (q, 0.5) \} $. 
        \item In line (6), let $\overline{H} = \{ (r \lto \Not p, \Not q, 0.3), (r \lto p, q, \Not r, 0.3) \}$ as $\mathit{PNE}(\overline{E},E^+) = \mathit{PNE}(\{ (p, 0.3), (q, 0.5) \})$ is assigned $\{ (r \lto p, q, \Not r, 0.5) \}$. 
        \item In line (13), the solution $\overline{H}  = \{ (r \lto \Not p, \Not q, 0.3), (r \lto p, q, \Not r, 0.5) \}$ is returned in the end. 
    \end{itemize}
    It is easy to see that $\{ (r \lto \Not p, \Not q, 0.3), (r \lto p, q, \Not r, 0.5) \} \in \mathit{ILP_{LPoSM}}(T_{31})$.    
\end{example}

We note that, in the definition of $\mathit{PNE}(\overline{E},E^+)$ in Equation~\ref{eq:PNE}, the head $x$ of the rule $(x \lto I, \Not ({\cal A} - I), \mu)$ is an arbitrary element in ${\cal A} - I$, which is sufficient to guarantee that the output of Algorithm~\ref{alg:ILPSM1} is an induction solution. However, for different choices of the head $x$, we may come up with different induction solutions.



\begin{proposition}[Correctness of algorithm ILPSM]\label{prop:alg1:correct}
    For an induction task $T = \tuple{\overline{B}, E^+, E^-}$, algorithm {\ILPSM} returns {\bf fail} when $\mathit{ILP_{LPoSM}}(T) = \emptyset$. Otherwise, it returns a solution of $T$.
\end{proposition}

We note that checking the existence of induction solutions and the subset relation ${\cal T}_{\overline{B} \sqcup \overline{H}}(\overline{e}) \sqsubseteq \overline{e}$ can be done in polynomial time. Also, the tasks of computing $\mathit{PPE}(E^+)$, $\mathit{PNE}(\overline{E},E^+)$ are also in polynomial time. The task of line 11 \ysm{requires to `guess' a poss-interpretation $\overline{G}$ and thus the complexity of this subtask is NP-complete in the worst case. Therefore, the complexity of Algorithm~\ref{alg:ILPSM1} is NP-complete in the worst case.}{can be achieved by enumerating $\overline G\in \overline{\mathbb A}-E^-$ and checking if ${\cal T}_{\overline B}(\overline G)\sqsubseteq \overline G$, which is bounded by $O(k2^m)$ where $m=|\overline{\mathbb  A}|$ and $k=|B|+|G|$ (we assume ${\cal T}_{\overline B}(\overline G)$ is computable in linear time). Therefore, the time complexity of this algorithm is $O(n2^n)$ where $n$ is the size of its input. }





While algorithm {\ILPSM} always returns an induction solution if it exists, the solution may not minimal \wrt the number of rules in the solution. 
We note that all rules in $\mathit{PPE}(E^+)$ has are negative (i.~e., no positive body atoms), which causes some minimal solutions are missed in the construction process.
For instance in Example~\ref{exam:PAS:solutions}, the minimal solution $\overline{H_1}$ will be missed by this algorithm since the positive body of rule $(medA \lto vomiting, \Not medB, 1)$ is not empty. 

In the rest of this section, we introduce an algorithm that computes a minimal induction solution. This algorithm is completely different from the previous one and we need to first investigate some properties of minimal induction solutions.

Let  $\alpha\in \cal{Q}$ and  $\star\in\{>,\ge, =\}$. The {\em  $\alpha^\star$-relevant atoms} of a possibilistic interpretation $\overline{I}$, written as ${RA}^\star(\overline{I},\alpha)$, is defined by ${RA}^\star(\overline{I},\alpha)=\{ p \mid (p,\beta)\in\overline{I}, \beta\star\alpha\}$. Note that when $\star$ is the equality sign $=$, ${RA}^= (\overline{I},\alpha)$ is empty if there is no atom $p$ such that $(p,\beta)\in\overline{I}$. So, the set ${RA}^=(\overline{I},\alpha)$ is not always be $\{p\}$.

Given a collection $A$ of sets, a set $B$ is a {\em hitting set} of $A$ if $X \cap B \neq \emptyset$ for each $X \in A$. A hitting set $B$ of $A$ is called {\em minimal} if there exists no $C \subset B$ such that $C$ is still a hitting set of $A$. We use $\mathit{SMHS}(A)$ to denote the set of all minimal hitting sets of $A$. 
Now we can define the notion of positive solution spaces in the following definition.

\begin{definition}[Positive solution space]\label{def:pos:SS}
      Let $\overline{I}$ be a possibilistic interpretation and $\epsilon=(p,\alpha)$ be a poss-atom in $\overline{I}$. Then
      \begin{enumerate}
          \item The {\em positive solution space} of $\epsilon$ under $\overline{I}$, 
      written as  $S^+(\overline{I}, \epsilon)$, is the following set of possibilistic rules:
            \begin{equation}\label{eq:positive:SS}
         \{(p\lto \Pos, \Not \Neg, \; \beta)\mid \Pos\subseteq RA^\ge(\overline{I},\alpha), \Neg\subseteq {\cal A}-I, \text{ and for all }\beta \in W\}   
      \end{equation}
      where 
      \[W = 
        \left\{
          \begin{array}{ll}
            \{ \alpha \}        &  , \hbox{$\Pos\cap RA^=(\overline{I},\alpha)=\emptyset$;} \\
            \{ \gamma \mid \gamma \in {\cal Q}, \gamma \geq \alpha \}  & ,  \hbox{otherwise.}
          \end{array}
        \right.
      \]
        \item The {\em positive solution space} of $\overline{I}$, written as  $S^+(\overline{I})$, is defined as $\mathit{SMHS}(\{ S^+(\overline{I}, \epsilon) \mid \epsilon \in \overline{I}\})$.
          \end{enumerate}
\end{definition}

\ysa{ The next propositions shows that $|P\cap S^+(\overline I,\epsilon)|=1$ for each $P\in S^+(I)$ and $\epsilon\in \overline I$.}
\ysm{In this definition, the computation of $\mathit{SMHS}(.)$ is usually quite complex. 
Therefore, we provide an easy construction to get the $S^+(\overline{I})$ via the the following proposition.         
By default, we compute $S^+(\overline{I})$ according to the construction in Proposition~\ref{prop:SS:construction}. }

\begin{proposition}
\label{prop:SS:construction}
Given a possibilistic interpretation $\overline{I}$, $S^+(\overline{I}) = \{ P \mid P \subseteq \cup_{\epsilon \in \overline{I}} S^+(\overline{I}, \epsilon) \mbox{ and } \forall \epsilon \in \overline{I}, \vert P \cap S^+(\overline{I}, \epsilon) \vert = 1 \}$. 
\end{proposition}

Intuitively, $S^+(\overline{I})$ is a collection of all poss-NLPs containing exactly $\vert I \vert$ rules, to which each $\epsilon \in \overline{I}$ provides exactly one rule from $S^+(\overline{I}, \epsilon)$.  
\ysa{Please note that the size of $S^+(\overline I,\epsilon)$ is bounded by $O(k2^n)$ where $n=|{\cal A}|$ and $k=|{\cal Q}|$ since there are possibly $2^n$ number of different $\Pos$ and $\Neg$ in Equation~\ref{eq:positive:SS}. Thus, the size of $S^+(\overline I)$ is bounded by $O((n2^n)^m)=O(n^m\times 2^{nm})$ where $n=|\overline I|+|{\cal Q}|$ and $m=|\overline I|$.}

Let us look at some examples of $S^+(\overline{I}, \epsilon)$ and $S^+(\overline{I})$.


\begin{example}\label{exam:pos:SS}
    Given ${\cal A} = \{ p,q,r \}$, ${\cal Q} = \{ 0.3, 0.5 \}$ and $0.3 \leq 0.5$, let $\overline{I} = \{ (r, 0.3) \}$ and $\overline{J} = \{ (q, 0.5), (r, 0.3) \}$. Then 
    \begin{itemize}
        \item $S^+(\overline{I}, (r, 0.3)) = \{ (r \lto r, 0.3), (r \lto r, 0.5), (r \lto r, \Not p, 0.3), (r \lto r, \Not p, 0.5), (r \lto r, \Not q, 0.3), (r \lto r, \Not q, 0.5), (r \lto r, \Not p, \Not q, 0.3), (r \lto r, \Not q, \Not q, 0.5), (r \lto , 0.3), (r \lto \Not p, 0.3), (r \lto \Not q, 0.3), (r \lto \Not p, \Not q, 0.3)$, 
        \item $S^+(\overline{J}, (q, 0.5)) = \{ (q \lto q, 0.5), (q \lto q, \Not p,  0.5), (q \lto , 0.5), (q \lto \Not p, 0.5) \}$, 
        \item $S^+(\overline{J}, (r, 0.3)) = \{ (r \lto r, 0.3), (r \lto r, 0.5), (r \lto r, q, 0.3), (r \lto r, q, 0.5), (r \lto , 0.3), (r \lto q, 0.3), (r \lto r, \Not p, 0.3), (r \lto r, \Not p, 0.5), (r \lto r, q, \Not p, 0.3), (r \lto r, q, \Not p, 0.5), (r \lto \Not p, 0.3), (r \lto q, \Not p, 0.3) \}$. 
    \end{itemize}
    As $\overline{I}$ contains exactly one element, we have $S^+(\overline{I}) = \{ \{ \overline{r} \} \mid \overline{r} \in S^+(\overline{I}, (r, 0.3)) \}$. As $\vert S^+(\overline{J}, (q, 0.5)) \vert = 4$ and $\vert S^+(\overline{J}, (r, 0.3)) \vert = 12$, $S^+(\overline{J})$ has $4 \times 12 = 48$ elements. For instance, $\{ (q \lto q, 0.5), (r \lto r, 0.3) \} \in S^+(\overline{J})$ and $\{ (q \lto q, 0.5), (r \lto r, 0.5) \} \in S^+(\overline{J})$.
\end{example}

The notion of positive solution space is helpful for guiding the search process for a minimal solution. 
The below lemma reveals a relation between the least fixpoint of a poss-definite program and the positive solution spaces.


\begin{lemma}[Positive solution space for the least fixpoint]\label{lm:SM:LFP}
    Let possibilistic interpretation $\overline{I}$ be the least fixpoint of a possibilistic definite logic program $\overline{P}$. There exists a program $\overline{H} \subseteq \overline{P}$ such that $\overline{H}  \in S^+(\overline{I})$ and $H$ is grounded.
\end{lemma}

This result can be generalised to general poss-NLPs. 

\begin{proposition}[Positive solution space for a poss-stable model]\label{prop:SS:exists}	
    If $\overline{I} \in \mathit{PSM}(\overline{P})$ for a poss-NLP $\overline{P}$ and a possibilistic interpretation $\overline{I}$, then there exists $\overline{H} \subseteq \overline{P}$ such that $\overline{H}  \in S^+(\overline{I})$ and $H^I$ is grounded.
\end{proposition}

By Definition~\ref{def:pos:SS}, it is straightforward to get $\overline{H}  \in S^+(\overline{I})$ in Proposition~\ref{prop:SS:exists}. 
The next step is to check whether $H^I$ is grounded. 
As each stable model has at least one acyclic support graph~\citep{cabalar2023model}\footnote{The dependency graph~\citep{li2021graph} of a NLP is defined on its literals s.t. there exists a positive (resp. negative) edge from $p$ to $q$ if $p$ appears positively resp. negatively) in the body of a rule with head $q$. Positive loops are loops with no negative edge.}, Proposition~\ref{prop:SS:noloop} provides a default method for checking the groundness of $H^I$ via the positive loop in its support graph.

\begin{proposition}[Loops in grounded program]\label{prop:SS:noloop}	
    Let $\overline{I}$ be a possibilistic interpretation and $\overline{P}$ be a poss-NLP such that $\overline{P} \in S^+(\overline{I})$. Then $P^I$ is grounded if and only if the dependency graph of $P$ does not have a positive loop.
\end{proposition}

To characterise the impact of negative examples on minimal induction solutions, we introduce the notion of negative solution spaces.
Informally, the negative solution space collects all potential rules that prevent $\overline{I} \in E^-$ from being a poss-stable model. While the positive solution space uses a minimal hitting set to ensure the outcome of each possibilistic atom in a poss-stable model, only one rule in the negative solution space (in a solution) can block the undesired poss-stable model. A negative solution space comprises two parts: one part leads to an existing interpretation bounded with an unacceptably higher necessity; and the other leads directly to an unacceptable new interpretation. 


\begin{definition}[Negative solution space]\label{def:neg:SS}
      Let $\overline{I}$ be a possibilistic interpretation and $\epsilon=(p,\alpha) \in \overline{I}$. 
      \begin{enumerate}
          \item The {\em negative solution space} of $\epsilon$ under $\overline{I}$, 
            written as  $S^-(\overline{I}, \epsilon)$, is the following set of possibilistic rules:
            \begin{equation}\label{eq:negative:SS}
            \{(p\lto \Pos, \Not \Neg, \quad \beta)\mid \Pos\subseteq RA^>(\overline{I},\alpha), \Neg\subseteq {\cal A}-I, \beta \in {\cal Q}, \beta > \alpha \}.   
            \end{equation}
        \item The {\em negative solution space} of $\overline{I}$, written as  $S^-(\overline{I})$, is defined as
            \begin{equation}
            (\bigcup_{\epsilon\in \overline{I}}S^-(\overline{I}, \epsilon))  \cup 
            \{(p\lto \Pos, \Not \Neg, \quad \beta)\mid p \notin I, \Pos\subseteq I, \Neg\subseteq {\cal A}-I, \beta \in {\cal Q} \} 
            \end{equation}
      \end{enumerate}
\end{definition}

\ysa{Please note that the size of $S^-(\overline I,\epsilon)$ is bounded by $O(k2^n)$ where $n=|{\cal A}|$ and $k=|{\cal Q}|$ since there are $2^n$ number of different $\Pos$ and $\Neg$. Thus, the size of $S^-(\overline I)$ is bounded by $O(n2^n)$ where $n=|\overline I|+|{\cal Q}|$.}


\begin{example}\label{exam:neg:SS}
    Let ${\cal A} = \{ p,q,r \}$, ${\cal Q} = \{ 0.3, 0.5 \}$ and $0.3 \leq 0.5$. For $\overline{I} = \{ (r, 0.3) \}$, we have $S^-(\overline{I}, (r, 0.3)) = \{ (r \lto ,0.5), (r \lto \Not p, 0.5), (r \lto \Not q, 0.5), (r \lto \Not p, \Not q, 0.5) \}$. 
    Moreover, $S^-(\overline{I}) = S^-(\overline{I}, (r, 0.3)) \cup \{ (a \lto r, \beta), (a \lto r, \Not p, \beta), (a \lto r, \Not q, \beta), (a \lto r, \Not p, \Not q, \beta), (a \lto ,\beta), (a \lto \Not p, \beta), (a \lto \Not q, \beta), (a \lto \Not p, \Not q, \beta) \}$ where $\beta \in \{ 0.3, 0.5 \}$ and $a \in \{ p,q \}$.

     For $\overline{J} = \{ (p, 0.3), (q, 0.5) \}$, $S^-(\overline{J})$ can be constructed similarly.  For instance, we have $\{ (p \lto \Not r, 0.5), (p \lto q, 0.5), (r \lto , 0.5), (r \lto , 0.3) \} \subseteq S^-(\overline{J})$.
\end{example}

\begin{lemma}[Incoherent interpretation for a poss-NLP]\label{lm:neg:SS}	
For a possibilistic interpretation $\overline{I}$ and poss-NLP $\overline{P}$, $\overline{I}$ is incoherent for $\overline{P}$ if and only if $\overline{P} \cap S^-(\overline{I}) \neq \emptyset$.    
\end{lemma}

Lemma~\ref{lm:neg:SS} connects the incoherency and the negative solution space. 
Combining both $S^-(\overline{I})$ and $S^+(\overline{I})$, we have the following result, which provides a characterisation of poss-stable models. Intuitively, $\overline{I}$ being a poss-stable model of $\overline{P}$ if and only if (\romannumeral 1) $\overline{P}$ does not contain any superfluous support for $\overline{I}$, and  (\romannumeral 2) $\overline{P}$ provides sufficient support for $\overline{I}$. 

\begin{theorem}[Conditions for a poss-stable model]\label{thm:SS:iff}	
    Let $\overline{P}$ be a poss-NLP, $\overline{I}$ a possibilistic interpretation. $\overline{I} \in \mathit{PSM}(\overline{P})$ if and only if
    \begin{enumerate}[label=(c\arabic*)]
        \item $\overline{P} \cap S^-(\overline{I}) = \emptyset$, and 
        \item there exists $\overline{H} \subseteq \overline{P}$ such that $\overline{H}  \in S^+(\overline{I})$ and $H^I$ is grounded. 
    \end{enumerate}
\end{theorem}


By Theorem~\ref{thm:SS:iff}, we have the following corollary. 

\begin{corollary}[Negative examples for a solution]\label{cor:neg:conditions}	
Let $T = \tuple{\overline{B}, E^+, E^-}$ be an induction task and $\overline{H}$ be a poss-NLP such that $\overline{H} \in \mathit{ILP_{LPoSM}}(T)$. If $\overline{P} = \overline{B} \sqcup \overline{H}$, then the following two statements hold.
\begin{enumerate}[label=(\roman*)]
    \item $\overline{I} \in E^-$ if and only if at least one of the two conditions (c1) or (c2) in Theorem~\ref{thm:SS:iff} does not hold.
    \item The condition (c2) in Theorem~\ref{thm:SS:iff} does not hold when $\overline{I} = \{ (a,\mu) \mid a \in {\cal A} \} \in E^-$ where $\mu$ is the supremum of ${\cal Q}$ in $({\cal Q},\le)$.
\end{enumerate}
\end{corollary}

Proposition~\ref{prop:MS:rule} shows that every rule in a minimal solution must directly affect at least a positive example or a negative example. Besides, the necessity of each rule has a certain degree of freedom within the corresponding solution spaces.

\begin{proposition}[Rules in a minimal solution]\label{prop:MS:rule}
    Let $\overline{H}$ be a minimal solution for an induction task $T = \tuple{\overline{B}, E^+, E^-}$. 
    For any $\overline{r} \in \overline{H}$
    \begin{enumerate}[label=(\roman*)]
        \item there exists $\overline{I} \in E^+$ and $\epsilon \in \overline{I}$ such that $\overline{r} \in S^+(\overline{I}, \epsilon)$, or 
        \item there exists $\overline{J} \in E^-$ such that $\overline{r} \in S^-(\overline{J})$.   
    \end{enumerate}
\end{proposition}


We are now ready to present algorithm {\ILPSMmin}, Algorithm~\ref{alg:ILPSM2}, for computing a minimal solution for an induction task (i.~e. with a minimum number of rules), if there is a solution. \ysa{Please note that the algorithm does not directly deal with atoms. The order in which the (positive and negative) examples are processed is not very relevant to its efficiency.}

\begin{algorithm}[t]
	\caption{{\ILPSMmin}$(\overline{B}, E^+, E^-)$}\label{alg:ILPSM2}
	\begin{algorithmic}[1]
		\Statex \textbf{Input:} An induction task $T = \tuple{\overline{B}, E^+, E^-}$.
		\Statex \textbf{Output:} A minimal solution of $T$ if $\mathit{ILP_{LPoSM}}(T) \neq \emptyset$, and {\bf fail} otherwise. 
		\State {\bf if} $\mathit{Existence}(\overline{B}, E^+, E^-)$ is {\bf false} {\bf then} {\bf return} {\bf fail}
		\State $\overline{H}\lto \emptyset$; $i \lto 0$; $norm \lto \vert E^+ \vert \cdot \vert {\cal A} \vert + \vert E^- \vert$; $seeds \lto \{ \emptyset \}$;
		\State $blacklist \lto \{ \overline{r} \in S^-(\overline{I}) \mid \overline{I} \in E^+ \}$ 
		\State {\bf foreach} $\overline{I} \in E^+$
        \State \quad $i \lto i + 1$
        \State \quad $W_i \lto \{ \overline{P} \in S^+(\overline{I}) \mid \overline{P} \cap blacklist = \emptyset$, $P^I$ is grounded $\} $ \ysd{\quad // Construct the $S^+(\overline{I})$ according to Proposition~\ref{prop:SS:construction} and then check the groundness according to Proposition~\ref{prop:SS:noloop}.} 
        \State {\bf if} $E^+ \neq \emptyset$ {\bf then} 
        \State \quad $seeds \lto \{ \sqcup_{1 \leq k \leq \vert E^+ \vert} Y_{k} \mid Y_{k} \in W_k \}$ \ysd{\quad // If $W_k=\emptyset$, take $Y_k=\emptyset$.}
		\State {\bf foreach} $X \in seeds$
		\State \quad {\bf if} no $e \in E^-$ is a poss-stable model of  $\overline{B} \sqcup X$ {\bf then}
		\State \quad \quad {\bf if} $\vert X - \overline{B} \vert < norm$ {\bf then}
		\State \quad \quad \quad $\overline{H} \lto X - \overline{B}$; $norm \lto \vert \overline{H} \vert$;
		\State \quad {\bf else}
		\State \quad \quad $\overline{E} \lto \{ \overline{e} \in E^- \mid \overline{e} \not\parallel E^+, {\cal T}_{\overline{B} \sqcup X}(\overline{e}) \sqsubseteq \overline{e} \}$  \ysd{\quad // \hbl{$\overline{e} \not\parallel E^+$ denotes $\forall \overline{I}\in E^+, (e \not\subseteq I) \land (e \not\supseteq I)$.}        }
		\State \quad \quad $whitelist \lto \{ \overline{r} \mid \overline{I} \in \overline{E}, \overline{r} \in S^-(\overline{I}), \overline{r} \notin blacklist \}$
		\State \quad \quad {\bf for} $j \lto 1$ to $\vert \overline{E} \vert$
		\State \quad \quad \quad $patches \lto \{ \overline{P} \subseteq whitelist \mid \vert \overline{P} \vert = j \}$		
		\State \quad \quad \quad {\bf foreach} $\overline{P} \in patches$ s.t. $\vert X \sqcup \overline{P} - \overline{B} \vert < norm$
		\State \quad \quad \quad \quad {\bf if}  no $e \in \overline{E}$ is a poss-stable model of  $\overline{B} \sqcup X \sqcup \overline{P}$ {\bf then}
		\State \quad \quad \quad \quad \quad $\overline{H} \lto X \sqcup \overline{P} - \overline{B}$; $norm \lto \vert \overline{H} \vert$;
		\State {\bf return} $\overline{H}$ 
	\end{algorithmic}
\end{algorithm}

The algorithm {\ILPSMmin} first collects poss-programs that cover positive examples in $E^+$ in lines 4-8 and then in lines 9-20, the resulting poss-programs are checked to make sure that no negative examples in $E^-$ are covered and if yes, they are expanded and the minimality is guaranteed. 

Below is a legend for the variables used in the algorithm.
\begin{itemize}[label=-]
	\item $\overline{H}$: The minimal solution returned by the algorithm if it exists. 
	\item $norm$: The number of rules in the minimal solution.  
	\item $seeds$: The set of poss-NLPs supporting the set $E^+$ of positive examples. When $E^+ = \emptyset$, the $seeds$ comprises $\emptyset$ so that the following analysis of the $E^-$ can go on. 
	\item $blacklist$: Rules in those poss-NLPs that make a positive example uncovered.
	\item $W_i$: The set of poss-NLPs supporting the $i$-th positive example. 
	\item $whitelist$: Rules in poss-NLPs that cover a negative example in $\overline{E}$.
	\item $patches$: A candidate patch for a preliminary hypothesis $X$ to block the negative examples that are covered. 
\end{itemize}




We briefly explain each step in the algorithm as follows.
\begin{itemize}[label=-]
	\item Line 1: Checking the existence of an induction solution;
	\item Line 2: Initialize the global variables (especially, $norm$ is set to an upper bound of the rule numbers in induction solutions);
	\item Line 3: According to the negative solution space, find a set of rules that cannot appear in the induction solution and assign the set to the variable $blacklist$;
	\item Lines 4-6: Iterate through each positive example $\overline{I}$ in $E^+$, ensuring that each positive example corresponds to a set $W_i$ of poss-NLPs that covers $\overline{I}$ with the minimum number of rules; 
	\item Lines 7-8: If the set of positive examples is not empty, update the $seeds$ by combining poss-programs that cover each positive example into a poss-program that covers the whole $E^+$;  
	\item Line 9: Iterate through each poss-NLP $X$ in $seeds$;
	\item Lines 10-12: If the coverage requirements of the negative examples have been met, then $X - \overline{B}$ is an induction solution; if this induction solution has fewer rules than all other induction solutions that have been examined, set this solution as the new candidate for the minimal solution;
	\item Line 13: If the coverage requirements of the negative examples have not been met, iteratively add rules in $whitelist$ to the candidate solution in lines 14-20;
	\item Line 14: Select negative examples that violate conditions \wrt the candidate solution.
	\item Line 15: Based in the negative examples obtained in line 14, collect those rules that can potentially be added to the candidate solution to construct a solution finally.
	\item Line 16: Iterate through the integer $j$ from $1$ to $\vert \overline{E} \vert$.
	\item Line 17: Pick up only those poss-programs having exactly $j$ rules from the $whitelist$;
	\item Line 18: Check only those poss-programs in $patches$ whose number of rules is less than the number of rules of the current candidate (otherwise, it is not minimal even if it can leads to a solution).
	\item Line 19-20: If the condition for $E^-$ in the definition of induction solutions is satisfied, then the current candidate solution is a solution. Moreover, if the number of rules in this solution is less than the value in $norm$, the new solution is set as the latest candidate for a minimal solution.   
	\item Line 21: Return a minimal solution $\overline{H}$.
\end{itemize}

Note that the operator $\sqcup$ can not be replaced by $\cup$ in the algorithm. 
Otherwise, the algorithm may be incorrect for some inputs.
To see this, let us look at the following example.

\begin{example}
    Consider the induction task $T = \tuple{\emptyset, E^+, \emptyset}$, where $E^+ = \{ \overline{I_1}, \overline{I_2} \} = \{ \{ (p, 0.3), (q, 0.3), (r, 0.1)  \}, \{ (p, 0.3), (q, 0.3), (s, 0.5)  \} \}$. 
    According to the algorithm {\ILPSMmin}, we can take  
    \begin{itemize}
        \item $Y_1 = \{ (p \lto , 0.3), (q \lto p, 0.3), (r \lto \Not s, 0.1) \} \in W_1$, and 
        \item $Y_2 = \{ (p \lto , 0.3), (q \lto p, 0.5), (s \lto \Not r, 0.5) \} \in W_2$.
    \end{itemize}
    Let $X_{sqcup} = Y_1 \sqcup Y_2$ and $X_{cup} = Y_1 \cup Y_2$. 
    Then we have 
    \begin{itemize}
        \item $X_{sqcup} = \{ (p \lto , 0.3), (q \lto p, 0.5), (r \lto \Not s, 0.1), (s \lto \Not r, 0.5) \}$, and 
        \item $X_{cup} = \{ (p \lto , 0.3), (q \lto p, 0.3), (q \lto p, 0.5), (r \lto \Not s, 0.1), (s \lto \Not r, 0.5) \}$.
    \end{itemize}
    In the algorithm, we pick up $X_{sqcup} \in seeds$ rather than $X_{cup} \in seeds$, since the rule $(q \lto p, 0.3) \in X_{cup}$ is redundant for the minimal solution at last. 
    Besides, $X_{cup}$ is not a poss-NLP, because the classic rule $q \lto p.$ can be only occur at most once in a poss-NLP. 
\end{example}


\ysm{We note that in the algorithm, we examine the complexity of the following subtasks, since it is obvious that other steps can be done in polynomial time or in NP.  }
{In the following we briefly analyze the time complexity of Algorithm~\ref{alg:ILPSM2}. Assume that every elements of $\cal A$ and $\cal Q$ have occurred in its input. Let $n=|\overline B|+\sum_{\overline I\in E^+\cup E^-}|I|$. }
\begin{enumerate}[label=-]
    \item Construct $blacklist$: The search space for constructing $S^-(\overline{I},\epsilon)$ is exponential and thus the task of constructing $blacklist$ is exponential \ysm{ (intuitively, it would be higher than NP-complete if P$\neq$ NP). As a result, the construction of $seeds$ is also exponential}{ $O(n^22^n)$ since $S^-(\overline I)$ is bounded by $O(n2^n)$}.
    \ysa{
    \item The first loop (lines 4-6) is bounded by $O(n^{n+1}2^{n^2})$, which is thus the upper bound of $|\mathit{seeds}|$, since $S^+(\overline I)$ is bounded by $O(n^n2^{n^2})$.
    }
    \item In lines 10 and 19, for each poss-interpretation in $\overline{E}$, a poss-stable model needs to be computed, which is \ysm{NP-complete}{in exponential time $O(2^n)$} in the worst case \ysa{ since the problem of deciding if a poss-NLP has a poss-stable model is NP-complete~\citep{nicolas2006possibilistic}}. \ysa{Note that $|\overline E|$ at line 14 is bounded by $n$, the size of $\mathit{whitelist}$ at line 15 is bounded by $O(n^22^n)$ and the size of $\mathit{patches}$ at line 17 is bounded by $O(2^{n^22^n})$. Therefore, the loop (lines 16-20) is bounded by $O(n\times 2^{n^22^n}\times n2^n)=O(n^22^{n^22^n+n})$. Thus, the loop (lines 9-20) is bounded by $O(n^n2^{n^2}\times n^22^{n^22^n+n})= O(n^{n+2}2^{n^2(2^n+1)+n})$.}
    \ysd{There are several different tasks for reasoning in possibilistic logics, for which the complexity is very high in most cases~\citep{Lang1997}. But the problem of computing a poss-stable model considered in our paper is \ysm{NP-complete~\citep{nicolas2006possibilistic}.}{in exponential time since the problem of deciding if a poss-NLP has a poss-stable model is NP-complete~\citep{nicolas2006possibilistic}.} }
\end{enumerate}

Thus, the overall \ysa{time} complexity of Algorithm~\ref{alg:ILPSM2} is  \ysm{exponential and beyond NP if the polynomial hierarchy does not collapse}{bounded by $O(n^n2^{n^2}+n^{n+2}2^{n^2(2^n+1)+n})=O(n^{n+2}2^{n^2(2^n+1)+n})$}.

Example~\ref{exam:alg2:running} illustrates the main process when given the same induction task as in Example~\ref{exam:alg1:running}, allowing us to discern the differences between these two algorithms.

\begin{example}[Continued from Examples~\ref{exam:alg1:running}, \ref{exam:pos:SS}  and \ref{exam:neg:SS}]\label{exam:alg2:running}
    Let $T_{31}  = \tuple{\overline{B}, E^+, E^-}$ where $\overline{B} = \{ (p \lto q, 0.3), (q \lto \Not r, 0.5) \}$, $E^+ = \{ \{ (r, 0.3) \} \}$ and $E^- = \{ \{ (q, 0.3), (r, 0.5) \}, \{ (p, 0.3), (q, 0.5) \} \}$. Then we have ${\cal A} = \{ p,q,r \}$, ${\cal Q} = \{ 0.3, 0.5 \}$ and $0.3 \leq 0.5$. 
    When inputting $T_{31}$ into algorithm {\ILPSMmin}, it can yield the minimal solution $\overline{H_{32}} = \{ (r \lto , 0.3) \}$ according to the analysis as follows.
\begin{itemize}
	\item In line (1): The value of $\mathit{Existence}(\overline{B}, E^+, E^-)$ is {\bf true}, so the algorithm continues;
	\item In line (3): According to the analysis in example~\ref{exam:neg:SS}, the value of $S^-(\{ (r, 0.3) \})$ can be determined. For example, $(r \lto , 0.5) \in blacklist$; 
	\item In line (4): Let $\overline{I} = \{ (r, 0.3) \}$ be the first positive example analyzed;
	\item In line (6): According to the analysis in example~\ref{exam:pos:SS}, $\{ \{(r \lto \Not q, 0.3)\}, \{(r \lto , 0.3)\}, \{(r \lto \Not p, \Not q, 0.3)\} \} \subseteq W_1$; 
	\item In line (8): $\{ \{(r \lto \Not q, 0.3)\}, \{(r \lto , 0.3)\}, \{(r \lto \Not p, \Not q, 0.3)\}\} \subseteq seeds$;
	\item In line (9): Assume $\{(r \lto \Not q, 0.3)\}$ enters the loop first, i.e., $X = \{(r \lto \Not q, 0.3)\}$;
	\item In line (10): $E^- \cap \mathit{PSM}(\overline{B} \sqcup X) = \{ \{ (p, 0.3), (q, 0.5) \} \} \neq \emptyset$, i.e., the condition in line (10) does not satisfy but the condition in line (13) satisfies, so lines (11-12) will be skipped and the algorithm jumps to lines (14-20);
	\item In line (14): Since $\{ (q, 0.3), (r, 0.3) \} \parallel E^+$, $\overline{E}$ is assigned the value $\{ \{ (p, 0.3), (q, 0.5) \} \}$;
	\item In line (15): Now only $\{ (p, 0.3), (q, 0.5) \} \in \overline{E}$, so $\overline{I}$ can only take $\{ (p, 0.3), (q, 0.5) \}$. Since $\{ (r \lto , 0.5), (p \lto q, 0.5), (r \lto p, 0.3) \} \subseteq S^-(\overline{I})$ and $(r \lto , 0.5) \in blacklist$, it follows that $\{ (p \lto q, 0.5), (r \lto p, 0.3) \} \subseteq whitelist$;
	\item In line (16): Now $j$ can only take the value 1, so this loop executes only once;
	\item In line (17): $\{ \{(p \lto q, 0.5)\}, \{(r \lto p, 0.3)\} \} \subseteq patches$;
	\item In line (18): Assume $\{(p \lto q, 0.5)\}$ enters the loop first, i.e., let $\overline{P} = \{(p \lto q, 0.5)\}$;
	\item In line (19): Two conditions are both fulfilled;
	\item In line (20): $\overline{H}$ is updated from initial $\emptyset$ to $\{ (r \lto \Not q, 0.3), (p \lto q, 0.5) \}$, while $norm$ is updated from initial value to $2$;
	\item In lines (18-20): Continue looping through the other poss-NLPs in $patches$, but the condition $\vert X \sqcup \overline{P} - \overline{B} \vert < norm$ remains false regardless of how $X$ is assigned afterwards. Therefore, execution jumps back to line (9) since the update operation in line (20) will not execute again until the loop ends;
	\item In line (9): Assume $\{(r \lto , 0.3)\}$ consequently enters the loop, i.e., let $X = \{(r \lto , 0.3)\}$;
	\item In line (10): Jump to line (11) since $E^- \cap \mathit{PSM}(\overline{B} \sqcup X) = \emptyset$;
	\item In line (11): Jump to line (12) since $\vert X - \overline{B} \vert = 1 < 2$;
	\item In line (12): $\overline{H}$ is updated to $\{(r \lto , 0.3)\}$, and $norm$ is updated to $1$, then jump back to line (9) to continue the subsequent loop;
	\item In line (9): Jump to line (21) since the subsequent loop can no longer find a solution whose total number is less than $1$;    
	\item In line (21): Return the minimal solution $\overline{H} = \{(r \lto , 0.3)\}$;
\end{itemize}
It is evident that $\emptyset$ is not a solution. 
A minimal solution yielded by algorithm {\ILPSMmin} contains only one rule. 
In contrast, the solution yielded by algorithm {\ILPSM} contains two rules.


To see a more general case, let us revisit Example~\ref{exam:PAS:solutions}. 
When $T_1$ is fed to algorithm {\ILPSMmin}, it can return a minimal solution. 
\[
\overline{H} = \{ (medA \lto vomiting, \Not medB, 1) \} .
\]
This hypothesis means that the specialist absolutely suggests prescribing medicine A if the patient vomits and has not taken medicine B. 
In contrast, algorithm {\ILPSM} must yield a solution with 10 rules. 
When $T_1$ is inputted into algorithm {\ILPSM}, $\mathit{PPE}(E^+)$ in line (4) is constructed as the following poss-NLP
\[
    \left \{ 
        \begin{array}{c}
            (pregnancy \lto \Not medB, 1), (vomiting \lto \Not medB, 1), \\  
            (medA \lto \Not medB, 1), (\mathit{relief} \lto \Not medB, 0.7), \\ 
            (malnutrition \lto \Not medB, 0.7), (pregnancy \lto \Not medA, 1), \\ 
            (vomiting \lto \Not medA, 1), (medB \lto \Not medA, 1), \\ 
            (\mathit{relief} \lto \Not medA, 0.6), (malnutrition \lto \Not medA, 0.1)
        \end{array}	
    \right\}.
\]
After assigning $\mathit{PPE}(E^+)$ to $\overline{H}$ in this step, $\overline{H}$ remains unchanged as $\overline{E} = \emptyset$ and then $\overline{H} \sqcup \mathit{PNE}(\overline{E},E^+) - \overline{B} = \overline{H}$.
\end{example}

As demonstrated in the example, the algorithm {\ILPSMmin} has a certain degree of randomness such as the variable order of values in line (9). But this does not hinder the algorithm from achieving the desired goal, i.e., the returned poss-NLP is guaranteed to be a minimal solution of the induction task.

\begin{proposition}[Correctness of algorithm ILPSMmin]\label{prop:alg2:correct}
    For an induction task $T = \tuple{\overline{B}, E^+, E^-}$, algorithm {\ILPSMmin} returns {\bf fail} when $\mathit{ILP_{LPoSM}}(T) = \emptyset$. Otherwise, algorithm {\ILPSMmin} returns a minimal solution of $T$.
\end{proposition}

In this section, we have introduced two algorithms {\ILPSM} for computing an induction solution and {\ILPSMmin} for computing a minimal induction solution. 
\hbl{
The time complexity of algorithms {\ILPSM} (resp. {\ILPSMmin}) is bounded by $O(n2^n)$ (resp. \ysm{$O(n^52^{n^22^n+3n})$}{$O(n^{n+2}2^{n^2(2^n+1)+n})$}).
}


\section{Variants of Possibilistic Induction Tasks}\label{sec:variants}

In this section we study three variants of the induction task for possibilistic logic programs under stable models. 
In the first subsection, we consider two special cases of induction tasks.
The first special case is when $\overline{\mathbb{U}} = E^+\cup E^-$, that is, each atom is either a positive sample or a negative example. We provide a characterisation of induction solutions through poss-stable models.
The second special case is when input programs are only ordinary NLPs for an induction task for poss-NLPs. In this case, the definition of induction tasks for poss-NLPs collapses to the induction for ordinary NLPs.
In the second subsection, we consider the generalisation of induction tasks for poss-NLPs by allowing partial interpretations. Interestingly, we show that such a generalised induction task can be equivalently reduced to an induction task as defined in Definition~\ref{def:ILT}.


\subsection{Two special cases of induction tasks for poss-NLPs}\label{subsec:cases}

In the definition of induction tasks (Definition~\ref{def:ILT}), we observe that the union $E^+ \cup E^-$ of positive and negative examples may not be the whole set $\overline{\mathbb{U}}$ of atoms. However, if $\overline{\mathbb{U}} = E^+ \cup E^-$, we can show that the solutions of such induction tasks have a simple characterisation.  


Under the condition $\overline{\mathbb{U}} = E^+ \cup E^-$, we need only to specify the set $E^+$ of positive examples as $E^- = \overline{\mathbb{U}} - E^+$. Thus, we have the following definition.

\begin{definition}[Induction task from complete poss-stable models]\label{def:ILT:comp:NLP}
An {\em induction task from complete poss-stable models} is a tuple $T = \tuple{ \overline{B}, E^+}$ where poss-NLP $\overline{B}$ is the background knowledge, $E^+$ is a set of possibilistic interpretations called the positive examples. A hypothesis $\overline{H}$ belongs to the set of induction solutions of $T$, written as $\overline{H} \in \mathit{ILP_{LCPoSM}}(T)$, if $E^+ = \mathit{PSM}(\overline{B} \sqcup \overline{H})$.
\end{definition}

Compared to the induction task in Definition~\ref{def:ILT}, the set $E^+$ of poss-stable models in Definition~\ref{def:ILT:comp:NLP} is complete, since there is no other poss-stable model outside $E^+$. 
By Theorem~\ref{thm:taskNLP:exist}, it follows the below Proposition~\ref{prop:comp:taskNLP:exist}, which is a necessary and sufficient condition for the existence of a solution for an induction task from complete poss-stable models.

\begin{proposition}[Existence of a poss-NLP solution for task in Definition~\ref{def:ILT:comp:NLP}]	\label{prop:comp:taskNLP:exist}  
For an induction task $T = \tuple{\overline{B}, E^+}$ from Definition~\ref{def:ILT:comp:NLP}, $\mathit{ILP_{LCPoSM}}(T) \neq \emptyset$ if and only if 
\begin{enumerate}[label=(C\arabic*)]
    \item $E^+$ is incomparable, and
    \item $E^+$ is coherent with $\overline{B}$, and 
    \item (\romannumeral 1)$\mathit{lfp}(T_{B^{\cal A}}) \neq {\cal A}$, or (\romannumeral 2)$\vert {\cal Q} \vert = 1$ and $E^+ = \{ \overline{\mathbb{A}} \}$, or (\romannumeral 3)$\overline{\mathbb{A}} \cap E^+ \neq \emptyset$.
\end{enumerate}
\end{proposition}

Let us look at the following example.

\begin{example}\label{exam:comp:NLP:exist}
    Consider Example~\ref{exam:comp:NLP:exist} again.
    Let $T_{41}  = \tuple{\overline{B}, E^+}$ be an induction task from complete poss-stable models where $\overline{B} = \{ (r \lto , 0.3) \}$ and $E^+ = \{ \{ (p, 0.5), (r, 0.5) \}, \{ (q, 0.3), (r, 0.8) \} \}$. Then $E^+$ is incomparable and coherent with $\overline{B}$. Besides, $\mathit{lfp}(T_{B^{\cal A}}) \neq {\cal A}$ holds at least. Therefore, $T_{41}$ must have a solution.     
    It is easy to verify that $\{ (p \lto \Not q, 0.5), (r \lto \Not q, 0.5), (q \lto \Not p, 0.3), (r \lto \Not p, 0.8) \} \in \mathit{ILP_{LCPoSM}}(T_{41})$.
\end{example}


An ordinary NLP can be seen as a poss-NLP in which all rules have the same weight. 
An induction task for poss-NLPs degenerate to induction for ordinary NLPs 
when the induction goal is to induce a NLP $H$ such that $E^+\subseteq \mathit{SM}(B\cup H)$ and $E^-\cap \mathit{SM}(B\cup H)=\emptyset$ in which the background $B$ is a NLP, and the examples in $E^+$ and $E^-$ are interpretations.
For ease of discussion, we call this kind of induction task learning NLP from stable models as LSM induction task, and use $\mathit{ILP_{LSM}}(T)$ to denote the solution of the LSM induction task $T$.
Thus, a LSM induction task can also be solved by our two induction algorithms for poss-NLPs.




\begin{proposition}[Existence of a NLP solution for a LSM induction task]\label{prop:taskNLP:exist}  
For an induction task $T = \tuple{B, E^+, E^-}$, $\mathit{ILP_{LSM}}(T) \neq \emptyset$ if and only if 
\begin{enumerate}[label=(c\arabic*)]
    \item $E^+$ is $\subseteq$-incomparable, and
    \item $e \models B$ for each $e \in E^+$, and 
    \item ${\cal A} \notin E^-$ or ${\cal A} \neq \mathit{lfp}(T_{B^{\cal A}})$, and 
    \item $E^+ \cap E^- = \emptyset$.
\end{enumerate}
\end{proposition}

\begin{example}\label{exam:taskNLP:B:neg:exist}
    Let $T_{42}  = \tuple{B_{med}, E^+ , \emptyset}$ where $E^+ = \{ A_1, A_2 \}$, and 
    $B_{med}$ is the NLP from Example~\ref{exam:PAS}. 
    It is easy to check that $\{ medA \lto vomiting, \Not medB. \} \in \mathit{ILP_{LSM}}(T_{42})$.     
    
    In contrast, let $T_{42}  = \tuple{B, \emptyset, E^-}$ where $E^- = \{ \{ p, q \}, \{ p \} \}$ and $B = \{ p \lto . , q \lto p. \}$. 
    Then we have ${\cal A} = \{ p, q \}$ so that ${\cal A} \in E^-$ and ${\cal A} = \mathit{lfp}(T_{B^{\cal A}})$. As a result, $\mathit{ILP_{LSM}}(T_{42}) = \emptyset$. 
\end{example}

    

An induction task $T = \tuple{B, E^+, E^-}$ here can be solved by a revised version of algorithm {\ILPSMmin} since it can be regarded as a subtask of the induction task as Definition~\ref{def:ILT}. 
On the other side, the induction task $T = \tuple{B, E^+, E^-}$ can also be solved by algorithm {\ILASP} ~\citep{ILASP} since a stable model can be represented as a partial interpretation.
In other words, an induction task $T = \tuple{B, E^+, E^-}$ here is both a subtask of Definition~\ref{def:ILT} and a subtask of the induction task solved by {\ILASP} ~\citep{ILASP}. 
\hbl{
It is interesting to see which approach can solve this subtask faster, although {\ILASP} induces  non-ground answer set programs from partial interpretations while {\ILPSMmin} induces ground poss-NLP from stable models.
}
So we discuss the implementation of our algorithm and show the experimental comparison between these two approaches in Section~\ref{sec:implementation:experiment}. 
Additionally, a brief theoretical comparison can be found in Section~\ref{sec:related:work} discussing the related work.

\subsection{Induction from partial stable models}

We first extend the definition of inductive learning for poss-NLPs by allowing partial interpretation. 
Recall that a partial interpretation is a pair $\tuple{E^{inc}, E^{exc}}$ of two sets $E^{inc}$ and $E^{exc}$ of atoms where $E^{inc} \cap E^{exc} = \emptyset$. Intuitively, each atom in $E^{inc}$ is true, each atom in $E^{exc}$ is false, and each element in ${\cal A} - E^{inc} - E^{exc}$ is unknown. 
If $E^{inc} \cup E^{exc} = {\cal A}$, the partial interpretation $\tuple{E^{inc}, E^{exc}}$ is equivalent to the interpretation $E^{inc}$.
An interpretation $I$ extends a partial interpretation $E = \tuple{E^{inc}, E^{exc}}$, written as $I \propto E$, if $(E^{inc} \subseteq I) \land (E^{exc} \cap I = \emptyset)$.



The generalised notion of induction tasks for poss-NLPs is formally defined as follows.

\begin{definition}[Induction task from partial stable models]\label{def:ILT:NLP:partial}
An {\em induction task from partial stable models} is a tuple $T = \tuple{B, E^+, E^-}$ where NLP $B$ is the background knowledge, two sets $E^+$ and $E^-$ of partial interpretations are respectively called the positive and negative examples. A hypothesis $H$ belongs to the set of induction solutions of $T$, written as $H \in \mathit{ILP_{LPaSM}}(T)$, if it achieves two aims
\begin{align}
    \forall e^+ \in E^+ \exists I \in \mathit{SM}(B \cup H), I \propto e^+,  \tag{{a1}} \label{eq:a1}\\
    \forall e^- \in E^- \nexists I \in \mathit{SM}(B \cup H), I \propto e^-. \tag{{a2}} \label{eq:a2}
\end{align}
\end{definition}

We borrow an Example~\ref{exam:partial:taskNLP} in~\citep{ILASP} to illustrate the above definition.

\begin{example}\label{exam:partial:taskNLP}
    Let $T_{43}  = \tuple{B, E^+, E^-}$ where $E^+ = \{ \tuple{\{ p \}, \emptyset}, \tuple{\{ q \}, \{ p \}} \}$, $E^- = \{ \tuple{\{ p,q \}, \emptyset} \}$ and $B = \{ q \lto r. \}$. Then we have ${\cal A} = \{ p, q, r \}$, $\mathit{SM}(B \cup H_{1}) = \{ \{ p \}, \{ q, r \} \}$ and $\mathit{SM}(B \cup H_{2}) = \{ p, q, r \}$ where $H_{1} = \{ p \lto \Not r. , r \lto \Not p. \}$ and $H_{2} = \{ p \lto r. , r \lto . \}$. It is easy to verify that $\{ p \} \propto \tuple{\{ p \}, \emptyset}$, $\{ q, r \} \propto \tuple{\{ q \}, \{ p \}}$ and $\{ p, q, r \} \propto \tuple{\{ p,q \}, \emptyset}$. As a result, $H_{1} \in \mathit{ILP_{LPaSM}}(T_{43})$ and $H_{2} \notin \mathit{ILP_{LPaSM}}(T_{43})$.
\end{example}

As we can see, the induction task in Definition~\ref{def:ILT:NLP:partial} and a LSM induction task reflect the same requirements. The denotation of a partial interpretation $o$ is essentially the set 
\[
    \mathit{de}(o) = \{ I \in 2^{\cal A} \mid I \propto o \} 
\]
of interpretations. The two induction tasks can be transformed as Proposition~\ref{prop:tasks:map}. See Example~\ref{exam:partial:taskNLP:trans} for more details.

\begin{proposition}[Transformations between induction tasks]\label{prop:tasks:map}  
Let $T = \tuple{B, O^+, O^-}$ be an induction task of NLP from partial stable models. A NLP $H$ such that $H \in \mathit{ILP_{LPaSM}}(T)$ if and only if $H \in \mathit{ILP_{LSM}}(T_1)$ for some $T_1 = \tuple{B, E^+, E^-}$ such that 
\begin{enumerate}[label=(c\arabic*)]
    \item $E^+$ is a minimal hitting set of $\{\mathit{de}(o) \mid o \in O^+ \}$, written as $E^+ \in \mathit{SMHS}(\{\mathit{de}(o) \mid o \in O^+ \})$, and 
    \item $E^- = \{ J \in \mathit{de}(o) \mid o \in O^- \}$.
\end{enumerate}
\end{proposition}

\begin{example}[Continued from Example~\ref{exam:partial:taskNLP}]\label{exam:partial:taskNLP:trans}
     We have $\{\mathit{de}(o) \mid o \in E^+ \} = \{ \{ \{ p \}, \{ p,q \},  \{ p,r \}, \{ p,q,r \} \}, \{ \{ q \}, \{ q,r \} \} \}$, $\{ J \in \mathit{de}(o) \mid o \in E^- \} = \{ \{ p,q \}, \{ p,q,r \} \}$ and $\{ \{ p \}, \{ q, r \} \} \in \mathit{SMHS}(\{\mathit{de}(o) \mid o \in O^+ \})$. As a result, induction task $T_{43}$ can generate a LSM task $T_1 = \tuple{B, \{ \{ p \}, \{ q, r \} \}, \{ \{ p,q \}, \{ p,q,r \} \}}$. It is evident that $H_{1} \in \mathit{ILP_{LSM}}(T_1)$.
\end{example}

According to Proposition~\ref{prop:tasks:map}, an induction task $T = \tuple{B, O^+, O^-}$ from partial stable models can be transformed into $\vert \mathit{SMHS}(\{\mathit{de}(o) \mid o \in O^+ \}) \vert$ induction tasks $T_i = \tuple{B, E^+_i, E^-}$ where 
\begin{itemize}
    \item $E^+_i$ is the $i$-th element in $\mathit{SMHS}(\{\mathit{de}(o) \mid o \in O^+ \})$, and 
    \item $E^- = \{ J \in \mathit{de}(o) \mid o \in O^- \}$.
\end{itemize}
For convenience, we use $\mathit{trans}(T)$ to denote the set of all corresponding induction tasks $T_i$ deriving from induction task $T$. 
A NLP $H$ is a solution of $T$ from partial stable models if and only if there exists $T_i \in \mathit{trans}(T)$ such that $H \in \mathit{ILP_{LSM}}(T_i)$. Formally, we have the necessary and sufficient condition as Corollary~\ref{cor:LPaSM:existence}. Example~\ref{exam:LPaSM:existence} helps to illustrate this conclusion.

\begin{corollary}[Existence of a NLP solution]\label{cor:LPaSM:existence}	
    Let $T = \tuple{B, O^+, O^-}$ be an induction task of NLP from partial stable models. $\mathit{ILP_{LPaSM}}(T) \neq \emptyset$ if and only if there exists $E^+ \in \mathit{SMHS}(\{\mathit{de}(o) \mid o \in O^+ \})$ and $E^- = \{ J \in \mathit{de}(o) \mid o \in O^- \}$ such that 
    \begin{enumerate}[label=(c\arabic*)]
        \item $E^+$ is $\subseteq$-incomparable, and
        \item $e \models B$ for each $e \in E^+$, and 
        \item ${\cal A} \notin E^-$ or ${\cal A} \neq \mathit{lfp}(T_{B^{\cal A}})$, and 
        \item $E^+ \cap E^- = \emptyset$.
    \end{enumerate}
\end{corollary}

\begin{example}\label{exam:LPaSM:existence}
    Let $T  = \tuple{B, O^+, O^-}$ where $O^+ = \{ \tuple{\{ p \}, \emptyset} \}$, $O^- = \{ \tuple{\{ p,q \}, \emptyset} \}$ and $B = \{ q \lto p. \}$. Then we have ${\cal A} = \{ p, q \}$, $\{\mathit{de}(o) \mid o \in O^+ \} = \{ \{ \{ p \}, \{ p,q \} \} \}$, $E^- = \{ J \in \mathit{de}(o) \mid o \in O^- \} = \{ \{ p,q \} \}$ and $\mathit{SMHS}(\{\mathit{de}(o) \mid o \in O^+ \}) = \{ \{ \{ p \} \},  \{ \{ p,q \} \} \}$. As a consequence, $\mathit{trans}(T) = \{ T_1, T_2 \}$ where $T_1 = \tuple{B,  E^+_1, E^-}$, $T_2 = \tuple{B, E^+_2, E^-} \}$, $E^+_1 = \{ \{ p \} \}$ and $E^+_2 = \{ \{ p,q \} \}$. Finally, $\mathit{ILP_{LPaSM}}(T) \neq \emptyset$ since $\{ p \} \in E^+_1, \{ p \} \not\models B$ and $E^+_2 \cap E^- \neq \emptyset$.
\end{example}


\section{Implementation and Experiments}\label{sec:implementation:experiment}
We have implemented a prototype for computing minimal solutions for induction tasks based on Algorithm~\ref{alg:ILPSM2}.
In this section, we first describe technical details of our implementation and report some preliminary experimental results, which show that our prototype significantly outperforms the baseline {\ILASP4}~\citep{DBLP:journals/tplp/Law23} when inducing ordinary NLPs. 

As we have seen in the algorithm, a solver for computing poss-stable models will be called. 
While several algorithms for computing poss-stable models have been proposed in the literature, we have not seen any efficient implementation for computing poss-stable models. 
For this reason, our implementation runs only for induction tasks in ordinary NLPs. But it can be easily adapted to the case of full poss-NLPs when an efficient solver for computing poss-stable models is available.

The source code of implementation, testing datasets and experimental results are all available on Github\footnote{https://github.com/gzu-ai/ILP-ILSM}.

\subsection{Implementation}
Our system is an implementation of the new algorithm {\ILPSMmin} using Python~3.11.5 when all inputs are ordinary NLPs. 
We call this implementation as {\ILSMmin} which is designed to solve a LSM induction task. 

To describe the implementation details, let us first recall the solution space of an induction task as defined in Definition~\ref{def:pos:SS} and Definition~~\ref{def:neg:SS}.
The basic idea of our algorithm is to explore various combinations of rules and find a minimal one that meets the coverage requirements for an induction task.
Specifically, the algorithm is to reduce an induction task into a combinatorial optimization problem (COP for short) below.
\begin{itemize}
    \item The discrete \emph{decision space} $\Omega$ of an induction task is determined 
    by the positive solution space and negative solution space by Theorem~\ref{thm:SS:iff}. 
    \item The coverage requirements $f(h, B, E^-) = 0$ and $f(h, B, E^+) = \vert E^+ \vert$ act as the \emph{constraints} where $f(h, B, E) = \vert \mathit{SM}(h \cup B) \cap E \vert$.
    \item The \emph{optimization objective} is to minimize $\vert h \vert$.
\end{itemize}
By invoking an efficient solver for computing stable models, the search process is able to efficiently discard those candidate solutions that  violate some constraints or does not meet the optimization objective. 
The ASP solver Clingo\footnote{https://github.com/potassco/clingo/releases} was used in our implementation.

\begin{figure}[h]
	{
	\includegraphics[width=9cm]{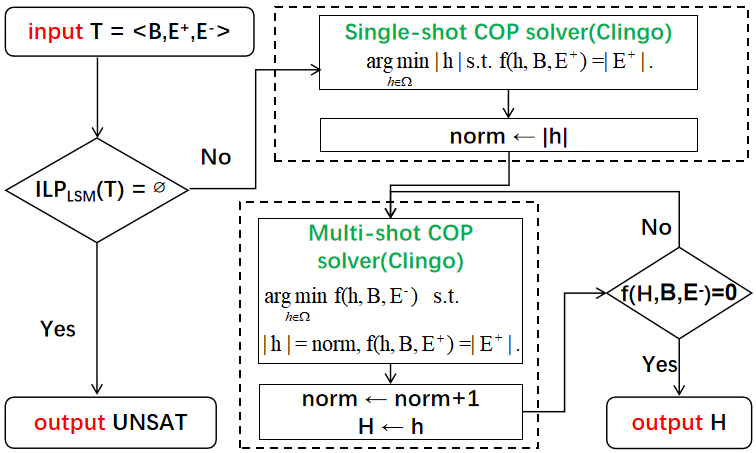}\\
    }
    \caption{The Architecture of {\ILSMmin}}
    \label{Fig:implementation:route}
\end{figure}

The architecture of our implementation is 
depicted in  Figure~\ref{Fig:implementation:route}. 
For an input induction task $T = \tuple{B, E^+, E^-}$, our algorithm 
consists of three main stages as follows.
\begin{enumerate}
    \item 
    Check the existence of a solution by Proposition~\ref{prop:taskNLP:exist}.
    \item 
    Compute the minimum number $norm$ of rules that cover all positive examples by invoking Clingo only once (one shot calling).
    \item In order to satisfy all negative examples, iteratively increase the total number $norm$ of rules by one each time. In this stage, Clingo is called for multiple times (multi-shot calling).    
    In this way, our algorithm iteratively minimize the number of negative examples covered after explored a set of NLPs covering all positive examples. 
    The NLP $H$ generated in the above process is a minimal solution of $T$ when $f(H, B, E^-) = 0$.
\end{enumerate}


\subsection{Experiment}

All experiments in this section were executed on an Intel(R) CPU @ 3.60GHz and Clingo version~5.2.2. 
The memory limit for each induction task is set to 5G of RAM. 
\hbl{
Similar to our algorithm {\ILSMmin}, the contrast algorithm {\ILASP4}~\citep{DBLP:journals/tplp/Law23}, as briefly discussed in Subsection~\ref{subsec:cases}, can also induce ordinary NLPs from stable models, although {\ILASP4} is originally designed to learn non-ground programs from partial interpretations.
}

We randomly generated three datasets of induction tasks with a total number of $440$ induction tasks. 
One dataset is 
from the medical domain (Example~\ref{exam:PAS:solutions}), where the NLP $P_{med}$ represents a part of some physician's epistemic belief. 
The other two datasets are 
from two GRNs (Arabidopsis thaliana and T-cell receptor), where the corresponding NLPs represent the dynamic systems~\citep{inoue2014LF1T,LFDT,ribeiro2022learning,hu2025learning}. 


The first dataset $Med$ contains $100$ randomly generated induction tasks, in which the total number $\vert {\cal A} \vert$ of atoms ranges from $1$ to $6$. 
A random induction task $T = \tuple{B, E^+, E^-}$ was synthesized according to the following requirements, and then we iterated this synthesis $100$ times. 
\begin{itemize}
    \item $B \subseteq P_{med}$;
    \item $E^+ \subseteq \mathit{SM}(P_{med})$;
    \item $E^- \subseteq 2^{hb(P_{med})}$ and $0 \leq \vert E^- \vert \leq 5$ where $hb(P)$ denotes the Herbrand base of NLP $P$. 
\end{itemize}
The second dataset $Ara$ contains $5 \times 4 \times 5 = 100$ randomly generated induction tasks, in which $\vert {\cal A} \vert$ ranges from $0$ to $15$.  
These induction tasks derived from the NLP $P_{Ara}$ representing a GRN of Arabidopsis thaliana, where $\vert P_{Ara} \vert = 28$, $\vert hb(P_{Ara}) \vert = 15$ and $\vert \mathit{SM}(P_{Ara}) \vert = 2$. 
Each $T = \tuple{B, E^+, E^-}$ of these $100$ induction tasks was synthesized according to the following requirements. 
\begin{itemize}
    \item $B \subseteq P_{Ara}$ where $\vert B \vert \in \{ 0, 6, 12, 18, 24 \}$;
    \item $E^+ \subseteq {\mathit{SM}(P_{Ara})}$; 
    \item $E^- \subseteq 2^{hb(P_{Ara})}$ where $\vert E^- \vert \in \{ 0, 1, 2, 3, 4 \}$. 
\end{itemize}
The third dataset $Tce$ contains $10 \times 3 \times 2 \times 4 = 240$ randomly generated induction tasks, in which $\vert {\cal A} \vert$ ranges from $22$ to $40$.  
These induction tasks derived from the NLP $P_{Tce}$ representing a GRN of T-cell receptor, where $\vert P_{Tce} \vert = 45$, $\vert hb(P_{Tce}) \vert = 40$ and $\vert \mathit{SM}(P_{Tce}) \vert = 1$. 
$10$ groups of induction tasks were iteratively generated where each group comprise $24$ induction tasks with different scale. 
Each $T = \tuple{B, E^+, E^-}$ of these $24$ induction tasks within one iteration was synthesized according to the following requirements. 
\begin{itemize}
    \item $B \subseteq P_{Tce}$ where $\vert B \vert \in \{ 15, 30, 45 \}$;
    \item $E^+ \in 2^{\mathit{SM}(P_{Tce})}$;
    \item $E^- \subseteq 2^{hb(P_{Tce})}$ where $\vert E^- \vert \in \{ 0, 5, 10, 15 \}$. 
\end{itemize} 

The experimental 
results 
are summarized in Table~\ref{tab:exp:comp}. 
If the program is aborted due to TO (time-out)
or OOM (Out Of Memory), the test 
is labeled with ``Fail". 
If the program returns an answer ``no solution'' (resp., returns a solution), the test is labeled with ``UNSAT" (resp., ``Success"). 

\begin{table}[h!]
 \centering
 \caption{
 A comparison of {\ILSMmin} and {\ILASP}4 against three benchmark datasets.
 The id of each induction task set is in the form $D\_M\_L\_U$ where $D$ is the name of the dataset, $M$ is the number of induction tasks with $L \leq \vert {\cal A} \vert \leq U$. 
Cnt(TO) (resp., Cnt(OOM)) is the number of induction tasks that the program runs out of CPU time (resp., runs out of memory).
 } 
 \label{tab:exp:comp}
 {\tablefont
    \begin{tabular}{@{\extracolsep{\fill}}lllrrrrrr}  
       \topline
        Tasks & CPU-limit & Program & \multicolumn{2}{c}{Fail} & \multicolumn{2}{c}{UNSAT} & \multicolumn{2}{c}{Success} \\
         & & & Cnt(TO) & Cnt(OOM) & Cnt & Time(s) & Cnt & time(s)
        \midline 
        \multirow{2}{*}{Med\_100\_1\_6} & \multirow{2}{*}{600s} & ILASP4 & 0 & 0 & 3 & 23.18485 & 97 & 3.37363 \\
         & & ILSMmin & 0 & 0 & 3 & 0.00005 & 97 & 0.00728 \\
        \multirow{2}{*}{Ara\_100\_0\_15} & \multirow{2}{*}{600s} & ILASP4 & 62 & 33 & 1 & 0.68667 & 4 & 3.62628 \\
         & & ILSMmin & 0 & 0 & 3 & 0.00004 & 97 & 0.06177 \\
        \multirow{2}{*}{Tce\_240\_22\_40} & \multirow{2}{*}{180s} & ILASP4 & 240 & 0 & 0 & - & 0 & - \\
         & & ILSMmin & 0 & 0 & 0 & - & 240 & 0.01538 
       \botline
    \end{tabular}
 }
\end{table}

For the first induction task set $Med\_100\_1\_6$, both {\ILSMmin} and {\ILASP}4 solved all induction tasks, but {\ILSMmin} is much faster than {\ILASP}4 for both ``Success'' and ``UNSAT''.
For the second dataset $Ara\_100\_0\_15$, {\ILSMmin} solved all $100$ induction tasks, but {\ILASP}4 solved only $5$ induction tasks.
{\ILSMmin} is also much faster than {\ILASP}4 for both ``Success'' and ``UNSAT''.
Additinally, {\ILSMmin} cost less memory than {\ILASP}4, as {\ILASP}4 is aborted due to OOM when solving the $33$ induction tasks which are successfully solved by {\ILSMmin}.
For the third dataset $Tce\_240\_22\_40$, {\ILSMmin} quickly solved all $240$ induction tasks, but {\ILASP}4 solved none of them. 

As we know, the major cost of a ILP solver takes is from the exclusion of those logic programs that cannot be used as final solutions.
Our optimisation strategies helped in discarding unsolvable cases in an early stage.
As a result, the solution space is significantly reduced and our system has fewer memory issues than {\ILASP}4 .

The following example 
illustrates some subtle difference between the two solvers {\ILSMmin} and {\ILASP}4.

\begin{example}\label{exam:comp:minimal:short}
    Induction task 11 in the set $Med\_100\_1\_6$ is $\tuple{B, E^+, E^-}$ where $B = \{ a \lto . , d \lto b,\Not c. , f \lto d,a. \}$, $E^+ = \{ \{f, b, a, c, e\} \}$ and $E^- = \{ \{f, d, e\}, \{f, b, d, a, c, e\}, \emptyset \}$. 
    {\ILSMmin} took $0.00491$ seconds to induce the minimal solution $H = \{ f \lto . , e \lto . , c \lto . , b \lto . \}$, while 
    {\ILASP}4 took $2.26562$ seconds for the same induction task and got the same solution. 

    Induction task 13 in the set $Med\_100\_1\_6$ is $\tuple{B, E^+, E^-}$ where $B = \{ f \lto d,a. , c \lto b,\Not d. , e \lto b,d. \}$, $E^+ = \{ \{f, b, a, d, e\} \}$ and $E^- = \{ \{f, c, d\}, \{a\}, \emptyset, \{f\} \}$. 
    {\ILSMmin} took $0.00521$ seconds to induce the minimal solution $H_1 = \{ a \lto . , d \lto . , b \lto a. \}$. 
    {\ILASP}4 took $1.56451$ seconds to induce the shortest solution $H_2 = \{ d \lto . , b \lto . , a \lto . \}$. It can be checked that $H_2$ is also in the $\Omega$. 
\end{example}

To obtain a more reliable trend in the data, we conducted multiple runs to count 
the average time after removing two extreme data points. 
For any induction task $T = \tuple{B, E^+, E^-}$ in induction task set $Tce\_240\_22\_40$, we can observe from the synthetic requirements that 
\begin{itemize}
    \item $E^+ = \emptyset$ or $E^+ = \mathit{SM}(P_{Tce})$, and 
    \item $\vert \{ \tuple{B_1, E^+, E_1^-} \in Tce\_240\_22\_40 \mid \vert B_1 \vert = \vert B \vert, \vert E_1^- \vert = \vert E^- \vert \} \vert = 10$.
\end{itemize}
The experimental results are divided into two groups according to $E^+$. 
For a set of execution times corresponding to $10$ induction tasks of the same scale in each group, the execution time is averaged over the remaining $8$ values after removing the highest and lowest values. 
Based on these $24$ average times, the three-dimensional subgraphs are plotted in Figure~\ref{fig:ILSMmin}. 
As can be seen from the figure, the CPU 
time tends to increase with the increase of $\vert B \vert$  and $\vert E^- \vert$ when $E^+$ is fixed. 

\begin{figure}[htbp]
    \centering
    \begin{subfigure}[b]{0.40\textwidth}
        \includegraphics[width=\linewidth]{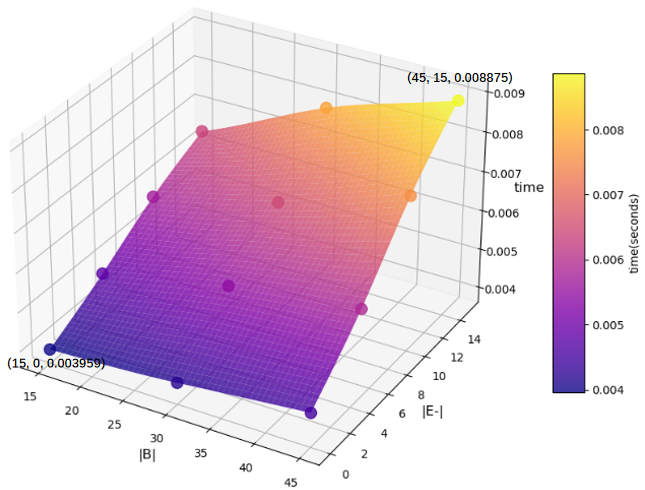}
        \caption{$E^+ = \emptyset$}
    \end{subfigure}
    \hspace{10pt}
    \begin{subfigure}[b]{0.40\textwidth}
        \includegraphics[width=\linewidth]{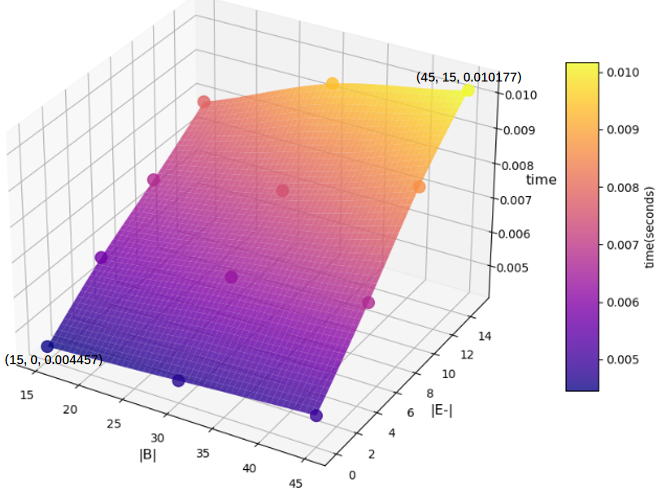}
        \caption{$E^+ = \mathit{SM}(P_{Tce})$}
    \end{subfigure}
    \caption{Runtime of algorithm ILSMmin solving induction tasks with different scales. 
    The coordinates of the highest and lowest points are labeled in both figures. 
    The $\vert E$-$\vert$-axis indicates the scale of the negative examples.}
    \label{fig:ILSMmin}
\end{figure}


\section{Related Work}\label{sec:related:work}
In this section we discuss some related work on induction under stable models for ordinary logic programs (Subsection~\ref{subsec:related1}) and induction for possibilistic logic programs (Subsection~\ref{subsec:related2}). 
\subsection{Induction under stable model semantics} \label{subsec:related1}
An approach to induction in answer set programming is introduced by~\citet{sakama2005induction}.
Their approach aims to find a rule such that an example can flip its entailment under skeptical reasoning. Specifically, given an extended logic program $B$ as background and a ground literal $L$ as an example, algorithm $IAS^{pos}$ (resp. $IAS^{neg}$) construct a rule $R$ such that $P \cup \{ R \} \models L$ (resp. $P \cup \{ R \} \not\models L$) when $P \not\models L$ (resp. $P \models L$). Here, $P \models L$ means $L \in A$ holds for each $A \in \mathit{SM}(P)$. This cautious induction regards the observation $L$ as a consequence rather than an interpretation under stable model semantics. Algorithms $IAS^{pos}$ and $IAS^{neg}$ are bottom-up as the construction starts with a specific hypothesis and then gradually generalises it. 
Later, \citet{sakama2009brave} also proposed an approach to brave induction, in which a hypothesis $H$ covers a set $O$ of ground literals under an extended disjunctive program $B$ if there exists $A \in \mathit{SM}(B \cup H)$ such that $O \subseteq A$. The bottom-up algorithm $BRAIN^{not}$ was designed to solve such an induction task. Both cautious induction and brave induction seek a rule to meet the covering requirement for a single observation. The basic idea of these two algorithms is to construct a rule based on some necessary and sufficient conditions for induction before it is generalized. The brave induction and the cautious induction are defined in two separate frameworks.

Thus, {\ILASP} (Inductive Learning of Answer Set Programs) aims to integrate both brave induction and cautious induction in the same framework~\citep{law2020ilasp}. Their induction task $T = \tuple{B, S_M, E^+, E^-}$ comprises of two ASP programs $B$ and $S_M$ and two sets $E^+$ and $E^-$ of partial interpretations. 
The search space $S_M$ allows to flexibly declare some language biases such as mode declarations and syntactic choices (i.~e. normal rules, disjunctive rules, choice rules, hard constraints, weak constraints, and other common syntax used in answer set programs). 
After peeling away these syntactic sugar coats, the goal of their induction framework is to find an answer set program $H \subseteq S_M$ such that 
\begin{itemize}
    \item for each $e^+ \in E^+$, there exists $I \in \mathit{SM}(B \cup H)$ such that $ I \propto e^+$ and,
    \item for each $e^- \in E^-$, there exists no $I \in \mathit{SM}(B \cup H)$ such that $ I \propto e^-$. 
\end{itemize} 
A meta-level approach is used in {\ILASP}1~\citep{ILASP} by invoking Clingo to search for the shortest hypothesis and to accelerates the search by starting from a shorter path. 
In a recent work~\citep{law2020ilasp} on {\ILASP} systems, the form of the examples is incrementally extended, such as ordered examples in {\ILASP}2,  context dependent examples in {\ILASP}2i, and noisy examples in {\ILASP}3. 
On the other hand, the performance of algorithm {\ILASP}4~\citep{DBLP:journals/tplp/Law23} peaks in these subsequent {\ILASP} systems, as {\ILASP}4 significantly reduces the scale of the search space in each iteration. 


The comparison experiments described in the last section show that our algorithm performs better than {\ILASP} in terms of both time and memory \hbl{on the task of inducing ground NLPs from stable models}. The performance improvement benefits from the theoretical results investigated in this paper including the characterisations of induction solutions.

While ILP systems like {\ILASP} algorithms generally perform an exhaustive search to discover the best hypothesis, 
a heuristic-based algorithm {\XFOLD} is proposed by
\citet{DBLP:conf/ilp/ShakerinG18} to inductively learn normal logic programs under stable models. {\XFOLD} iteratively employs a refinement operator to specialize constructed rules starting from a most general rule whose body is empty until the score of these rules is acceptable. Information gain has been chosen as its heuristic score to maximize the coverage of positive examples. The greedy approach adopted in {\XFOLD} makes the algorithm more efficient and noise-resilient. 

\hbl{
%
\cite{proba/ilp/RaedtK08} proposed an induction tasks to find out a probabilistic logic program $H^*$ such that the likelihood of the conditional probability $Pr(E \mid H^*, B)$ is maximized, where $E$ is a set of examples, $B$ is a probabilistic logic program acting as the background theory. The examples may be definite clauses, Herbrand interpretations or proof-trees depending on the notion of ``cover''.
Similarly, \cite{DBLP:journals/corr/NicklesM14} proposed to learn a probabilistic answer set program based on the above maximum likelihood.
\cite{DBLP:conf/kr/LeeW18} proposed a weight learning task for a parameterized
$\textsc{lp}^{\textsc{mln}}$ program $P$ to find the maximum likelihood
estimation of the non-$\alpha$ weights of $P$, whose observations (or training data) are conjunction of literals. All these learning can be classified as parameters learning and thus different from our induction task in this paper.
}

\subsection{Induction of possibilistic theories} \label{subsec:related2}
The approaches discussed in the last subsection are proposed for ordinary logic programs.
However, there are few works on induction in possibilistic logics and possibilistic logic programs.
A framework for induction in possibilistic logic is proposed by ~\citet{serrurier2007introducing}, which is to use possibilities to handle exceptions in classification tasks.
Given a possibilistic logic theory $B$ as the background and a set $E$ of possibilistic ground facts as the examples, the induction task is to obtain a hypothesis, a possibilistic logic theory $H$, such that $B \cup H \models_{\Pi} E$, here $\models_{\Pi}$ is the entailment relation in possibilistic logic. 
A possibilistic ground fact $(C(x, y),\alpha)$ means that the object $x$ is assigned to the class $y$ with the certainty $\alpha$. For a possibilistic formula $(\phi, \alpha)$ and a possibilistic logic theory $K$, $K \models_{\Pi} (\phi, \alpha)$ if and only if there exists $K' \subseteq K_{\alpha}$ and $K' \not\models \bot$ such that $K'$ is minimal w.r.t. $\phi$ and there exists no $K''$ minimal w.r.t. $\bot$ such that $K' \subset K'' \subseteq K_{\alpha}$ where $K_{\alpha} = \{ \phi \mid (\phi,\beta) \in K, \beta \geq \alpha \}$. 
A logic theory $T$ is minimal w.r.t. a formula $\phi$ if and only $T \models \phi$ and $T - \{ \psi \} \not\models \phi$ for each $\psi \in T$. 
The corresponding induction algorithm {\PossILP} for possibilistic logic extends the simulated annealing algorithm, but it is designed only for classification. 
The induction problem addressed in {\PossILP} is different from ours. In their approach, only positive examples are provided and the condition for covering (positive) examples is also different. 

A transformation between possibilistic logic theories and Markov logic networks (MLNs) is proposed in ~\citep{DBLP:conf/uai/KuzelkaDS15}. This implies that the problem of learning a possibilistic logic theory is reduced to that of learning a Markov logic network. 
Based on this idea, a statistical relational induction approach~\citep{DBLP:conf/ijcai/KuzelkaDS17} is introduced to induce possibilistic logic theories from relational datasets like the UWCSE dataset, which outlines relationships within the CS department at the University of Washington among students, professors, papers, subjects, terms, and projects. This method first learns Horn rules using beam search and then a greedy search assigns necessity degrees to the rules. The induction process is based on the notion of relational marginal distribution, viewed as a possibility distribution~\citep{dubois1998possibility} since possibilistic logic can encode probability distributions as well~\citep{DBLP:conf/ecai/KuzelkaDS16}. 
These approaches adopt a quantitative perspective by employing product operations on possibilities akin to statistics. In contrast, our study takes a qualitative approach, where the set ${\cal Q}$ of necessities acts as a bounded finite linear ordinal scale. For a comprehensive discussion on the distinctions between quantitative and qualitative possibility theory, refer to~\citep{DUBOIS2023109028}.

In order to develop an explainable and efficient approach to induction for possibilistic rules, 
the RIDDLE (Rule Induction with Deep Learning) algorithm uses artificial neural networks (ANNs) to extract possibilistic rules~\citep{DBLP:conf/nldl/Persia023}, which consists of two phases (ANN training and possibilistic rule extraction). It is shown that the trained ANN and the extracted possibilistic logic program are equivalent in terms of classification problems.  

The approaches discussed in this subsection are proposed for possibilistic logic under classic logic semantics, instead of possibilistic logic programs under nonmonotonic semantics like stable models. Therefore, to our best knowledge, the approach proposed in this paper is the first attempt to establish an induction framework for possibilistic logic programs under nonmonotonic semantics. While {\ILASP} is a special case of our framework, it is unclear how their algorithms can be directly extended to the inductive learning for possibilistic programs under stable models. Moreover, the theoretical results in this paper further convince that the problem of computing induction solution for possibilistic programs under stable models is challenging and novel techniques are needed to tackle the problem.

\section{Concluding Remarks}\label{sec:conclusion}
In this paper, we have proposed a framework for inductive reasoning in possibilistic logic programs under stable model semantics, which is based on an extension of induction tasks for ordinary logic theories (and logic programs). We have investigated formal properties of the framework, including several important characterizations for solutions of induction tasks and construction methods for possibilistic logic programs that satisfy certain conditions. Based on these results, we have proposed three algorithms for the existence of induction solutions, computing solutions and minimal solutions. In our algorithms, to narrow the search space for possibilistic rules, we introduced the notions of positive solution space and negative solution space. These notions and their construction methods significantly contribute to the efficiency of our algorithms and novelty of the whole work in this paper. We have shown the correctness of these algorithms, which are mathematically involving. Based on the algorithm {\ILPSMmin}, we have also implemented a prototype system for computing minimal solutions of induction tasks for possibilistic logic programs under stable models. The experimental results show that our prototype outperforms {\ILASP} \hbl{when solving induction tasks of learning ordinary NLPs from stable models}.

There are several interesting directions for future work.
First, our approach could be extended to possibilistic logic programs that allow more general language biases, such as constraints and disjunction.
Second, our approach could be adapted to some other semantics for possibilistic logic such as~\citep{BAUTERS2014739}.
In addition, while the problem of finding solutions for a given induction task for poss-programs is hard (at least NP-complete), it is still possible to significantly lift the scalability of our system by applying latest technologies in deep learning, which is another interesting research topic in future. 
\hbl{
Besides, we only implemented one algorithm inducing NLP from table models and conducted the corresponding experiments. Other efficient implementations and practical applications of algorithms, in particular Algorithm~3, deserve our further efforts.
}

\section*{Competing Interests}

The authors declare none.

\bibliographystyle{tlplike}
\bibliography{LFI}

\newpage
\begin{appendices}

\section*{Appendix: Proofs}	
\setcounter{lemmaA}{0}
\setcounter{corollaryA}{0}
\setcounter{propositionA}{2}
\setcounter{theoremA}{1} 
\setcounter{section}{2}

%




\begin{propositionA}[Equivalent consequence]
    For a given poss-NLP $\overline{P}$ and a possibilistic interpretation $\overline{I}$, ${\cal T}_{\overline{P}^I}(\overline{I}) = {\cal T}_{\overline{P}}(\overline{I})$.  
\end{propositionA}
\begin{proof}
    $(q,\delta) \in {\cal T}_{\overline{P}}(\overline{I})$. \\
    $\LRto$ By Definition~\ref{def:Tp}, (\romannumeral 1)$q=\Head(r) \mbox{ for some } \overline{r} \in \mathit{App}(\overline{P},\overline{I},q)$, and (\romannumeral 2)$\delta=\max\{\beta\mid \overline{r'} \in \mathit{App}(\overline{P},\overline{I},q)$ is $\beta$-applicable in $\overline{I}\}$. \\
    $\LRto$ By Definition~\ref{def:applicable}, (\romannumeral 1)$q=\Head(r) \mbox{ for some } \overline{r} \in \mathit{App}(\overline{P}^I,\overline{I},q)$, and (\romannumeral 2)$\delta=\max\{\beta\mid \overline{r'} \in \mathit{App}(\overline{P}^I,\overline{I},q)$ is $\beta$-applicable in $\overline{I}$. \\
    $\LRto$ By Definition~\ref{def:Tp}, $(q,\delta) \in {\cal T}_{\overline{P}^I}(\overline{I})$.
\end{proof}

\begin{corollaryA}[Absorption for a poss-stable model]
    Given a possibilistic interpretation $\overline{I}$ and two poss-NLPs $\overline{P}$ and $\overline{B}$, $\overline{I} \in \mathit{PSM}(\overline{P} \sqcup \overline{B})$ if $\overline{I} \in \mathit{PSM}(\overline{P})$ and ${\cal T}_{\overline{B}}(\overline{I}) \sqsubseteq \overline{I}$.
\end{corollaryA}
\begin{proof}
    $\overline{I} \in \mathit{PSM}(\overline{P})$.\\
    $\Rto$ $\overline{I} = \mathit{Cn}(\overline{P}^I)$.\\
    $\Rto$ $\mathit{lfp}({\cal T}_{\overline{P}^I}) = \overline{I}$.\\
    $\Rto$ ${\cal T}_{\overline{P}^I}(\overline{I}) = \overline{I}$ and $\forall \overline{J} \sqsubset \overline{I}, {\cal T}_{\overline{P}^I}(\overline{J}) \neq \overline{J}$. \\
    $\Rto$ ${\cal T}_{\overline{B}^I \sqcup \overline{P}^I}(\overline{I}) = \overline{I}$ and $\forall \overline{J} \sqsubset \overline{I}, {\cal T}_{\overline{B}^I \sqcup \overline{P}^I}(\overline{J}) \neq \overline{J}$ if ${\cal T}_{\overline{B}}(\overline{I}) \sqsubseteq \overline{I}$ by Proposition~\ref{prop:eq:con}. \\
    $\Rto$ $\overline{I} = \mathit{lfp}({\cal T}_{(\overline{B} \sqcup \overline{P})^I})$. \\
    $\Rto$ $\overline{I} \in \mathit{PSM}(\overline{B} \sqcup \overline{P})$.
\end{proof}


\setcounter{section}{3}
\setcounter{propositionA}{0}
\setcounter{corollaryA}{0}
\setcounter{theoremA}{0} 

\begin{propositionA}[Incomparability between poss-stable models]
    Given two different possibilistic interpretations $\overline{I}$ and $\overline{J}$ such that $\overline{I} \parallel \overline{J}$, $\{ \overline{I}, \overline{J} \} \not\subseteq \mathit{PSM}(\overline{P})$ for any poss-NLP $\overline{P}$.
\end{propositionA}
\begin{proof}
    Assume $\{ \overline{I}, \overline{J} \} \subseteq \mathit{PSM}(\overline{P})$.  Then we have $\overline{I} \in \mathit{PSM}(\overline{P})$ and $\overline{J} \in \mathit{PSM}(\overline{P})$. Suppose $I \subseteq J$ by Definition~\ref{def:Tp}.    
    There are only two cases: $I = J$ or $I \subset J$.

    {\bf Case 1:}$I = J$. \\
    $\Rto$ $\mathit{Cn}(\overline{P}^{I}) = \mathit{Cn}(\overline{P}^{J})$ when $I = J$ by the definition of the possibilistic consequence of a poss-NLP. \\
    $\Rto$ $\overline{I} = \overline{J}$ since $\overline{I} \in \mathit{PSM}(\overline{P})$ and $\overline{J} \in \mathit{PSM}(\overline{P})$.  \\
    $\Rto$ Contradict to $\overline{I} \neq \overline{J}$.
    
    {\bf Case 2:}$I \subset J$. \\
    $\Rto$ $I \subset J$ and $\{ I, J \} \subseteq \mathit{SM}(P)$ since $\{ \overline{I}, \overline{J} \} \subseteq \mathit{PSM}(\overline{P})$ and Proposition~\ref{prop:SM:maps}. \\
    $\Rto$ $I \subset J$ and $I \not\subset J$ since $\mathit{SM}(P)$ is $\subseteq$-incomparable. \\
    $\Rto$ A contradiction.
\end{proof}

\begin{propositionA}[Satisfiability for program $\mathit{PPE}(\overline{S})$]	 
    For a given set $\overline{S}$ of possibilistic interpretations, $\mathit{PSM}(\mathit{PPE}(\overline{S})) = \overline{S}$ if $\overline{S}$ is incomparable. 
\end{propositionA}
\begin{proof}
    It is evident that $\overline{S} = \mathit{PSM}(\mathit{PPE}(\overline{S}))$ when $\overline{S} = \emptyset$. In the case of $\overline{S} \neq \emptyset$, it can be proved as follows.
    
    ($\Longrightarrow$)
    $\overline{S}$ is incomparable. \\
    $\Rto$ $S$ is $\subseteq$-incomparable. \\
    $\Rto$ $I_1 \not\subset I_2$ for any two interpretations $I_1 \in S$ and $I_2 \in S$. \\
    $\Rto$ $\exists y \in I_1$ such that $ y \notin I_2$ for any two interpretations $I_1 \in S$ and $I_2 \in S$. \\
    $\Rto$ $I_1 \cap ({\cal A} - I_2) \neq \emptyset$ for any two interpretations $I_1 \in S$ and $I_2 \in S$.\\
    $\Rto$ For each $I \in S$, $\mathit{PPE}(\overline{S})^I = \{ (x \lto , \alpha) \mid (x,\alpha) \in \overline{I} \}$ since $I \cap ({\cal A} - I) = \emptyset$.\\
    $\Rto$ For each $\overline{I} \in \overline{S}$, $\mathit{lfp}({\cal T}_{\mathit{PPE}(\overline{S})^I}) = \overline{I}$.\\
    $\Rto$ For each $\overline{I} \in \overline{S}$, $\overline{I} \in \mathit{PSM}(\mathit{PPE}(\overline{S}))$.\\
    $\Rto$ $\overline{S} \subseteq \mathit{PSM}(\mathit{PPE}(\overline{S}))$. 

    ($\Longleftarrow$)For an possibilistic interpretation $\overline{K}$ such that $\overline{K} \notin \overline{S}$, it is evident that $\overline{K} \notin \mathit{PSM}(\mathit{PPE}(\overline{S}))$ when $K$ is $\subseteq$-comparable with respect to an element in $S$ because of $\overline{S} \subseteq \mathit{PSM}(\mathit{PPE}(\overline{S}))$ above and Proposition~\ref{prop:possSM:incomparable}. Let $\overline{K}$ be an possibilistic interpretation such that $K$ is $\subseteq$-incomparable \wrt every interpretation in $S$. \\
    $\Rto$ $\forall \overline{I} \in \overline{S}, K \cap ({\cal A} - I) \neq \emptyset$.\\
    $\Rto$ $K \cap \Neg(r) \neq \emptyset$ for each $\overline{r} \in \mathit{PPE}(\overline{S})$.\\
    $\Rto$ $\mathit{PPE}(\overline{S})^K = \emptyset$.\\
    $\Rto$ $\mathit{lfp}({\cal T}_{\mathit{PPE}(\overline{S})^{K}}) = \emptyset$.\\
    $\Rto$ $\overline{K} \notin \mathit{PPE}(\overline{S})^{K}$ since $K \neq \emptyset$ because $K$ is $\subseteq$-incomparable \wrt every interpretation in $S$. \\
    $\Rto$ $\overline{K} \notin \mathit{PSM}(\mathit{PPE}(\overline{S}))$ whenever $\overline{K} \notin \overline{S}$.\\
    $\Rto$ $\mathit{PSM}(\mathit{PPE}(\overline{S})) \subseteq \overline{S}$.
\end{proof}

\begin{propositionA}[Unsatisfiability for program $\mathit{PNE}(\overline{S_N},\overline{S_P})$]
    Given a poss-NLP $\overline{P}$ and two sets $\overline{S_N}$ and $\overline{S_P}$ of possibilistic interpretations, $\overline{I} \notin \mathit{PSM}(\mathit{PNE}(\overline{S_N},\overline{S_P}) \sqcup \overline{P})$ if $\overline{I} \in \overline{S_N}$, $I \neq {\cal A}$, and $ \overline{I} \not\parallel \overline{S_P}$.
\end{propositionA}
\begin{proof}
    $\overline{I} \in \overline{S_N}, I \neq {\cal A}, \overline{I} \not\parallel \overline{S_P}$. \\
    $\Rto$ By Definition~\ref{def:poss:comparable}, $\overline{I} \in \overline{S_N}, I \neq {\cal A}$, and $I$ is $\subseteq$-incomparable \wrt every interpretation in $S_P$.\\
    $\Rto$ $I \in S_N - \{ {\cal A} \}$, and $I$ is $\subseteq$-incomparable \wrt every interpretation in $S_P$.\\
    $\Rto$ $I \neq {\cal A}$\\
    $\Rto$ $I \models \Body(\mathit{PNE}(\overline{I}))$ and $I \not\models \Head(\mathit{PNE}(\overline{I}))$.\\
    $\Rto$ $I \not\models r$ where $\overline{r} = \mathit{PNE}(\overline{I})$.\\
    $\Rto$ $I \not\models H$ where $\overline{H} = \mathit{PNE}(\overline{S_N},\overline{S_P})$ since $I \in S_N - \{ {\cal A} \}$, and $I$ is $\subseteq$-incomparable \wrt every interpretation in $S_P$.\\
    $\Rto$ $I \not\models H \cup P$ where $\overline{H} = \mathit{PNE}(\overline{S_N},\overline{S_P})$. \\
    $\Rto$ $I \notin \mathit{SM}(H \cup P)$ where $\overline{H} = \mathit{PNE}(\overline{S_N},\overline{S_P})$.\\
    $\Rto$ $\overline{I} \notin \mathit{PSM}(\mathit{PNE}(\overline{S_N},\overline{S_P}) \sqcup \overline{P})$ by Proposition~\ref{prop:SM:maps}.
\end{proof}

\begin{lemmaA}[Unique stable model for program $\mathit{PPE}(\overline{G})$]
    Given a poss-NLP $\overline{B}$ and a possibilistic interpretation $\overline{G}$ with $G = {\cal A}$, if ${\cal T}_{\overline{B}}(\overline{G}) \sqsubseteq \overline{G}$, then $\mathit{PSM}(\overline{B} \sqcup \mathit{PPE}(\overline{G})) = \{ \overline{G} \}$.
\end{lemmaA}
\begin{proof}
    $\mathit{PPE}(\overline{G}) = \{ (x \lto , \alpha) \mid (x, \alpha) \in \overline{G} \}$ where $G = {\cal A}$.\\
    $\Rto$ $\mathit{PPE}(\overline{G})^G = \mathit{PPE}(\overline{G})$.\\
    $\Rto$ $\mathit{lfp}({\cal T}_{\mathit{PPE}(\overline{G})^G}) = \overline{G}$.\\
    $\Rto$ $\overline{G} \in \mathit{PSM}(\mathit{PPE}(\overline{G}))$.\\
    $\Rto$ $\overline{G} \in \mathit{PSM}(\overline{B} \sqcup \mathit{PPE}(\overline{G}))$ since Corollary~\ref{cor:SM:absorption} and ${\cal T}_{\overline{B}}(\overline{G}) \sqsubseteq \overline{G}$.\\
    $\Rto$ $\mathit{PSM}(\overline{B} \sqcup \mathit{PPE}(\overline{G})) = \{ \overline{G} \}$ since $G = {\cal A}$ and Proposition~\ref{prop:possSM:incomparable}.
\end{proof}

\begin{propositionA}[Existence of a program w.r.t. a poss-stable model]	
    For a possibilistic interpretation $\overline{I}$ and a poss-NLP $\overline{B}$, there exists a poss-NLP $\overline{P}$ such that $\overline{I} \in \mathit{PSM}(\overline{B} \sqcup \overline{P})$ if and only if ${\cal T}_{\overline{B}}(\overline{I}) \sqsubseteq \overline{I}$.
\end{propositionA}
\begin{proof}
    ($\Longrightarrow$) Assume (\romannumeral 1)${\cal T}_{\overline{B}}(\overline{I}) \not\sqsubseteq \overline{I}$, and (\romannumeral 2)$\overline{I} \in \mathit{PSM}(\overline{B} \sqcup \overline{P})$. \\
    $\Rto$ ${\cal T}_{\overline{B}^I}(\overline{I}) \not\sqsubseteq \overline{I}$  by Proposition~\ref{prop:eq:con}, and (\romannumeral 2)${\cal T}_{(\overline{B} \sqcup \overline{P})^I}(\overline{I}) = \overline{I}$. \\
    $\Rto$ ${\cal T}_{(\overline{B} \sqcup \overline{P})^I}(\overline{I}) \not\sqsubseteq \overline{I}$ since $(\overline{B} \sqcup \overline{P})^I = (\overline{B} \sqcup \overline{B} \sqcup \overline{P})^I = \overline{B}^I \sqcup (\overline{B} \sqcup \overline{P})^I$. \\
    $\Rto$ ${\cal T}_{(\overline{B} \sqcup \overline{P})^I}(\overline{I}) \neq \overline{I}$. \\    
    $\Rto$ $\overline{I} \notin fp({\cal T}_{(\overline{B} \sqcup \overline{P})^I})$. \\
    $\Rto$ $\overline{I} \neq \mathit{lfp}({\cal T}_{(\overline{B} \sqcup \overline{P})^I})$ since $\mathit{lfp}({\cal T}_{(\overline{B} \sqcup \overline{P})^I}) \in fp({\cal T}_{(\overline{B} \sqcup \overline{P})^I})$. \\
    $\Rto$ $\overline{I} \notin \mathit{PSM}(\overline{B} \sqcup \overline{P})$ by the definition of a poss-stable model. \\
    $\Rto$ Contradict to the assumption (\romannumeral 2)$\overline{I} \in \mathit{PSM}(\overline{B} \sqcup \overline{P})$.
    
    ($\Longleftarrow$) Let $\overline{P} = \mathit{PPE}(\{ \overline{I} \})$. \\
    $\Rto$ $\overline{I} \in \mathit{PSM}(\overline{P})$ by Proposition~\ref{prop:possNLP:exist}. \\
    $\Rto$ $\overline{I} \in \mathit{PSM}(\overline{B} \sqcup \overline{P})$ since Corollary~\ref{cor:SM:absorption} and ${\cal T}_{\overline{B}}(\overline{I}) \sqsubseteq \overline{I}$.
\end{proof}

\begin{corollaryA}
For an incomparable set $\overline{S}$ of possibilistic interpretations and a poss-NLP $\overline{B}$, there exists a poss-NLP $\overline{P}$ such that $\overline{S} \subseteq \mathit{PSM}(\overline{B} \sqcup \overline{P})$ if and only if ${\cal T}_{\overline{B}}(\overline{I}) \sqsubseteq \overline{I}$ for each $\overline{I} \in \overline{S}$.
\end{corollaryA}
\begin{proof}
    ($\Longrightarrow$)Assume $\neg \forall \overline{I} \in \overline{S}, {\cal T}_{\overline{B}}(\overline{I}) \sqsubseteq \overline{I}$. \\
    $\Rto$ $\exists \overline{I} \in \overline{S}$ such that ${\cal T}_{\overline{B}^I}(\overline{I}) \not\sqsubseteq \overline{I}$. \\
    $\Rto$ $\exists \overline{I} \in \overline{S}$ such that $\overline{I} \notin \mathit{PSM}(\overline{B} \sqcup \overline{P})$ for each poss-NLP $\overline{P}$ by Proposition~\ref{prop:possNLP:interpretation:exist}. \\
    $\Rto$ there does not exist a poss-NLP $\overline{P}$ such that $\overline{S} \subseteq \mathit{PSM}(\overline{B} \sqcup \overline{P})$.  \\
    $\Rto$ A contradiction.     

    ($\Longleftarrow$)$\forall \overline{I} \in \overline{S}, \overline{I} \in \mathit{PSM}(\mathit{PPE}(\overline{S}))$ by Proposition~\ref{prop:possNLP:exist}. \\
    $\Rto$ $\forall \overline{I} \in \overline{S}, \overline{I} \in \mathit{PSM}(\overline{B} \sqcup \mathit{PPE}(\overline{I}))$ since Corollary~\ref{cor:SM:absorption} and $\forall \overline{I} \in \overline{S}, {\cal T}_{\overline{B}}(\overline{I}) \sqsubseteq \overline{I}$. \\
    $\Rto$ $\overline{S} \subseteq \mathit{PSM}(\overline{B} \sqcup \mathit{PPE}(\overline{S}))$. \\
    $\Rto$ There exists a poss-NLP $\overline{P}$ such that $\overline{S} \subseteq \mathit{PSM}(\overline{B} \sqcup \overline{P})$.
\end{proof}

\begin{corollaryA}[Existence of a solution]
For a task $T = \tuple{\overline{B}, E^+, \emptyset}$, $\mathit{ILP_{LPoSM}}(T) \neq \emptyset$ if and only if $E^+$ is incomparable and $E^+$ is coherent with $\overline{B}$.
\end{corollaryA}
\begin{proof}
    It can be proved by Proposition~\ref{prop:possNLP:exist} and Corollary~\ref{cor:possNLP:B:exist}.
\end{proof}

\begin{propositionA}
Let $T = \tuple{\overline{B}, \emptyset, E^-}$ be an induction task and $\overline{\mathbb{A}}=\{\overline I\mid \overline{I}\mbox{ is a poss-interpretation with } I={\cal A}\}$. Then $\mathit{ILP_{LPoSM}}(T) = \emptyset$ if and only if the following three conditions are satisfied. 
\begin{enumerate}[label=(c\arabic*)]
    \item $\mathit{lfp}(T_{B^{\cal A}}) = {\cal A}$.
    \item $\overline{\mathbb{A}} - E^- \neq \overline{\mathbb{A}}$.  
    \item $\overline{\mathbb{A}} - E^- = \emptyset$, or ${\cal T}_{\overline{B}}(\overline{G}) \not\sqsubseteq \overline{G}$ for each $\overline{G} \in \overline{\mathbb{A}} - E^-$.    
\end{enumerate}
\end{propositionA}
\begin{proof}
    ($\Longrightarrow$) We will prove $\mathit{ILP_{LPoSM}}(T) \neq \emptyset$ if $c1 \land c2 \land c3$ does not hold. Namely, it can be proved that $\neg c1 \lor \neg c2 \lor \neg c3$ implies $\mathit{ILP_{LPoSM}}(T) \neq \emptyset$ in the following three cases. 
    
    {\bf Case 1:}($\neg c1$) Let $\overline{H} = \mathit{PNE}(E^-,\emptyset)$ when $\mathit{lfp}(T_{B^{\cal A}}) \neq {\cal A}$. \\
    $\Rto$ $H^{\cal A} = \emptyset$ since $\forall r \in H, \Neg(r) \cap {\cal A} \neq \emptyset$. \\
    $\Rto$ $\mathit{lfp}(T_{(B \cup H)^{\cal A}}) \neq {\cal A}$ since $\mathit{lfp}(T_{B^{\cal A}}) \neq {\cal A}$ and $\mathit{lfp}(T_{(B \cup H)^{\cal A}}) = \mathit{lfp}(T_{B^{\cal A} \cup H^{\cal A}}) = \mathit{lfp}(T_{B^{\cal A}})$. \\
    $\Rto$ ${\cal A} \notin \mathit{SM}(B \cup H)$. \\
    $\Rto$ $\forall \overline{e} \in E^-, e = {\cal A} \rto \overline{e} \notin \mathit{PSM}(\overline{B} \sqcup \overline{H})$ by Proposition~\ref{prop:SM:maps}.  \\
    $\Rto$ $\forall \overline{e} \in E^-, \overline{e} \notin \mathit{PSM}(\overline{B} \sqcup \overline{H})$ since $\forall \overline{e} \in E^-, e \neq {\cal A} \rto \overline{e} \notin \mathit{PSM}(\overline{B} \sqcup \overline{H})$ by Proposition~\ref{prop:PSM:PNE}.  \\
    $\Rto$ $\overline{H} \in \mathit{ILP_{LPoSM}}(T)$. \\
    $\Rto$ $\mathit{ILP_{LPoSM}}(T) \neq \emptyset$.

    {\bf Case 2:}($\neg c2$) Let $\overline{H} = \mathit{PNE}(E^-,\emptyset)$ when $\overline{\mathbb{A}} - E^- =\overline{\mathbb{A}}$. \\
    $\Rto$ $\overline{\mathbb{A}} \cap E^- = \emptyset$. \\
    $\Rto$ $\forall \overline{e} \in E^-, e \neq {\cal A}$. \\
    $\Rto$ $\forall \overline{e} \in E^-, \overline{e} \notin \mathit{PSM}(\overline{B} \sqcup \overline{H})$ by Proposition~\ref{prop:PSM:PNE}. \\
    $\Rto$ $\overline{H} \in \mathit{ILP_{LPoSM}}(T)$. \\
    $\Rto$ $\mathit{ILP_{LPoSM}}(T) \neq \emptyset$.

    {\bf Case 3:}($\neg c3$) Let $\overline{H} = \mathit{PPE}(\overline{G})$ when $\overline{\mathbb{A}} - E^- \neq \emptyset$ and $\exists \overline{G} \in \overline{\mathbb{A}} - E^-, {\cal T}_{\overline{B}}(\overline{G}) \sqsubseteq \overline{G}$. \\
    $\Rto$ $\overline{H} \neq \emptyset$ since $\overline{G} \in \overline{\mathbb{A}} - E^-$ and $\overline{\mathbb{A}} - E^- \neq \emptyset$.\\
    $\Rto$ $\mathit{PSM}(\overline{B} \sqcup \overline{H}) = \{ \overline{G} \}$ by Lemma~\ref{lm:SM:PPI}. \\
    $\Rto$ $\forall \overline{e} \in E^-, \overline{e} \notin \mathit{PSM}(\overline{B} \sqcup \overline{H})$ since $\overline{G} \notin E^-$. \\
    $\Rto$ $\overline{H} \in \mathit{ILP_{LPoSM}}(T)$. \\
    $\Rto$ $\mathit{ILP_{LPoSM}}(T) \neq \emptyset$.

    ($\Longleftarrow$) Assume $\mathit{ILP_{LPoSM}}(T) \neq \emptyset$ when three conditions hold simultaneously. Let $\overline{H} \in \mathit{ILP_{LPoSM}}(T)$.  \\
    $\Rto$ By condition (c3), we have (\romannumeral 1)$\overline{\mathbb{A}} - E^- = \emptyset$ or (\romannumeral 2)$\forall \overline{G} \in \overline{\mathbb{A}} - E^-, {\cal T}_{\overline{B}}(\overline{G}) \not\sqsubseteq \overline{G}$. \\
    $\Rto$ (\romannumeral 1)$\overline{\mathbb{A}} \subseteq E^-$ or (\romannumeral 2)$\forall \overline{G} \in \overline{\mathbb{A}} - E^-, \overline{G} \notin \mathit{PSM}(\overline{B} \sqcup \overline{H})$ by Proposition~\ref{prop:possNLP:interpretation:exist}. \\
    $\Rto$ $\forall \overline{I} \in \overline{\mathbb{A}}, \overline{I} \notin \mathit{PSM}(\overline{B} \sqcup \overline{H})$ since $\forall \overline{e} \in E^-, \overline{e} \notin \mathit{PSM}(\overline{B} \sqcup \overline{H})$. \\
    $\Rto$ ${\cal A} \notin \mathit{SM}(B \cup H)$ by Proposition~\ref{prop:SM:maps}.\\
    $\Rto$ A contradiction to ${\cal A} \in \mathit{SM}(B \cup H)$ since $\mathit{lfp}(T_{H^{\cal A}}) \subseteq {\cal A}$ and condition (c1).
\end{proof}


\begin{theoremA}[Existence of a poss-NLP solution for task]	
For a task $T = \tuple{\overline{B}, E^+, E^-}$, $\mathit{ILP_{LPoSM}}(T) \neq \emptyset$ if and only if 
\begin{enumerate}[label=(C\arabic*)]
    \item $E^+$ is incomparable, and
    \item $E^+$ is coherent with $\overline{B}$, and 
    \item $E^-$ is compatible with  $\overline{B}$, and 
    \item $E^+ \cap E^- = \emptyset$.
\end{enumerate}
\end{theoremA}
\begin{proof}
    ($\Longrightarrow$) It can be proved by contrapositive as follows.
    \begin{enumerate}[label=(\arabic*)]
        \item By Proposition~\ref{prop:possSM:incomparable}, comparable $E^+$ implies $\mathit{ILP_{LPoSM}}(T) = \emptyset$.
        \item By Corollary~\ref{cor:possNLP:B:exist} and Definition~\ref{def:incoherent}, $\mathit{ILP_{LPoSM}}(T) = \emptyset$ if $E^+$ is incoherent with $\overline{B}$.
        \item By Proposition~\ref{prop:task:B:neg:exist} and Definition~\ref{def:compatible}, $\mathit{ILP_{LPoSM}}(T) = \emptyset$ if $E^-$ is incompatible with  $\overline{B}$.
        \item It is evident that $E^+ \cap E^- \neq \emptyset$ implies $\mathit{ILP_{LPoSM}}(T) = \emptyset$.
    \end{enumerate}

    ($\Longleftarrow$) When $E^+ = \emptyset$, $\mathit{ILP_{LPoSM}}(T) \neq \emptyset$ by Condition (C3) and Proposition~\ref{prop:task:B:neg:exist}. When $E^+ \neq \emptyset$, let $\overline{H} = \mathit{PPE}(E^+) \sqcup \mathit{PNE}(E^-,E^+)$ where $\mathit{PPE}(E^+) \neq \emptyset$. we need to prove Conditions \eqref{eq:G1} and \eqref{eq:G2} achieved by Definition~\ref{def:ILT}. \\
    $\forall \overline{e} \in E^+ \forall \overline{I} \in E^-, \overline{I} \not\parallel E^+ \rto e \neq \Pos(\mathit{PNE}(\overline{I}))$. \\
    $\Rto$ $\forall \overline{e} \in E^+ \forall \overline{r} \in \mathit{PNE}(E^-,E^+), e \neq \Pos(r)$. \\
    $\Rto$ $\forall \overline{e} \in E^+ \forall \overline{r} \in \mathit{PNE}(E^-,E^+), \overline{r}$ is not applicable in $\overline{e}$ by Definition~\ref{def:applicable}. \\
    $\Rto$  $\forall \overline{e} \in E^+, {\cal T}_{\mathit{PNE}(E^-,E^+)}(\overline{e}) = \emptyset$ by Definition~\ref{def:Tp}.  \\
    $\Rto$ $\forall \overline{e} \in E^+, {\cal T}_{\mathit{PNE}(E^-,E^+)}(\overline{e}) \sqsubseteq \overline{e}$.  \\
    $\Rto$ $\forall \overline{e} \in E^+, {\cal T}_{\overline{B} \sqcup \mathit{PNE}(E^-,E^+)}(\overline{e}) \sqsubseteq \overline{e}$ by Condition (C2).   \\
    $\Rto$ Condition \eqref{eq:G1} holds because Corollary~\ref{cor:SM:absorption} and $\forall \overline{e} \in E^+, \overline{e} \in \mathit{PSM}(\mathit{PPE}(E^+))$ by Proposition~\ref{prop:possNLP:exist} and Conditions (C1) and (C2).  \\ 
    $\Rto$ When $\overline{e} \in E^-$ is a possibilistic interpretation such that $\overline{e} \not\parallel E^+$ does not hold, $\overline{e} \notin \mathit{PSM}(\overline{B} \sqcup \overline{H})$ since Proposition~\ref{prop:possSM:incomparable}, Definition~\ref{def:poss:comparable} and Condition (C4). \\
    $\Rto$ Condition \eqref{eq:G2} holds because $\forall \overline{e} \in E^-, \overline{e} \not\parallel E^+ \rto \overline{e} \notin \mathit{PSM}(\overline{B} \sqcup \overline{H})$ since Proposition~\ref{prop:PSM:PNE}, Definition~\ref{def:poss:comparable} and the fact $\forall \overline{I} \in E^-, I \subseteq {\cal A}$.
\end{proof}

\begin{propositionA}[Solution for induction task containing fact rules]	
    Let $T = \tuple{\overline{B}, E^+, E^-}$ be an induction task.
    If there exists $a \in {\cal A}$ such that $(a \lto , \mu) \in \overline{B}$, then 
    \begin{enumerate}[label=(\roman*)]
        \item $\mathit{ILP_{LPoSM}}(T) = \emptyset$ when there exists $\overline{e} \in E^+$ such that $(a, \mu) \notin \overline{e}$, and
        \item for any poss-NLP $\overline{H}$, $(a, \mu) \notin \overline{e}$ implies $\overline{e} \notin \mathit{PSM}(\overline{B} \sqcup \overline{H})$ where $\overline{e} \in E^-$, and  
        \item if $\overline{H} \in \mathit{ILP_{LPoSM}}(T)$ and $a \in \Head(H)$, then there exists $\overline{K} \in \mathit{ILP_{LPoSM}}(T)$ such that $\vert K \vert < \vert H \vert$. 
    \end{enumerate}
\end{propositionA}
\begin{proof}
    ({\romannumeral1})
    $\exists \overline{e} \in E^+, (a, \mu) \notin \overline{e}$. \\
    $\Rto$ $\exists \overline{e} \in E^+, {\cal T}_{\overline{B}}(\overline{e}) \not\sqsubseteq \overline{e}$ since $(a, \mu) \in T_{\overline{B}}(\overline{e})$. \\
    $\Rto$ There exists $\overline{e} \in E^+$ such that $\overline{e}$ is incoherent with $\overline{B}$. \\
    $\Rto$ $E^+$ is incoherent with $\overline{B}$ by Definition~\ref{def:incoherent}. \\
    $\Rto$ $\mathit{ILP_{LPoSM}}(T) = \emptyset$ by Corollary~\ref{cor:B:pos}.
    
    ({\romannumeral2})
    $(a, \mu) \notin \overline{e}$ where $\overline{e} \in E^-$.  \\
    $\Rto$ ${\cal T}_{\overline{B}}(\overline{e}) \not\sqsubseteq \overline{e}$ since $(a, \mu) \in T_{\overline{B}}(\overline{e})$.  \\
    $\Rto$ By Proposition~\ref{prop:possNLP:interpretation:exist}, $\overline{e} \notin \mathit{PSM}(\overline{B} \sqcup \overline{H})$ where $\overline{H}$ is a poss-NLP.
    
    ({\romannumeral3})
    $\overline{H} \in \mathit{ILP_{LPoSM}}(T)$ and $a \in \Head(H)$.  \\
    $\Rto$ $\overline{H} \in \mathit{ILP_{LPoSM}}(T)$ and $\exists \overline{r} \in \overline{H}, \Head(r) = a$.  \\
    $\Rto$ $\mathit{PSM}(\overline{B} \sqcup \overline{H}) = \mathit{PSM}(\overline{B} \sqcup \overline{K})$ and $\overline{K} \subset \overline{H}$ where $\overline{K} = \{ \overline{r} \in \overline{H} \mid \Head(r) \neq a \}$ since $(a \lto , \mu) \in \overline{B}$.  \\
    $\Rto$ $\overline{K} \in \mathit{ILP_{LPoSM}}(T)$ and $\vert K \vert < \vert H \vert$.  \\
    $\Rto$ $\exists \overline{K} \in \mathit{ILP_{LPoSM}}(T)$ such that $\vert K \vert < \vert H \vert$.
\end{proof}

\setcounter{section}{4}
\setcounter{propositionA}{0}
\setcounter{lemmaA}{0}
\setcounter{theoremA}{0} 
\setcounter{corollaryA}{0}

\begin{propositionA}[Correctness of algorithm ILPSM]
    For a task $T = \tuple{\overline{B}, E^+, E^-}$, algorithm {\ILPSM} returns {\bf fail} when $\mathit{ILP_{LPoSM}}(T) = \emptyset$. Otherwise, algorithm {\ILPSM} returns a solution of $T$.
\end{propositionA}
\begin{proof}
    According to Theorem~\ref{thm:taskNLP:exist}, algorithm {\ILPSM} directly returns {\bf fail} when   $\mathit{ILP_{LPoSM}}(T) = \emptyset$. Therefore, it returns   $\overline{H}$ only if $T$ has a solution. The solution of $T$ comes from two cases only. 
    
    When $E^+ \neq \emptyset$ in line (3), $\mathit{PPE}(E^+)$ is a solution of $\tuple{\overline{B}, E^+, \emptyset}$ according to Theorem~\ref{thm:taskNLP:exist}. Line (5) ensures ${\cal A} \notin E$ as $E^+ \neq \emptyset$. Then adding $\mathit{PNE}(\overline{E},E^+)$ into $\overline{H}$ achieves Condition \eqref{eq:G2} of Definition~\ref{def:ILT} according to Proposition~\ref{prop:PSM:PNE}.

    When $E^+ = \emptyset$ in line (7), the task becomes $\tuple{\overline{B}, \emptyset, E^-}$ so that we care only covering the negative examples. According to the proof of Proposition~\ref{prop:task:B:neg:exist}, the constructed $\overline{H}$ must be a solution of $T$.
\end{proof}

\begin{propositionA}[Construction of $S^+(\overline{I})$]
Given a possibilistic interpretation $\overline{I}$, $S^+(\overline{I}) = \{ P \mid P \subseteq \cup_{\epsilon \in \overline{I}} S^+(\overline{I}, \epsilon) \mbox{ and } \forall \epsilon \in \overline{I}, \vert P \cap S^+(\overline{I}, \epsilon) \vert = 1 \}$. 
\end{propositionA}
\begin{proof}
    ($\Longrightarrow$) Let $X \in \{ P \mid P \subseteq \cup_{\epsilon \in \overline{I}} S^+(\overline{I}, \epsilon) \mbox{ and } \forall \epsilon \in \overline{I}, \vert P \cap S^+(\overline{I}, \epsilon) \vert = 1 \}$. \\
    $\Rto$ $\forall \epsilon \in \overline{I}, \vert P \cap S^+(\overline{I}, \epsilon) \vert = 1$. \\
    $\Rto$ $X$ is a minimal hitting set of $\{ S^+(\overline{I}, \epsilon) \mid \epsilon \in \overline{I}\}$. \\
    $\Rto$ $X \in \mathit{SMHS}(\{ S^+(\overline{I}, \epsilon) \mid \epsilon \in \overline{I}\})$. \\
    $\Rto$ By Definition~\ref{def:pos:SS}, $X \in S^+(\overline{I})$. 
    
    ($\Longleftarrow$) Let $X \in S^+(\overline{I})$. \\
    $\Rto$ By Definition~\ref{def:pos:SS}, $X \in \mathit{SMHS}(\{ S^+(\overline{I}, \epsilon) \mid \epsilon \in \overline{I}\})$. \\
    $\Rto$ (\romannumeral 1) $X \subseteq \cup_{\epsilon \in \overline{I}} S^+(\overline{I}, \epsilon)$, and (\romannumeral 2) $\forall \epsilon \in \overline{I}, X \cap S^+(\overline{I}, \epsilon) \neq \emptyset$, and (\romannumeral 3) no proper subset of $X$ satisfies the former two conditions. \\
    $\Rto$ By Definition~\ref{def:pos:SS}, $\forall \{ \epsilon_1, \epsilon_2 \} \subseteq \overline{I}, S^+(\overline{I}, \epsilon_1) \cap S^+(\overline{I}, \epsilon_2) = \emptyset$. So we have (\romannumeral 1) $X \subseteq \cup_{\epsilon \in \overline{I}} S^+(\overline{I}, \epsilon)$, and (\romannumeral 2) $\forall \epsilon \in \overline{I}, \vert P \cap S^+(\overline{I}, \epsilon) \vert = 1$. \\
    $\Rto$ $X \in \{ P \mid P \subseteq \cup_{\epsilon \in \overline{I}} S^+(\overline{I}, \epsilon) \mbox{ and } \forall \epsilon \in \overline{I}, \vert P \cap S^+(\overline{I}, \epsilon) \vert = 1 \}$.     
\end{proof}

\begin{lemmaA}[Positive solution space for the least fixpoint]
    Let possibilistic interpretation $\overline{I}$ be the least fixpoint of a possibilistic definite logic program $\overline{P}$. There exists a program $\overline{H} \subseteq \overline{P}$ such that $\overline{H}  \in S^+(\overline{I})$ and $H$ is grounded.
\end{lemmaA}
\begin{proof}
    $\overline{P}$ is a possibilistic definite logic program.  \\
    $\Rto$ Function ${\cal T}_{\overline{P}}$ is monotone by its definition. Namely, ${\cal T}_{\overline{P}}(\overline{J}) \sqsubseteq {\cal T}_{\overline{P}}(\overline{K})$ if $\overline{J} \sqsubseteq \overline{K}$. \\
    $\Rto$ ${\cal T}_{\overline{P}}^k \sqsubseteq {\cal T}_{\overline{P}}^{k+1}$ for $k \geq 0$.  \\
    $\Rto$ During the iteration towards $\bigsqcup_{n\geq0}{\cal T}_{\overline{P}}^n$, each $\epsilon \in \overline{I}$ arises and subsequently remains unchanged since $\overline{I}$ is the least fixpoint of $\overline{P}$. 
    Along with each $\epsilon \in \overline{I}$ arising for the first time, there exists one rule $\overline{r} \in \overline{P}$ such that $\Head(r) = p$ and $\overline{r}$ is $\alpha$-applicable in a possibilistic interpretation $\overline{J} \sqsubseteq \overline{I}$. Put these $\vert I \vert$ rules together to form a program $\overline{H}$. \\
    $\Rto$ There exists a program $\overline{H} \subseteq \overline{P}$ such that $\overline{H}  \in S^+(\overline{I})$ and $H$ is grounded.
\end{proof}

\begin{propositionA}[Positive solution space for a poss-stable model]
    If $\overline{I} \in \mathit{PSM}(\overline{P})$ for a poss-NLP $\overline{P}$ and a possibilistic interpretation $\overline{I}$, then there exists $\overline{H} \subseteq \overline{P}$ such that $\overline{H}  \in S^+(\overline{I})$ and $H^I$ is grounded.
\end{propositionA}
\begin{proof}
    $\overline{I} \in \mathit{PSM}(\overline{P})$.  \\
    $\Rto$ $\overline{I} = \mathit{Cn}(\overline{P}^I)$.  \\
    $\Rto$ $\overline{I} = \bigsqcup_{n\geq0}{\cal T}_{\overline{P}^I}^n$.  \\
    $\Rto$ By Definition~\ref{def:pos:SS} and Lemma~\ref{lm:SM:LFP}, then there exists $\overline{H} \subseteq \overline{P}$ such that $\overline{H}  \in S^+(\overline{I})$ and $H^I$ is grounded.
\end{proof}

\begin{propositionA}[Loops in grounded program]	
    Let $\overline{I}$ be a possibilistic interpretation and $\overline{P}$ be a poss-NLP such that $\overline{P} \in S^+(\overline{I})$. Then $P^I$ is grounded if and only if the dependency graph of $P$ does not have a positive loop.
\end{propositionA}
\begin{proof}
    We should only prove $P^I$ is not grounded if and only if the dependency graph of $P$ has a positive loop.
	
	The dependency graph of $P$ has a positive loop. \\
	$\LRto$ The dependency graph of $P^I$ has a positive loop, because $P^I = \{ p\lto \Pos \mid (p\lto \Pos, \Not \Neg, \quad \beta) \in \overline{P} \}$ as $\forall r \in P, \Neg(r) \subseteq {\cal A}-I$. \\
	$\LRto$ $P^I$ is not grounded, because $\forall a \in P, \vert \{r \in P \mid \Head(r) = a\} \vert = 1$ as $\overline{P} \in S^+(\overline{I})$ and by Definition~\ref{def:pos:SS}.      
\end{proof}

\begin{lemmaA}[Incoherent interpretation with a poss-NLP]	
For a possibilistic interpretation $\overline{I}$ and poss-NLP $\overline{P}$, $\overline{I}$ is incoherent with $\overline{P}$ if and only if $\overline{P} \cap S^-(\overline{I}) \neq \emptyset$.    
\end{lemmaA}
\begin{proof}
    $\overline{P} \cap S^-(\overline{I}) \neq \emptyset$. \\
    $\LRto$ $\exists \overline{r} \in \overline{P}, \overline{r} \in S^-(\overline{I})$. \\
    $\LRto$ There exists $\overline{r} \in \overline{P}$ such that ${\cal T}_{\{ \overline{r} \}}(\overline{I}) \not\sqsubseteq \overline{I}$ by Definition~\ref{def:neg:SS} and Definition~\ref{def:Tp}. \\
    $\LRto$ $\overline{I}$ is incoherent with $\overline{P}$ by Definition~\ref{def:incoherent}.
\end{proof}

\begin{theoremA}[Conditions for a poss-stable model]
    Let $\overline{P}$ be a poss-NLP, $\overline{I}$ a possibilistic interpretation. $\overline{I} \in \mathit{PSM}(\overline{P})$ if and only if
    \begin{enumerate}[label=(c\arabic*)]
        \item $\overline{P} \cap S^-(\overline{I}) = \emptyset$, and 
        \item there exists $\overline{H} \subseteq \overline{P}$ such that $\overline{H}  \in S^+(\overline{I})$ and $H^I$ is grounded. 
    \end{enumerate}
\end{theoremA}
\begin{proof}
    ($\Longrightarrow$) By Proposition~\ref{prop:possNLP:interpretation:exist} and Definition~\ref{def:incoherent} and Lemma~\ref{lm:neg:SS}, Condition (c1) holds if $\overline{I} \in \mathit{PSM}(\overline{P})$.    
    By Proposition~\ref{prop:SS:exists}, Condition (c2) holds if $\overline{I} \in \mathit{PSM}(\overline{P})$. 

    ($\Longleftarrow$) By Definition~\ref{def:incoherent} and Lemma~\ref{lm:neg:SS}, ${\cal T}_{\overline{P}}(\overline{I}) \sqsubseteq \overline{I}$ if Condition (c1) holds. By Condition (c2), $\overline{I} \sqsubseteq \bigsqcup_{n\geq0}{\cal T}_{\overline{P}^I}^n$.
    Under these two opposite restrictions, we have $\overline{I} = \bigsqcup_{n\geq0}{\cal T}_{\overline{P}^I}^n$. In other words, $\overline{I} = \mathit{Cn}(\overline{P}^I)$ and further $\overline{I} \in \mathit{PSM}(\overline{P})$.
\end{proof}

\begin{corollaryA}[Negative examples for a solution]
Let $T = \tuple{\overline{B}, E^+, E^-}$ be an induction task and $\overline{H}$ be a poss-NLP such that $\overline{H} \in \mathit{ILP_{LPoSM}}(T)$. If $\overline{P} = \overline{B} \sqcup \overline{H}$, then the following two statements hold.
\begin{enumerate}[label=(\roman*)]
    \item $\overline{I} \in E^-$ if and only if at least one of the two conditions (c1) or (c2) in Theorem~\ref{thm:SS:iff} does not hold.
    \item The condition (c2) in Theorem~\ref{thm:SS:iff} does not hold when $\overline{I} = \{ (a,\mu) \mid a \in {\cal A} \} \in E^-$ where $\mu$ is the supremum of ${\cal Q}$ in $({\cal Q},\le)$.
\end{enumerate}
\end{corollaryA}
\begin{proof}
    The first conclusion is evident by Theorem~\ref{thm:SS:iff} and Definition~\ref{def:ILT}. If $\overline{I} = \{ (a,\mu) \mid a \in {\cal A} \} \in E^-$, then $S^-(\overline{I}) = \emptyset$ by Definition~\ref{def:neg:SS}. We have $\overline{P} \cap S^-(\overline{I}) = \emptyset$ for any poss-NLP $\overline{P}$. So the second conclusion holds too.
\end{proof}

\begin{propositionA}[Rules in a minimal solution]
    Let $\overline{H}$ be a minimal solution for an induction task $T = \tuple{\overline{B}, E^+, E^-}$. 
    For any $\overline{r} \in \overline{H}$
    \begin{enumerate}[label=(\roman*)]
        \item there exists $\overline{I} \in E^+$ and $\epsilon \in \overline{I}$ such that $\overline{r} \in S^+(\overline{I}, \epsilon)$, or 
        \item there exists $\overline{J} \in E^-$ such that $\overline{r} \in S^-(\overline{J})$.   
    \end{enumerate}  
\end{propositionA}
\begin{proof}
    $\overline{H}$ is a minimal solution of $T = \tuple{\overline{B}, E^+, E^-}$. \\
    $\Rto$ By Definition~\ref{def:minimal:solution}, $\overline{H} \in \mathit{ILP_{LPoSM}}(T)$ and $\nexists \overline{K} \in \mathit{ILP_{LPoSM}}(T),\vert K \vert < \vert H \vert$. \\
    $\Rto$ For every $\overline{r} \in \overline{H}$, $\overline{K} = \overline{H} - \{ \overline{r} \}$ is not a solution of $T$. \\    
    $\Rto$ For every $\overline{K} = \overline{H} - \{ \overline{r} \}$ where $\overline{r} \in \overline{H}$, $\exists \overline{I} \in E^+, \overline{I} \notin \mathit{PSM}(\overline{B} \sqcup \overline{K})$  or  $\exists \overline{J} \in E^-, \overline{J} \in \mathit{PSM}(\overline{B} \sqcup \overline{K})$ by Definition~\ref{def:ILT}.  \\
    $\Rto$ In contrast, $\overline{I} \in \mathit{PSM}(\overline{B} \sqcup \overline{H})$ and $\overline{J} \notin \mathit{PSM}(\overline{B} \sqcup \overline{H})$
    since $\overline{H} \in \mathit{ILP_{LPoSM}}(T)$. By Theorem~\ref{thm:SS:iff}, this contrast implies two possibilities after removing $\overline{r} \in \overline{H}$. (\romannumeral 1) There exists $\overline{I} \in E^+$ such that Condition (c1) or Condition (c2) of Theorem~\ref{thm:SS:iff} changes from true to false. Or (\romannumeral 2) There exists $\overline{I} \in E^-$ such that Condition (c1) or Condition (c2) of Theorem~\ref{thm:SS:iff} changes from false to true.  \\
    $\Rto$ After eliminating two impossible events, there are exactly two possibilities after removing $\overline{r} \in \overline{H}$ from $\overline{P} = \overline{B} \sqcup \overline{H}$. (\romannumeral 1) There exists $\overline{I} \in E^+$ such that Condition (c2) of Theorem~\ref{thm:SS:iff} changes from true to false. Or (\romannumeral 2) There exists $\overline{I} \in E^-$ such that Condition (c1) of Theorem~\ref{thm:SS:iff} changes from false to true. \\
    $\Rto$ For every $\overline{r} \in \overline{H}$, $\exists \overline{I} \in E^+ \exists \epsilon \in \overline{I}, \overline{r} \in S^+(\overline{I}, \epsilon)$ by Definition~\ref{def:pos:SS}, or $\exists \overline{J} \in E^-, \overline{r} \in S^-(\overline{J})$ by Lemma~\ref{lm:neg:SS}.
\end{proof}

\begin{propositionA}[Correctness of algorithm ILPSMmin]
    For a task $T = \tuple{\overline{B}, E^+, E^-}$, algorithm {\ILPSMmin} returns {\bf fail} when $\mathit{ILP_{LPoSM}}(T) = \emptyset$. Otherwise, algorithm {\ILPSMmin} returns a minimal solution of $T$.
\end{propositionA}
\begin{proof}
By Theorem~\ref{thm:taskNLP:exist}, algorithm {\ILPSMmin} outputs {\bf fail} if and only if $\mathit{ILP_{LPoSM}}(T) = \emptyset$. 
When the algorithm does not output {\bf fail}, the task $T$ must have a solution, which implies the existence of a minimal solution, allowing the algorithm to execute the steps after line (1).
The search process can be divided into two main parts: lines (4-8) ensure $E^+$ covered, while lines (9-20) ensure $E^-$ uncovered. Besides, the minimal solution must be taken into account, i.e., the minimality must be ensured.

For the first part, we need to prove that 
\begin{enumerate}[label=(c\arabic*)]
	\item each poss-NLP in $seeds$ can cover $E^+$; 
	\item beyond the range of $seeds$, no other poss-NLP covering $E^+$ has fewer rules. 
\end{enumerate}
Condition (c1) can be further subdivided into two aspects for proof, corresponding to the two conditions of Theorem~\ref{thm:SS:iff}.
\begin{itemize}
    \item In line (3), the negative solution space is utilized to obtain the $blacklist$ corresponding to $E^+$. The use of the blacklist in line (6) ensures condition (\romannumeral1) of Theorem~\ref{thm:SS:iff}.    
    \item Line (6) employs the positive solution space to construct candidate rules. By Proposition~\ref{prop:SS:noloop}, the groundness has also been taken into consideration. So the condition (\romannumeral2) of Theorem~\ref{thm:SS:iff} is ensured.
\end{itemize}
When $E^+ \neq \emptyset$, line (8) collects all poss-NLPs with the minimum number of rules to cover $E^+$. 
As a result, it can be concluded that any $X \in seeds$ satisfies that $X \sqcup \overline{B}$ covers $E^+$. By Definition~\ref{def:ILT}, each poss-NLP in $ seeds $ is a solution of $\tuple{\overline{B}, E^+, \emptyset}$. 
Based on this result and Definition~\ref{def:minimal:solution}, condition (c2) is also satisfied since line (6) uses the minimal hitting set via the positive solution space. 

For the second part, we should prove that $ E^- $ are not covered while no more unnecessary rule is added into the last poss-NLP. It can be proven that the execution of two branches in lines (10-20) of the algorithm achieves both of these requirements.
\begin{itemize}
    \item Case one corresponds to lines (10-12). It does not add extra rules. 
    The condition in line (10) ensures that negative examples $ E^- $ are not covered, and the updates to $ \overline{H} $ in lines (11-12) do not add extra rules to $ X $. 
    Thereby, $ E^- $ are not covered while added rules are as less as possible.
    \item Case two corresponds to lines (13-20). Extra rules should be added. 
    The simplification operation in line (14) figures out the set $ \overline{E} $ of all negative examples that need further consideration. 
    Line (15) takes the negative solution space and $blacklist$ into account to compute $whitelist$ where the maximal number of needed rules is $ \vert \overline{E} \vert$. 
    Therefore, the loop in line (16) from 1 rule to $ \vert \overline{E} \vert $ rules is sufficient to exclude negative examples, ensuring that the poss-NLP patches appearing in line (17) do not miss the minimal solution. 
    Under this premise, lines (18-20) can ensure that $ E^- $ are not covered while added rules are as less as possible, for the same reasons as in Case one.
\end{itemize}

In summary, the algorithm {\ILPSMmin} is correct.
\end{proof}

\setcounter{section}{5}
\setcounter{propositionA}{0}
\setcounter{lemmaA}{0}
\setcounter{theoremA}{0} 
\setcounter{corollaryA}{0}

\begin{propositionA}[Existence of a poss-NLP solution for task in Definition~\ref{def:ILT:comp:NLP}]	  
For a task $T = \tuple{\overline{B}, E^+}$ from Definition~\ref{def:ILT:comp:NLP}, $\mathit{ILP_{LCPoSM}}(T) \neq \emptyset$ if and only if 
\begin{enumerate}[label=(C\arabic*)]
    \item $E^+$ is incomparable, and
    \item $E^+$ is coherent with $\overline{B}$, and 
    \item (\romannumeral 1)$\mathit{lfp}(T_{B^{\cal A}}) \neq {\cal A}$, or (\romannumeral 2)$\vert {\cal Q} \vert = 1$ and $E^+ = \{ \overline{\mathbb{A}} \}$, or (\romannumeral 3)$\overline{\mathbb{A}} \cap E^+ \neq \emptyset$.
\end{enumerate}
\end{propositionA}
\begin{proof}
    Transform a task $T_1 = \tuple{ \overline{B}, E^+}$ from Deinition~\ref{def:ILT:comp:NLP} into the LSM task $T_2 = \tuple{ \overline{B}, E^+, E^-}$ where $E^- = \overline{\mathbb{U}} - E^+$. Now we have $\mathit{ILP_{LCPoSM}}(T_1) = \mathit{ILP_{LPoSM}}(T_2)$. By Theorem~\ref{thm:taskNLP:exist}, $\mathit{ILP_{LCPoSM}}(T_1) \neq \emptyset$ if and only if 
    \begin{enumerate}[label=(\arabic*)]
        \item $E^+$ is incomparable, and
        \item $E^+$ is coherent with $\overline{B}$, and 
        \item $E^-$ is compatible with  $\overline{B}$, and 
        \item $E^+ \cap E^- \neq \emptyset$.
    \end{enumerate}
    It is evident that $E^+ \cap E^- \neq \emptyset$ and $\overline{\mathbb{A}} - E^- =  \overline{\mathbb{A}} \cap E^+$ since $E^- = \overline{\mathbb{U}} - E^+$. By Definition~\ref{def:compatible}, $\mathit{ILP_{LCPoSM}}(T_1) \neq \emptyset$ if and only if 
    \begin{enumerate}[label=(C\arabic*)]
        \item $E^+$ is incomparable, and
        \item $E^+$ is coherent with $\overline{B}$, and 
        \item (\romannumeral 6)$\mathit{lfp}(T_{B^{\cal A}}) \neq {\cal A}$, or (\romannumeral 7)$\overline{\mathbb{A}} \cap E^+ = \overline{\mathbb{A}}$, or (\romannumeral 8)$\overline{\mathbb{A}} \cap E^+ \neq \emptyset$ and $\exists \overline{G} \in \overline{\mathbb{A}} \cap E^+, {\cal T}_{\overline{B}}(\overline{G}) \sqsubseteq \overline{G}$.
    \end{enumerate}
    To prove this proposition, we only need to prove $(\romannumeral 2) \LRto (\romannumeral 7)$ and $(\romannumeral 3) \LRto (\romannumeral 8)$ when Conditions (C1) and (C2) hold. 

    (\romannumeral 7)$\overline{\mathbb{A}} \cap E^+ = \overline{\mathbb{A}}$.  \\
    $\LRto$ $\overline{\mathbb{A}} \subseteq E^+$. \\
    $\LRto$ As $E^+$ is incomparable by Condition (C1), $E^+ = \{ \overline{\mathbb{A}} \}$ and $\vert \overline{\mathbb{A}} \vert = 1$. \\
    $\LRto$ $E^+ = \{ \overline{\mathbb{A}} \}$ and $\vert {\cal Q} \vert = 1$ by the definition of $\overline{\mathbb{A}}$. \\
    $\LRto$ (\romannumeral 2)$\vert {\cal Q} \vert = 1$ and $E^+ = \{ \overline{\mathbb{A}} \}$.

    $E^+$ is coherent with $\overline{B}$ by Condition (C2). \\
    $\Rto$ By Definition~\ref{def:incoherent}, $\forall \overline{e} \in E^+, {\cal T}_{\overline{B}}(\overline{e}) \sqsubseteq \overline{e}$. \\
    $\Rto$ $\forall \overline{G} \in \overline{\mathbb{A}} \cap E^+, {\cal T}_{\overline{B}}(\overline{G}) \sqsubseteq \overline{G}$. \\
    $\Rto$ $\exists \overline{G} \in \overline{\mathbb{A}} \cap E^+, {\cal T}_{\overline{B}}(\overline{G}) \sqsubseteq \overline{G}$. \\
    $\Rto$ $(\romannumeral 3) \LRto (\romannumeral 8)$.
\end{proof}

\begin{propositionA}[Existence of a NLP solution for a LSM induction task]  
For a task $T = \tuple{B, E^+, E^-}$, $\mathit{ILP_{LSM}}(T) \neq \emptyset$ if and only if 
\begin{enumerate}[label=(c\arabic*)]
    \item $E^+$ is $\subseteq$-incomparable, and
    \item $e \models B$ for each $e \in E^+$, and 
    \item ${\cal A} \notin E^-$ or ${\cal A} \neq \mathit{lfp}(T_{B^{\cal A}})$, and 
    \item $E^+ \cap E^- = \emptyset$.
\end{enumerate}
\end{propositionA}
\begin{proof}
    Let us view $T = \tuple{B, E^+, E^-}$ as a special task $T' = \tuple{\overline{B}, O^+, O^-}$ in which $\vert {\cal Q} \vert = 1$. Then $\mathit{ILP_{LPoSM}}(T') \neq \emptyset$ if and only if $\mathit{ILP_{LSM}}(T) \neq \emptyset$. After such a transformation, we need to prove Conditions (c1-c4) here holds if and only if Conditions (C1-C4) in Theorem~\ref{thm:taskNLP:exist} hold. As these conditions correspond one to one, we will prove it via four parts.

    (\uppercase\expandafter{\romannumeral1})
    $O^+$ is incomparable by Condition (C1).   \\
    $\LRto$ $\forall \overline{I} \in O^+ \forall \overline{J} \in O^+, \overline{I} \not\parallel \overline{J}$.   \\
    $\LRto$ $\forall I \in E^+ \forall J \in E^+$, $I$ and $J$ are $\subseteq$-comparable since $\vert {\cal Q} \vert = 1$.  \\
    $\LRto$ Condition (c1).

    (\uppercase\expandafter{\romannumeral2})
    $O^+$ is coherent with $\overline{B}$  by Condition (C2).   \\
    $\LRto$ $\forall \overline{I} \in O^+ \forall \overline{r} \in \overline{B}, {\cal T}_{\overline{B}}(\overline{I}) \sqsubseteq \overline{I}$ by Definition~\ref{def:incoherent}.   \\
    $\LRto$ $\forall I \in E^+ \forall r \in B, T_{B}(I) \subseteq I$ since $\vert {\cal Q} \vert = 1$.  \\
    $\LRto$  $\forall I \in E^+ \forall r \in B, I \models \Body(r) \rto \Head(r) \in I$.   \\
    $\LRto$  $\forall I \in E^+ \forall r \in B, I \models r$.    \\
    $\LRto$ $\forall I \in E^+, I \models B$.  \\
    $\LRto$ Condition (c2).

    (\uppercase\expandafter{\romannumeral3})
    $O^-$ is compatible with  $\overline{B}$  by Condition (C3).   \\
    $\LRto$ By Definition~\ref{def:compatible}, (\romannumeral 1) $\mathit{lfp}(T_{B^{\cal A}}) \neq {\cal A}$, or (\romannumeral 2) $\overline{\mathbb{A}} - O^- = \overline{\mathbb{A}}$, or (\romannumeral 3) $\overline{\mathbb{A}} - O^- \neq \emptyset$ and $\exists \overline{G} \in \overline{\mathbb{A}} - O^-, {\cal T}_{\overline{B}}(\overline{G}) \sqsubseteq \overline{G}$.   \\
    $\LRto$ As $\vert {\cal Q} \vert = 1$, (\romannumeral 1) $\mathit{lfp}(T_{B^{\cal A}}) \neq {\cal A}$, or (\romannumeral 2) $\{ {\cal A} \} - E^- = \{ {\cal A} \}$, or (\romannumeral 3) $\{ {\cal A} \} - E^- \neq \emptyset$ and $\exists G \in \{ {\cal A} \} - E^-,  T_{B}(G) \subseteq G$.    \\
    $\LRto$ (\romannumeral 1) $\mathit{lfp}(T_{B^{\cal A}}) \neq {\cal A}$, or (\romannumeral 2) ${\cal A} \notin E^-$, or (\romannumeral 3) $\{ {\cal A} \} - E^- = \{ {\cal A} \}$ and $\exists G \in \{ {\cal A} \},  T_{B}(G) \subseteq G$.   \\
    $\LRto$ (\romannumeral 1) $\mathit{lfp}(T_{B^{\cal A}}) \neq {\cal A}$, or (\romannumeral 2) ${\cal A} \notin E^-$, or (\romannumeral 3) ${\cal A} \notin E^-$ and $T_{B}({\cal A}) \subseteq {\cal A}$.   \\
    $\LRto$ (\romannumeral 1) $\mathit{lfp}(T_{B^{\cal A}}) \neq {\cal A}$, or (\romannumeral 2) ${\cal A} \notin E^-$, or (\romannumeral 3) ${\cal A} \notin E^-$.   \\
    $\LRto$  ${\cal A} \notin E^-$ or ${\cal A} \neq \mathit{lfp}(T_{B^{\cal A}})$.   \\
    $\LRto$ Condition (c3).

    (\uppercase\expandafter{\romannumeral4})
    $O^+ \cap O^- \neq \emptyset$  by Condition (C4).   \\
    $\LRto$ $E^+ \cap E^- = \emptyset$  since $\vert {\cal Q} \vert = 1$.  \\
    $\LRto$ Condition (c4).
\end{proof}

\begin{propositionA}[Transformations between induction tasks]  
Let $T = \tuple{B, O^+, O^-}$ be an induction task of NLP from partial stable models. A NLP $H$ such that $H \in \mathit{ILP_{LPaSM}}(T)$ if and only if $H \in \mathit{ILP_{LSM}}(T_1)$ for some $T_1 = \tuple{B, E^+, E^-}$ such that 
\begin{enumerate}[label=(c\arabic*)]
    \item $E^+$ is a minimal hitting set of $\{\mathit{de}(o) \mid o \in O^+ \}$, written as $E^+ \in \mathit{SMHS}(\{\mathit{de}(o) \mid o \in O^+ \})$, and 
    \item $E^- = \{ J \in \mathit{de}(o) \mid o \in O^- \}$.
\end{enumerate}
\end{propositionA}
\begin{proof}
    By Definition~\ref{def:ILT:NLP:partial} and a LSM induction task, we should prove $E^+\subseteq \mathit{SM}(B\cup H) \land E^-\cap \mathit{SM}(B\cup H)=\emptyset \LRto~\ref{eq:a1} \land~\ref{eq:a2}$. It can be realized via $\ref{eq:a1} \LRto E^+\subseteq \mathit{SM}(B\cup H)$ and $\ref{eq:a2} \LRto E^-\cap \mathit{SM}(B\cup H)=\emptyset$ as follows.

    (\uppercase\expandafter{\romannumeral1})
   \ref{eq:a1} in Definition~\ref{def:ILT:NLP:partial}.   \\
    $\LRto$ $\forall o \in O^+ \exists I \in \mathit{SM}(B \cup H), I \propto o$.   \\
    $\LRto$ $\forall o \in O^+ \exists I \in \mathit{SM}(B \cup H), I \in \mathit{de}(o)$.   \\
    $\LRto$ $\forall o \in O^+, \mathit{SM}(B \cup H) \cap \mathit{de}(o) \neq \emptyset$.   \\
    $\LRto$ $\mathit{SM}(B \cup H)$ is a hitting set of $\{\mathit{de}(o) \mid o \in O^+ \}$. \\
    $\LRto$ $E^+ \subseteq \mathit{SM}(B \cup H)$ since Condition (c1). \\
    $\LRto$ $E^+\subseteq \mathit{SM}(B\cup H)$ in a LSM induction task.
    
    (\uppercase\expandafter{\romannumeral2})
    \ref{eq:a2} in Definition~\ref{def:ILT:NLP:partial}.   \\
    $\LRto$ $\forall o \in O^- \nexists I \in \mathit{SM}(B \cup H), I \propto o$.   \\
    $\LRto$ $\forall o \in O^- \nexists I \in \mathit{SM}(B \cup H), I \in \mathit{de}(o)$.   \\
    $\LRto$ $\forall o \in O^- \forall I \in \mathit{SM}(B \cup H), I \notin \mathit{de}(o)$.   \\
    $\LRto$ $\forall o \in O^-, \mathit{SM}(B \cup H) \cap \mathit{de}(o) = \emptyset$.   \\
    $\LRto$ $\mathit{SM}(B \cup H) \cap \{ J \in \mathit{de}(o) \mid o \in O^- \} = \emptyset$.   \\
    $\LRto$ $\mathit{SM}(B \cup H) \cap E^- = \emptyset$ since Condition (c2). \\
    $\LRto$ $E^-\cap \mathit{SM}(B\cup H)=\emptyset$ in a LSM induction task.   
\end{proof}

\begin{corollaryA}[Existence of a NLP solution]
    Let $T = \tuple{B, O^+, O^-}$ be an induction task of NLP from partial stable models. $\mathit{ILP_{LPaSM}}(T) \neq \emptyset$ if and only if there exists $E^+ \in \mathit{SMHS}(\{\mathit{de}(o) \mid o \in O^+ \})$ and $E^- = \{ J \in \mathit{de}(o) \mid o \in O^- \}$ such that 
    \begin{enumerate}[label=(c\arabic*)]
        \item $E^+$ is $\subseteq$-incomparable, and
        \item $e \models B$ for each $e \in E^+$, and 
        \item ${\cal A} \notin E^-$ or ${\cal A} \neq \mathit{lfp}(T_{B^{\cal A}})$, and 
        \item $E^+ \cap E^- = \emptyset$.
    \end{enumerate}
\end{corollaryA}
\begin{proof}
    It can be proved by Proposition~\ref{prop:tasks:map} and Proposition~\ref{prop:taskNLP:exist}.
\end{proof}

\end{appendices}

\end{document}